\documentclass{article}

\PassOptionsToPackage{numbers, compress}{natbib}

\usepackage[utf8]{inputenc} 
\usepackage[T1]{fontenc}    
\usepackage{lmodern}        
\usepackage{hyperref}       
\usepackage{url}            
\usepackage{booktabs}       
\usepackage{amsfonts}       
\usepackage{nicefrac}       
\usepackage{microtype}      
\usepackage{xcolor}         

\usepackage{fullpage}

\usepackage{amsmath}
\usepackage{amssymb}
\usepackage{mathtools}
\usepackage{amsthm}
\usepackage{graphicx}
\usepackage{caption}
\usepackage{subcaption}
\usepackage{float}
\usepackage[capitalize,noabbrev]{cleveref}

\usepackage{graphicx} 
\usepackage{amsmath}
\usepackage{amsfonts}
\usepackage{hyperref}
\usepackage{blindtext}
\usepackage{mathtools}

\usepackage{natbib}

\allowdisplaybreaks

\usepackage{xcolor}

\usepackage{amsthm}
 
\newcounter{theorems}
\newtheorem{theorem}[theorems]{Theorem}
\newtheorem{lemma}[theorems]{Lemma}

\newtheorem{definition}[theorems]{Definition}
\newtheorem{remark}[theorems]{Remark}
\theoremstyle{definition}

\usepackage{tocloft}
\renewcommand{\contentsname}{Contents}
\renewcommand{\cfttoctitlefont}{\Large\bfseries}

\makeatletter
\renewcommand{\tableofcontents}{%
  \begin{center}
    {\cfttoctitlefont \contentsname\par}
  \end{center}
  \vspace{-0.3em}
  \@starttoc{toc}%
}
\makeatother
\title{A Sharper Picture of Generalization in Transformers}

\newcommand{\AND}{\par\medskip}

\author{%
Paul Lintilhac \\
Thayer School of Engineering\\
Dartmouth College\\
\texttt{paul.s.lintilhac.th@dartmouth.edu} \\
\AND
Sair Shaikh\\
Dartmouth College \\
\texttt{\{sair.shaikh.26\}@dartmouth.edu} \\
}

\makeatletter
\renewcommand{\maketitle}{
  \begin{center}
    {\Large\bfseries \@title \par}
    \vskip .3em
    {\normalsize \@author}
  \end{center}
}
\makeatother

\begin{document}

\maketitle

\begin{abstract}
  We study transformers' generalization behavior on boolean domains from the perspective of the Fourier spectra of their target functions. In contrast to prior work \citep{edelman2022inductive, trauger2024sequence}, which derived generalization bounds from Rademacher complexity, we investigate the feasibility of obtaining generalization bounds via PAC-Bayes theory. We show that sparse spectra concentrated on low-degree components enable low-sharpness constructions with good generalization properties. Our idea is to show the existence of flat minima implementing any boolean function of sparsity no greater than the context length, and then apply a PAC-Bayes bound to an idealized low-sharpness learner, resulting in a non-vacuous generalization bound. We use this to give a formal account of why chain-of-thought improves generalization for high-degree target functions, and show that the complexity parameters in our bound can be efficiently estimated via property testing. We evaluate predictions empirically and conduct a mechanistic interpretability study to support the realism of our theoretical construction in real transformers.

\end{abstract}

While various authors have studied the expressive capacity of transformers to express different classes of  functions \citep[e.g.][]{hahn2020theoretical, bhattamishra2020ability, merrill2023parallelism, strobl2023survey, sanford2023representational, liu2023shortcuts,merrill2023parallelism, sanford2024transformers}, their learning biases still are incompletely understood. 
\citet{edelman2022inductive, trauger2024sequence} showed that transformers with fixed positional encodings have a relatively small statistical complexity. Furthermore, they showed that transformers can represent sparse functions, and showed that those are empirically learned well by transformers, though learnability results on sparse functions remained empirical. \citet{abbe2023generalization} hypothesized that transformers prefer to learn the minimum-degree function interpolating the training data, though evidence was limited to experiments and much simpler architectures.
Other work showed a low-sensitivity and low-degree biases in Transformers \citep{bhattamishra2022simplicity, hahn-rofin-2024-sensitive, abbe2023generalization}. 

The present paper provides a generalization result for transformers learning low-degree functions, under an idealized low-sharpness learner. 
In contrast to prior work on capacity-based bounds grounded in parameter norms \cite{edelman2022inductive, trauger2024sequence}, we explore the feasibility of using PAC-Bayes methods grounded in the existence of flat minima, introducing a novel ``constructive'' methodology for oracle PAC-Bayes bounds. Our result could be applied to other formal languages and algorithmic tasks beyond just boolean functions, leveraging sharpness in the loss landscape as a possible path to bridging the so-called ``expressivity-learnability'' gap.

\paragraph{Contributions.}
Our contributions are as follows:
\begin{enumerate}
    \item We introduce a \emph{constructive PAC-Bayes methodology} for deriving generalization bounds for transformers. Rather than relying on capacity-based arguments, we explicitly construct a transformer with bounded norm and sharpness for a given function class, and show that any learner finding a solution dominated in these quantities inherits the generalization guarantee. This approach could be applied to other formal languages and algorithmic tasks (e.g., Dyck-$(k,D)$ languages or other problems admitting canonical transformer constructions), providing function-class-specific bounds parameterized by the natural complexity measures of the task.
    \item We use this methodology to provide a constructive bridge between \emph{expressivity and learnability} for transformers on sparse Boolean functions: while prior work has shown that transformers can express high-degree Boolean functions \citep{chiang2022overcoming, liu2023shortcuts}, they remain difficult to learn in practice. Our bound makes this precise, showing that the generalization gap scales as $O(\omega D_f^3)$ in the Fourier sparsity $\omega$ and degree $D_f$, thereby providing a concrete complexity-theoretic explanation for why expressible functions may not be learnable.
    \item We use our techniques to derive a sharpness-based learning bound showing that \emph{chain-of-thought exponentially improves the generalization bound} for Parity: our bound grows only linearly in $T$ with CoT, versus exponentially without it.
\end{enumerate}

We state reasonable conditions under which our bound is non-vacuous, including lower bounds on the number of training points required to make our bound non-vacuous for each complexity class (defined by the pair [maximum Fourier degree, and Fourier sparsity] of the target function).  Using the PARITY task as an example (a special case of our general learning scenario where the Fourier sparsity is 1, and the maximum degree is $T$) we shed light on exactly why CoT achieves lower error than single-pass learning.  Our proof approach is to show the existence of low-sharpness solutions representing sparse functions, and then apply a PAC-Bayes bound. 
%
Any function on a Boolean domain, $f : \{0,1\}^T \rightarrow \mathbb{R}$, possesses a unique representation in terms of parity functions:
\begin{equation}\label{eq:unique-fourier-sum}
    f(x) = \sum_{A \subseteq [1,T]} \lambda_A \cdot \chi_A(x)
\end{equation}
where $\chi_A(x) = \frac{1+(-1)^{\sum_{i \in A} x_i}}{2} \in \{0,1\}$ is the even-parity indicator for the subset $A$.

The \emph{degree} of $f$ is $ D_f :=   \max \{|A| : \lambda_A \neq 0\}$.
Re-expressing the above sum (\ref{eq:unique-fourier-sum}) as $f(x) = \sum_{t=1}^{\omega} c_t \chi_{S_t}$,
$\omega$ is the number of nonzero coefficients (the spectral \emph{sparsity}),  $c_t \in \mathbb{R}$, $S_t \subseteq [T]$ are the indices of the $t^{th}$ Fourier component,   and $|S_t|\leq D_f$.

For simplicity, we assume that all Fourier components have the same degree, i.e. $|S_t|=D_f \, \forall t \in [\omega]$. In addition, we focus on the special case where all Fourier components have a positive coefficient. Handling the general case of arbitrary boolean functions (of mixed degree, and with positive or negative coefficients) can be easily accomplished by adding additional heads; however this complicates the math significantly, and we believe that this special case already captures the key generalization behavior we want to highlight.
We claim that a transformer representing a function $f: \{0,1\}^T \to \mathbb{R}$ of low degree, and having  $\omega\leq T$ nonzero Fourier components, can be learned by a $1.5$-layer (2 attention+1 MLP sublayers), $1$-head transformer. Let $x_t \in \mathbb{R}^{d+1}$ be the embedding at the $t^{th}$ token position. The total hidden dimension includes $d$ purely positional dimensions as well as one extra dimension holding the bit value. The input matrix $X \in \mathbb{R}^{(T+1) \times (d+1)}$ then has the vector $x_t$ as its $t^{th}$ row, with a special output position at the end ($t=T+1$, with bit value fixed to $0$).

The attention score between locations $i$ and $j$ in the first layer is given by 
\[a_{i,j} = x_i^TW^{(1)}x_j,i,j \in [T],\]
where $W^{(1)} \in \mathbb{R}^{(d+1) \times (d+1)}$ is the combined key and query projection matrices. The softmax score is given by 
\[\hat{a}_{i,j} = \frac{\exp(a_{i,j})}{\sum_s \exp(a_{i,s})}\]
The output of the attention layer is then given by 
\[b_t = \sum_{j=1}^T \hat{a}_{t,j}(V^{(1)})^Tx_j,\]
where $V^{(1)}\in\mathbb{R}^{(d+1)\times (d+1)}$ is our value projection.  The activations after the attention layers are the result of applying the two-layer MLP with a ReLU with a skip-connection:
\[g_t = f^{MLP} \big(b_t\big)=F^T(M^Tb_t+\Gamma)_+\]
Note that our construction does not include the layer norm. We then use a second attention layer parametrized by $W^{(2)} \in \mathbb{R}^{(d+1) \times (d+1)}, V^{(2)}\in \mathbb{R}^{d\times1}, $ such that the output is read from the final output position $T+1,$ i.e.
$$\mathcal{T}(X,W^{(1)},V^{(1)},M,F,\Gamma,W^{(2)},V^{(2)})$$
$$=(V^{(2)})^T G^T \phi\Big(G (W^{(2)})^T g_{T+1}\Big) ,$$
where $G\in \mathbb{R}^{(T+1)\times (d+1)}$ is a matrix with rows equal to $g_t, \forall t\in [T+1].$  The explicit form of the full transformer is discussed in more detail in the ``complete construction" section in the appendix. 

\section{Theoretical Results}\label{sec:theoretical-results}

Our theoretical results are centered around the relationship between generalization and sharpness. Empirical evidence that stochastic gradient descent (SGD) tends to find solutions with maximal parameter flatness on the training dataset, especially with smaller batch sizes \cite{DBLP:journals/corr/KeskarMNST16}. Sharpness-aware minimization has now been accepted as a powerful tool for regularization \citep{oikonomou2025sharpnessaware}, while PAC-Bayes theory has used sharpness as a key term in some of the first tight generalization bounds for deep learning \citep{neyshabur2017generalization}.  

We study the general setup of a learner minimizing loss, sharpness, and norm. 
Abstracting away from training dynamics, we study a general learner that, through implicit bias or explicit regularization, finds an interpolator that minimizes parameter norms and the sharpness of the loss:
\begin{definition}[Idealized Low-Sharpness Learner]
Given a space of possible parameters $\Theta$ for transformers, and a finite training set for a function $f : \{0,1\}^T \rightarrow \mathbb{R}$, find parameters $\Theta$ minimizing
\begin{equation}
     \hat{L}(f_{\mathbf{\Theta}})+ \alpha\|\Theta\| + \beta Tr\Big(\nabla^2 \big[\hat{L}(f_{\Theta})\big]\Big),
\end{equation}
where $\hat{L}(f_{\mathbf{\Theta}})$ is the training loss, and $\alpha,\beta$ are parameters that control the aversion to high parameter norm and sharpness, respectively. 

\end{definition}
\noindent\emph{Remark.} Sharpness is parameterization-dependent; however, in the PAC--Bayes bound it must be balanced against the parameter norm, and simple rescalings that reduce one often increase the other, making the joint complexity measure more robust to rescaling than a merely sharpness-based complexity measure. As shown in Appendix~\ref{non-pareto}, there exist functionally equivalent or training-equivalent parameterizations with strictly lower sharpness and strictly lower norm than ours.

\noindent\emph{Comparison with standard PAC-Bayes bounds.}
Standard PAC-Bayes bounds require knowledge of the actual learned parameters $\hat{\Theta}$ in order to compute the KL divergence between the posterior (centered at $\hat{\Theta}$) and a data-independent prior. Our approach replaces this requirement with a weaker structural assumption: that the learner, via implicit or explicit regularization, finds \emph{some} interpolator whose parameter norm and sharpness are each dominated by those of our explicit construction $\Theta$. Crucially, we require no knowledge of \emph{which} specific minimum is found---only that regularization pressure ensures the learned solution does not exceed the norm and sharpness of a construction we can analyze. This is strictly less information than knowing $\hat{\Theta}$ itself, making our assumption weaker in this respect, at the cost of introducing the structural assumption that such a dominating construction exists (which we validate empirically in Section~\ref{lowsharpnessvalidation}). The key consequence is that the PAC-Bayes bound, which depends on the expected perturbed training loss and parameter norm of the learned solution $\hat{\Theta}$, can be upper bounded by the corresponding quantities computed for our explicit construction $\Theta$, whose sharpness and norms we have bounded analytically. This allows us to derive generalization guarantees for $\hat{\Theta}$ without knowing anything about its specific parameter values.

\begin{theorem}
\label{PAC-Bayes Bound}
Let $f$ be a target function with degree $D_f$ and sparsity $\omega \leq T$, and let $\Theta$ denote our explicit construction for $f$.
Let $\sigma^2>0$.
Suppose $\mathbf{\hat{\Theta}}$ is any interpolator on a training set of size $m$ (i.e., $\hat{L}(f_{\mathbf{\hat\Theta}})=0$) whose parameter norm and expected perturbed sharpness (trace of the Hessian under Gaussian perturbation) are each dominated by those of $\Theta$:
$$\|\hat{\Theta}\|_F^2 \leq \|\Theta\|_F^2, \qquad \mathbb{E}_{\epsilon}\Big[Tr\Big(\nabla^2\big[\hat{L}(f_{\hat{\Theta}+\epsilon})\big]\Big)\Big]\leq \mathbb{E}_{\epsilon}\Big[Tr\Big(\nabla^2\big[\hat{L}(f_{\Theta+\epsilon})\big]\Big)\Big].$$
Suppose further that for some constant $\Sigma > 0,$ the loss random variable $L(X,f_{\hat{\Theta}})$ (over $X$, for fixed $\hat{\Theta}$) satisfies the sub-Gaussian condition:
$\mathbb{E}_X\left[e^{t[L(X,f_{\hat{\Theta}})-\mathbb{E}_X[L(X,f_{\hat{\Theta}})]]}\right]\leq e^{\frac{\Sigma^2 t^2}{2}}.$
With probability $\geq 1-\delta$, 
\begin{align*}
   & \mathbb{E}_{\epsilon \sim \mathcal{N}(0,\sigma^2)}[{L}(f_{\hat\Theta+\epsilon})]  \leq \\
   & \underbrace{\sigma^2\Big(O(\omega D_f^3) + P(\sigma;\omega, D_f,T)\Big)}_{\text{Perturbed Sharpness Term}} \\
&    + \underbrace{2\sqrt{ \frac{\Sigma^2}{2m}\Big(\frac{ O\Big(D_f^3+\log(T)^2\omega D_f\Big)}{2\sigma^2}+ln\frac{1}{\delta}\Big)}}_{\text{Parameter Norm Term}}
\end{align*}
\end{theorem}
\begin{remark}[On the domination assumption]
The assumption that $\hat{\Theta}$ is coordinatewise dominated by $\Theta$ in both norm and sharpness does not follow from minimizing a weighted sum (Definition~1.1) alone, since such an objective can trade one quantity against the other. We therefore state it as an explicit assumption. Its plausibility rests on two observations: (i)~as shown in Appendix~\ref{non-pareto}, our construction is provably \emph{not} Pareto-optimal---it admits strict simultaneous reductions in both norm and sharpness via simple reparameterizations---so dominated interpolators exist in principle; and (ii)~empirically, learned solutions achieve both lower norm and lower sharpness than the construction by roughly two orders of magnitude (Section~\ref{lowsharpnessvalidation}).
\end{remark}
Conversely, the bound has two identifiable failure modes. First, if the construction were Pareto-optimal in the (norm, sharpness) plane, the domination assumption would hold only trivially---with $\hat{\Theta}=\Theta$---and the bound would reduce to the generalization guarantee for the explicit construction itself, forfeiting any gain from the learner finding a superior solution. Second, if the learned solution implements the target function via a qualitatively different mechanism whose sharpness profile scales differently with degree and sparsity, domination of our particular construction need not yield a tight bound. Both failure modes are empirically verifiable: the first is ruled out by Appendix~\ref{non-pareto} and Figure~\ref{fig:norm comparison}; the second can be monitored by comparing the spectral structure of learned and constructed Hessians, which we leave to future work.
\begin{remark}[On the sub-Gaussian assumption]
The sub-Gaussian condition is not restrictive. Since the inputs are binary and $\|\hat{\Theta}\|_F^2$ is bounded by assumption, the transformer output $\mathcal{T}(X,\hat{\Theta})$ is uniformly bounded over all $X \in \{0,1\}^T$. The target $|f(x)| \leq \|c\|_1 \leq \sqrt{\omega}\|c\|_2$ is likewise bounded. Hence the squared loss is uniformly bounded and therefore sub-Gaussian by Hoeffding's lemma. We also estimate $\Sigma$ empirically in Appendix~\ref{experimentaldetails} and confirm that it is small in practice.
\end{remark}

We express the first term in the perturbed sharpness above in big-O notation because it has a simple, tight analytic upper bound. On the other hand, we write the sharpness perturbation term as $P(\sigma; \omega, D_f, T),$ because we prefer to use an empirical upper bound for this term rather than our analytic upper bound.  As shown in \autoref{full perturbed hessian bound}, the trace of the perturbed loss hessian can be written as:
\begin{align*}
Tr\Big(\nabla^2[L_f(\mathcal{T}(X,\Theta+\zeta))]\Big) &\leq 2\| \nabla \mathcal{T}(X,\Theta+\zeta)\|^2 \\
&\quad + 2|f(x)-\mathcal{T}(X,\Theta+\zeta)|\Big|Tr\Big(\nabla^2 \mathcal{T}(X,\Theta+\zeta)\Big)\Big|
\end{align*}
As we can see, the first term is a simple function of the gradient of the unperturbed transformer construction, and along with the analysis of norms, is the main focus of our analysis. While we derive an analytic upper bound on $P$ in the appendix, yielding a fully analytic bound, this bound retains virtually all of the theoretical content while being dramatically more usable than the fully analytic bound. Thus, we refer to our main result as the ``semi-analytic'' generalization bound with an empirically certified perturbation term.

We present two instantiations of this bound. The first is a \emph{fully analytic} bound (Appendix~\ref{full perturbed hessian bound}) in which the perturbation term $P(\sigma;\omega,D_f,T)$ is bounded by a closed-form expression. While mathematically rigorous, this analytic bound on $P(\sigma)$ has large constants and grows rapidly in $D_f, \omega, T$, making it overly conservative (as confirmed empirically in Section~\ref{perturbation analysis}). The second is a \emph{semi-analytic} (hybrid theory-plus-measurement) bound that retains the analytic structure of the unperturbed sharpness and norm terms---which carry the core theoretical content of how degree and sparsity govern generalization---but replaces the analytic bound on $P(\sigma)$ with an empirical upper bound estimated from our explicit construction. This semi-analytic bound is the version we use in our main experiments (Figure~\ref{fig:generalizationgap}), and is non-vacuous: as long as $m$ grows asymptotically fast enough, the norm term is $o(1)$, and we are free to pick $\sigma$ dependent on $m$ to make the sharpness term $o(1)$ as well. To illustrate: for a tiling of $2$-parities with $D_f=2, \omega=10, T=20$---a favorable parameter regime where our semi-analytic bound is significantly tighter than previous purely norm-based bounds \citep{edelman2022inductive}---the semi-analytic bound is already non-vacuous at $m=8192$ (Figure~\ref{fig:generalizationgap}), whereas the fully analytic bound requires $m\approx 2\times 10^9$, with the dominant source of looseness being the analytic bound on the Hessian perturbation term $T_p\cdot|\Theta|\cdot(H_u + H_p)$. We provide a detailed numerical comparison, including explicit constant evaluations, in Appendix~\ref{comparison}. We regard the looseness of the analytic $P(\sigma)$ bound as a limitation of the current analysis rather than a fundamental barrier; tightening it (e.g., via parameter-specific perturbation variances) is a natural direction for future work.

\paragraph{Proof Strategy}
The proof of this result proceeds via PAC-Bayes theory.
Our proof proceeds in two stages.
The first part is to show the existence of a construction with a small parameter volume and low sharpness.
PAC-Bayes theory then tells us that learned minima with such properties will possess good generalization properties.
The size of the perturbation $\sigma$ controls the trade-off between the parameter norms and the sharpness of the loss landscape. When combined with our assumption that $\hat{\Theta}$ is a low-sharpness interpolator of our exact construction, $\Theta,$ the (unknown) norm and expected perturbed sharpness of our interpolator $\hat{\Theta}$ can be replaced by those of the exact transformer $\Theta$, for which these quantities have been bounded explicitly, and from which the bound on the expected perturbed training loss follows.

\subsection{Step 1: Existence of a ``Good'' Construction}

We specify a simple transformer construction for Boolean functions, and show that -- for bounded degree and sparsity -- it has bounded norm and sharpness.
Our first theoretical results show that the trace of the loss Hessian for our construction is bounded by a function that increases polynomially in both the maximum degree, $D_f,$ and the Fourier sparsity, $\omega.$ The basic intuition behind this result is that when our construction implements a higher-degree function, not only are the parameter norms larger, but also the MLP must interpolate a function of higher frequency but equal amplitude, and will therefore carry a larger derivative with respect to changes in the parameters. As a first step, we bound the trace of the loss hessian in terms of the norms of the gradient.
\begin{theorem}
\label{ExactTraceBounds}
    Let $f$ be a target boolean function of sparsity $\omega\leq T$ and maximum degree $D_f$, and let $\mathcal{T}(X,\Theta)$ be a transformer of context length T implementing it exactly according to our construction. Let $L_f(X,\Theta)=(\mathcal{T}(\Theta,X)-f(X))^2$ be the unbounded, quadratic loss for our transformer learning the function $f$ evaluated at input $X.$ Then for every $X \in \{0,1\}^T$, the trace of the pointwise loss Hessian is bounded by:
$$Tr\Big(\nabla^2 L_f(X,\Theta)\Big)=2 \| \nabla \mathcal{T}(X,\Theta)\|^2 \leq 2G_u(\omega, D_f) \in O(\omega D_f^3)$$
Since $G_u$ is uniform in $X$, defining $L(f_{\mathbf{\Theta}}):=\mathbb{E}_X\big[L_f(X,\Theta)\big]$, the same bound holds for the global loss: $Tr\big(\nabla^2 L(f_{\mathbf{\Theta}})\big) \leq 2G_u(\omega, D_f).$
\end{theorem}

We defer a full proof to \autoref{ExactTraceBoundsAppendix} in the Appendix.
Here, we provide an outline of the construction.
We start with a simple transformer construction that approximates functions on Boolean domains, based on their Fourier-Walsh transforms. The activations at the output of the first (attention + MLP)  layer contain the values of each non-zero Fourier component $\chi_{S_t}$ in one of the coordinates. From there, our second attention layer will take a linear combination of these components with weights $c_t$. 

In order to calculate the value of each Fourier component $\chi_{S_t}$ at each position in the first layer, we will make use of a purely position-aware attention mechanism with $O(\log(T))$ scaling. This trick allows us to overcome a theoretical limitation of transformers resulting from the softmax giving significant weight to even inactive positions, which becomes more pronounced for large T \cite{hahn2020theoretical,chiang2022overcoming,edelman2022inductive}. For each position this puts post-softmax weights at approximately $\frac{1}{D_f}$ on all positions $j \in S_t,$ and close to 0 elsewhere. Let the value weight matrix in the first attention layer be $V^{(1)} \in \mathbb{R}^{(d+1) \times( d+1)}$, a projection matrix that multiplies the attention weights with the bit values and stores the normalized component sums $( \frac{1}{D_f}\sum_{j \in S_t} x_j)$ in the final, $d+1^{th}$ dimension. 

After computing this normalized prefix sum and storing it in each of the activations after the first attention layer, we then apply the MLP to approximate the ``mod 2" function. To do so minimally, we employ the first layer MLP matrix, $M\in \mathbb{R}^{(d+1)\times 4(D_f+1)}$ ,  which together with our bias term $\Gamma  \in \mathbb{R}^{(d+1) \times 4(D_f+1)}$ acts as a a set of indicators for each unique value in our grid of possible prefix sums. Then our second-layer MLP matrix $F\in \mathbb{R}^{4(D_f+1)\times (d+1)}$ linearly combines these indicators with the correct memorized function values. 


The following theorem approximates the error of the perturbed transformer on a test point, assuming the model has achieved $0$ error on the training set.
\begin{theorem}
\label{average sharpness approximation}
Let $\mathcal{T}(X,\Theta+\epsilon)$ be our transformer with parameters perturbed by $\epsilon \sim \mathcal{N}(0,\sigma^2)^n,$ where $n=|\Theta|$ is the parameter count. Let $\hat{L}(f_{\Theta+\epsilon})$ be the empirical (training) loss of the perturbed transformer with parameters $\Theta+\epsilon$. The expected empirical loss under $\epsilon$ decomposes exactly as:

\begin{align*}
    \mathbb{E}_{\epsilon \sim \mathcal{N}(0,\sigma^2)^n}\Big[\hat{L}(f_{\Theta+\epsilon})\Big] &=
    \frac{\sigma^2}{2}Tr\Big(\nabla^2 \big[\hat{L}(f_{\Theta})\big]\Big) \\
    &\quad + \frac{1}{2}\mathbb{E}_{\epsilon}\Big[\epsilon^T\Big(\nabla^2\big[\hat{L}(f_{\Theta+\zeta(\epsilon)})\big]
    - \nabla^2\big[\hat{L}(f_{\Theta})\big]\Big)\epsilon\Big]
\end{align*}
where $\zeta(\epsilon)\in \mathbb{R}^n$ are the ``remainder perturbations'', i.e. $\zeta_i(\epsilon) \in [0,\epsilon_i]$ is the Lagrange remainder point making the second-order Taylor expansion exact, conditional on each Gaussian perturbation $\epsilon_i$.
\end{theorem}
The proof is given in Theorem~\ref{average sharpness approximation appendix}, and follows from the Lagrange form of Taylor's theorem with Hutchinson's Trace Estimator applied to the fixed (unperturbed) Hessian $\nabla^2[\hat{L}(f_\Theta)]$. Since $\zeta(\epsilon)$ depends on $\epsilon$, Hutchinson cannot be applied directly to the full perturbed Hessian. Note that the same decomposition applies to the global (test) loss just as it applies to the training loss; we state it in terms of the training loss as this is the only way in which it will be used in downstream theorems. This bound simply states that the larger the sharpness, the larger the generalization error. Of course, the generalization error is also controlled by a free parameter, $\sigma;$ the larger the perturbation away from our exact construction, the larger the generalization error. Later on, we will contextualize $\sigma$ as part of a trade-off, where sharpness is but one of several factors controlling the generalization gap. It will be useful to bound the RHS. The first term is bounded by $\sigma^2 G_u(\omega,D_f)$ via \autoref{ExactTraceBounds}. For the second term, applying $\epsilon^T A \epsilon \leq \|A\|_{\mathrm{op}}\|\epsilon\|^2$ and taking expectations (using $\mathbb{E}[\|\epsilon\|^2]=\sigma^2|\Theta|$) gives a contribution controlled by the perturbed Hessian operator norm bound of Theorem~\ref{perturbedhessianbounds}, which is then combined with the unperturbed trace in \autoref{full perturbed hessian bound}, yielding:
$$\mathbb{E}_\epsilon\Big[\hat{L}(f_{\Theta+\epsilon})\Big] \leq \frac{\sigma^2}{2}\Big(2G_u(\omega, D_f) + P(\sigma;\omega,D_f,T,d)\Big)$$
Recall that the need to evaluate the hessian of the loss of the transformer at a perturbed point arises from the fact that we are using the second-order Taylor expansion to the sharpness using the Lagrange form of the remainder. While we could have simply written the trace of our exact construction plus an error term that is $O(\sigma^2),$ eschewing the exact remainder term, this would have lacked rigor. Instead, we provide both a theoretical and empirical analysis of this $P(\sigma;...) $ term.

Note that the component of the Hessian trace from the unperturbed Hessian described by \autoref{ExactTraceBounds}, , $G_u(\omega, D_f)$, intuitively does not depend on the size of the perturbations $\sigma;$ interestingly, nor does it depend on $d,T$ . The perturbation to the Hessian, $P(\sigma;\omega, D_f, T,d),$ ends up being quite complex, and we defer the complete analytic bound on this term to \autoref{full perturbed hessian bound} in the supplementary material. While mathematically tractable, $P(\sigma;...)$ has large constants and grows very rapidly in  $D_f,\omega, T,d$. Any polynomial dependency on T is especially unfavorable, given that no other part of our bound is $poly(T).$ Luckily, our empirical experiments suggest that the perturbation to the trace of the loss of the perturbed construction in reality is relatively small compared to our theoretical upper bound. \\To maintain the theoretical insights of our bound's dependency on the key complexity parameters $D_f, \omega$ while still ensuring it does not suffer from the pessimistic dependencies in the $P(\sigma)$ that  do not reflect reality, we define the ``semi-analytic'' bound. This bound keeps the same form of the unperturbed sharpness and the norm, but replaces the analytic bound on the $P(\sigma)$ term with an empirical estimate the worst-case perturbation $P(\sigma).$ This bound mitigates the issue of the analytic $P(\sigma)$ term making the bound impractical, while still expressing the balance between the norm term and the (unperturbed) sharpness term, modulated by the free parameter $\sigma,$ retaining the core feature of the PAC-Bayes approach.



\subsection{Step 2: PAC-Bayes}

Our construction provides an interpolator $\Theta$ with $0$ training loss for some particular values of $\|\Theta\|$ and $\mathbb{E}_{\epsilon}\big[Tr\big(\nabla^2[\hat{L}(f_{\Theta+\epsilon})]\big)\big]$. For some setting of the regularization parameters $\alpha,\beta$, one can thus find an interpolator $\hat{\Theta}$ for which the parameter norm and expected perturbed sharpness are both upper-bounded by those of our construction, $\Theta$ \footnote{We would like to add one important caveat, which is that our idealized learner may not find a solution that strictly dominates both the sharpness and parameter norms of our construction if ours is already on the Pareto Frontier. Aside from the empirical experiments in \ref{lowsharpnessvalidation} showing that this is not the case, we defer the reader to \ref{non-pareto} in the appendix for the strong theoretical case for why our construction is nowhere near the Pareto frontier.} :
    $\hat{L}(f_{\mathbf{\hat{\Theta}}})=0,$
    $\mathbb{E}_{\epsilon \sim \mathcal{N}(0,\sigma^2)^n}\Big[Tr\Big(\nabla^2\big[\hat{L}(f_{\mathbf{\hat{\Theta}}+\epsilon})\big]\Big)\Big]\leq \mathbb{E}_{\epsilon \sim \mathcal{N}(0,\sigma^2)^n}\Big[Tr\Big(\nabla^2\big[\hat{L}(f_{\mathbf{\Theta}+\epsilon})\big]\Big)\Big],$
    $\|\hat{\Theta}\|\leq \|\Theta\|$
Thus, our construction provides information about the outcome of this general learner even if it finds an entirely different construction. In particular, since $\hat{\Theta}$ is an interpolator ($\hat{L}(f_{\hat{\Theta}})=0$), these two assumptions together imply domination of the expected perturbed empirical loss via the Taylor expansion $\mathbb{E}_\epsilon[\hat{L}(f_{\hat{\Theta}+\epsilon})] \approx \frac{\sigma^2}{2}\mathbb{E}_\epsilon[Tr(\nabla^2[\hat{L}(f_{\hat{\Theta}+\epsilon})])]$, which is the quantity that enters the PAC-Bayes bound. 
To show that $\hat{\Theta}$ has good generalization properties, we can leverage the machinery of PAC-Bayes theory. PAC-Bayes theory considers the weights learned during training to be a data-dependent distribution (referred to as the "posterior", denoted by $\mathcal{Q}$), which is compared against the initial distribution of weights (referred to as the "prior", denoted by $\mathcal{P}$) using the KL divergence to bound the generalization gap. Our specific PAC-Bayes bound builds upon a variant due to \cite{alquier2021pacbayes}. Our approach is to use the assumed low-sharpness interpolator of our construction as the center of the posterior distribution. This version of the PAC-Bayes bound can be expressed in such a way that it includes a term representing the parameter sharpness on the training set. We further adapt our bound to the case of a quadratic loss, using the assumption that the test losses are sub-gaussian. 
For a detailed derivation of our Oracle PAC-Bayes bound, see \autoref{pac-bayes discussion} in the appendix. \footnote{We note that our PAC-Bayes bound is of the type where the bound on the sharpness is not data-dependent, but rather dependent on other assumptions about the complexity of the target function and the learning algorithm. These PAC-Bayes bounds are referred to by e.g. \citet{alquier2021pacbayes} as ``Oracle PAC-Bayes Bounds''}

Our result also applies to Chain-of-Thought (CoT), providing another perspective on its effectiveness: For PARITY, there is a CoT using a function of degree 2; we get a much better generalization bound for the CoT than for a transformer performing it without CoT. See App.~\ref{sec:cot} for more.

%

\begin{figure}[tbh]
    \includegraphics[width=1\columnwidth]{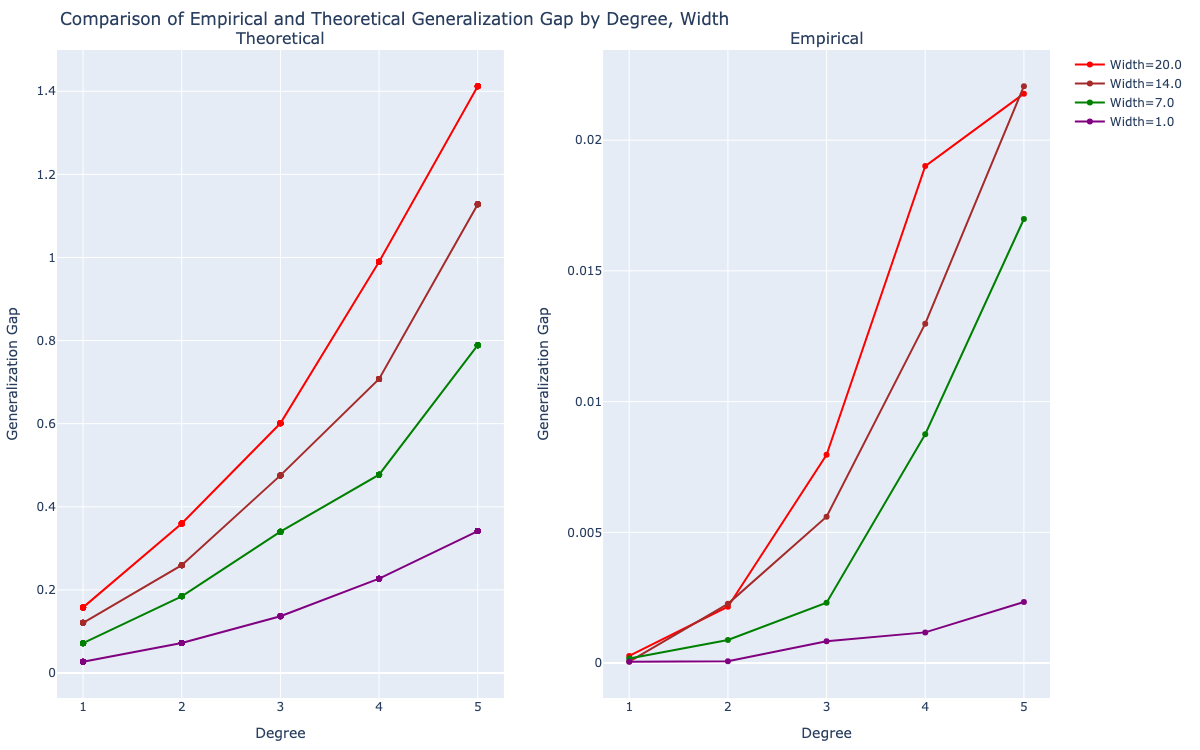}
    \caption{(left) A plot showing our semi-analytic generalization bound for $m=8192.$ Note that the free parameter $\sigma$ in our bound, representing the magnitude of the parameter perturbation, is optimized separately for each degree a sparsity.  (right) A plot showing the average empirical generalization gap over a sample of 11 functions from each Degree, Sparsity class, also for $m=8192.$  We note that, as predicted by our generalization bound, the gap increases super-linearly with degree, and the rate of increase increases with larger sparsity.}
    \label{fig:generalizationgap}
\end{figure}


\section{Experiments}

\subsection{Generalization Gap}
As a first step in instantiating our bound, we algorithmically picked our values for $\Sigma, m$ according to the procedure described in \autoref{experimentaldetails}. With the bound of \autoref{PAC-Bayes Bound} instantiated, we then had to pick the value to be used for $\sigma.$ To do so for the semi-analytic bound where we only had the worst-empirical case perturbations on a discrete mesh, we optimized the bound over a discrete mesh of $\sigma$ for each $\omega, D_f$ . Our mesh contained 20 evenly spaced values of $\sigma$ ranging from .01 to .00001. For our fully analytic bound, we were able to optimize in continuous space using the simple, default solver provided by the SciPy package in Python, which leverages BFGS.

With our bound instantiated, we obtained the predictions that an increase in both of the complexity parameters $D_f,$ $\omega$ should increase the generalization gap super-linearly, and that since they are multiplied together in our bound, the increasing one of the complexity parameters should increase the rate of increase for the other. We then validated these predictions empirically in the controlled regime described below.

For our main experiments for calculating the empirical generalization gap, we trained a simple transformer with 2 layers and 1 head, and without layernorm as in our construction. Due to the difficulty in generating boolean functions of arbitrary spectral sparsity and degree, we opted to relax this requirement to functions on a boolean domain. To be precise, we relax only the \emph{codomain}: rather than restricting to Boolean outputs, we sample real-valued targets on $\{0,1\}^T$ via sparse Fourier--Walsh expansions; degree and sparsity are defined exactly as in \ref{eq:unique-fourier-sum}. Note that our above theoretical analysis is valid for any function with a Fourier-Walsh spectrum with all-positive coefficients and sparsity no greater than the sequence length, T.

In order to simplify both our analysis and our experiments, we restricted the class of functions to those with a constant degree for all components.  In order to generate such random functions, we generate the coefficients according to a standard normal distribution, take their absolute values to ensure all-positive coefficients (matching the theoretical assumption), and then normalize so that they are centered and with variance 1. For the Fourier components of degree $D_f$, we select a random subset of size $D_f$ from our $T$ input dimensions by selecting the initial $\omega$ elements of a permutation of all $D_f$ sized combinations of $[1...T]$. 

\subsection{Validating Low-Sharpness Interpolator Assumption}
\label{lowsharpnessvalidation}
In order to validate our approach, we hard-coded a transformer according to our exact construction, and compared its sharpness (as measured by the trace of the loss Hessian) to that of our learned solutions for each degree and sparsity. \autoref{fig:sharpness comparison} and \autoref{fig:norm comparison}  show that both the sharpness and Frobenius norm of the hard-coded transformer were significantly higher than that of the learned constructions, providing empirical support for the plausibility of our low-sharpness interpolator assumption in the tested regime. This is a key assumption in our methodology. Note that the we evaluated the exact same functions as in the main generalization gap experiments: we tested 11 randomly generated functions with positive Fourier Coefficients,  degrees in $\{1,2,3,4,5\},$ and widths in $\{1,7,14,20\}.$

\begin{figure}[!htbp]
    \centering
    \begin{subfigure}{\columnwidth}
        \includegraphics[width=\columnwidth]{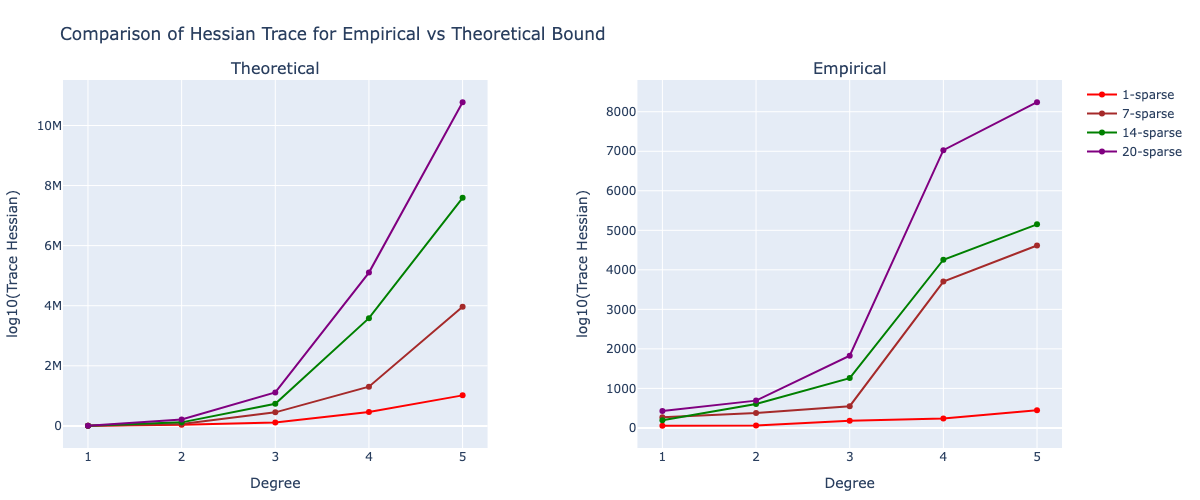}
        \caption{Left: the sharpness, as measured by the trace of the loss Hessian measured on the training set, of our exact construction for each degree and sparsity. Right: the sharpness of the learned solutions for each degree and sparsity. Note that the sharpness of our construction indeed upper bounds the sharpness of the learned solutions at each degree and sparsity, by roughly two orders of magnitude. Interestingly, the increase in sharpness for degree 5 functions is not as steep as we would expect relative to our construction, suggesting that the structure of the solution changes at degree 5. We note that degree 5 were the highest degrees we were able to get to converge, and took many times longer than the lower degree functions; perhaps the longer time to converge explains this behavior.}
        \label{fig:sharpness comparison}
    \end{subfigure}

    \vspace{0.5em}
    \begin{subfigure}{\columnwidth}
        \includegraphics[width=\columnwidth]{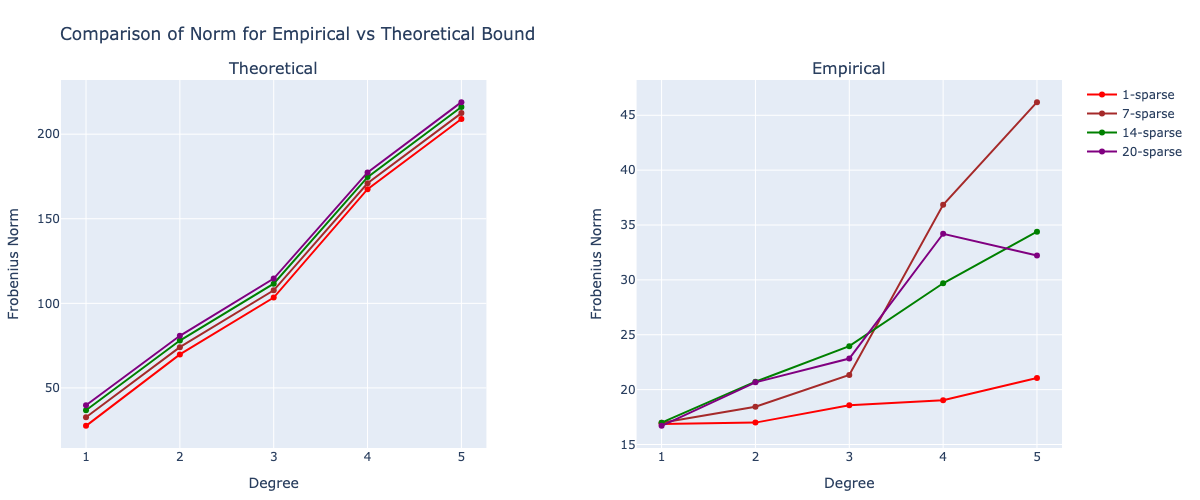}
        \caption{Left: the Frobenius norm of our exact constructions for each degree and sparsity. Right: the Frobenius norm of the learned solutions for each degree and sparsity. Note that the norm of our construction indeed upper bounds the norm of the learned solutions for each degree and sparsity, by around two orders of magnitude. Note that the models trained here are the same as that of both \autoref{fig:sharpness comparison} and \autoref{fig:generalizationgap}, with 11 samples from each (deg,width) function class, using 8192 samples for all models.}
        \label{fig:norm comparison}
    \end{subfigure}
    \caption{Sharpness and Frobenius norm comparisons between the exact construction and the learned solutions.}
    \label{fig:sharpness norm comparison}
\end{figure}

\subsection{Assessing the True Impact of Perturbations on Sharpness of our construction}
\label{perturbation analysis}
One of the main limitations of our approach is the difficulty in bounding the perturbation term in our bound. Since all of our perturbation analysis considered an isotropic perturbation across all parameters and used worst-case analysis, we might wonder whether this upper bound is overly pessimistic, or whether it even closely resembles the true impact of perturbations on the sharpness of the construction. 

To answer this question, we again make use of our hardcoded transformer that exactly reflects our mathematical construction, and perform a perturbation analysis. As in our analysis above, we assume that random projections are not being used due to the relatively small $T,$ instead using a standard one-hot positional encoding, which results in a total hidden dimension of $T+2.$ We then perform an isotropic perturbation of the weights and calculate the trace of the loss Hessian over the full training set. We perform this calculation over the same set of $\omega, D_f$ values in our other experiments, as well as over a selection of sequence lengths, $T=\{20,30,40,50,60\}$, so that we can see how the perturbed sharpness changes with sequence length. Finally, we repeat the entire calculation for 10 different random functions in each class. To estimate the empirical worst-case perturbation for each $(\sigma, \omega, D_f, T)$ at the $90^{th}$ percentile level, we then take the maximum over the 10 randomly generated functions in each group. Several plots showing the results are shown below.



\begin{figure*}[tbh]
    \includegraphics[width=\textwidth]{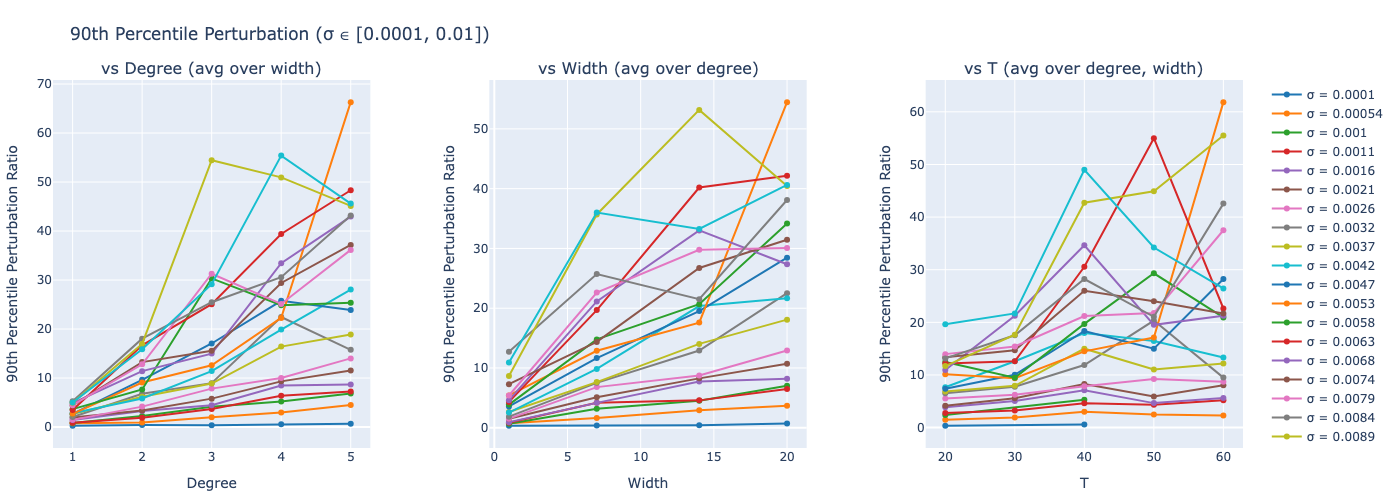}
    \caption{This plot shows the empirical $90^{th}-$percentile perturbation of the sharpness (the trace of the loss Hessian), and how it depends on the degree $D_f$ and sparsity $\omega$ of the target function it expresses, as well as the sequence length $T.$ each line is a different magnitude of the perturbation, and naturally the size of the perturbation always increases as $\sigma$ increases. While qualitatively similar to our analytic bound on $P(\sigma) $ in that the sharpness increases with $\omega, D_f,T,$ the magnitude of the perturbations are significantly smaller, and increase much more slowly, than our analytic bound would suggest.   }
    \label{fig:perturbations}
\end{figure*}

\begin{figure*}[tbh]
    \includegraphics[width=\textwidth]{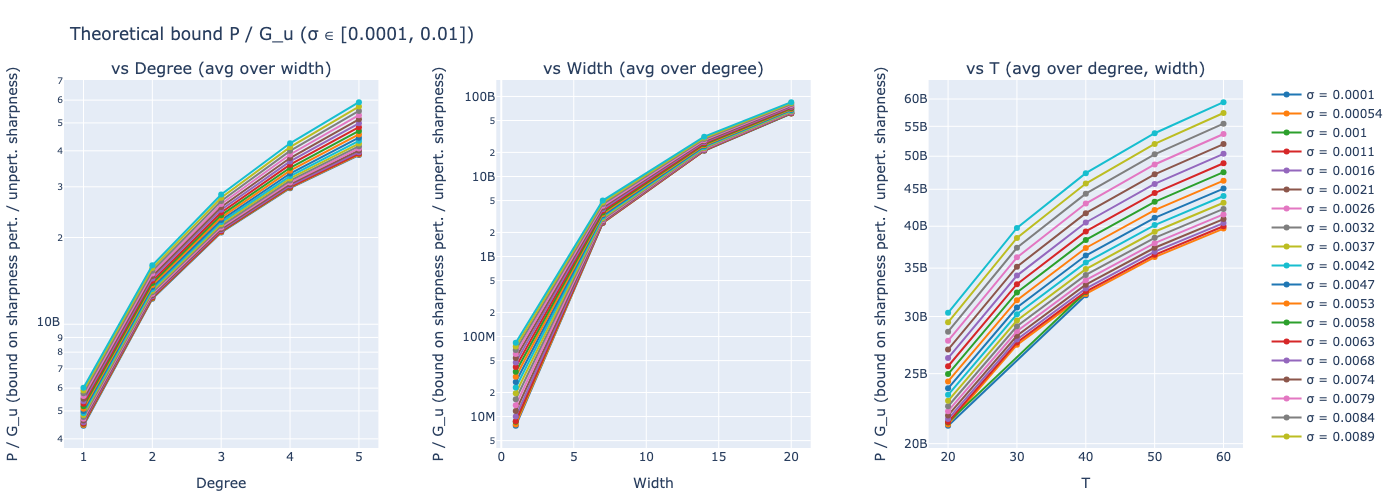}
    \caption{This plot shows our analytic bound on the perturbation to the sharpness over the same grid of $D_f. \omega, T$ as \autoref{fig:perturbations} above . Note that this plot is necessarily on a log-scale, due to the magnitude being astronomically larger than the empirical perturbation to the sharpness.  }
    \label{fig:perturbations-theoretical}
\end{figure*}

\subsection{Mechanistic Interpretability}\label{sec:mechint}

The architecture used in this section is simplified for interpretability and is not intended to instantiate the full construction analyzed in the theory; these experiments are illustrative rather than direct quantitative validations of the PAC–Bayes bound. While our bound does not directly rely  on the assumption that anything resembling our construction is actually learned in practice, it would nonetheless provide additional supporting evidence for the relevancy of our constructive approach. We conducted experiments with small transformer models closely matching the hyperparameters of our construction without random projections, designed with fixed embedding sizes of \( T+2 \), where the first \( T \) channels correspond to hard-coded one-hot positional embeddings, and the remaining channel represents one-hot encoded bit values. Additionally, we constrained the multi-layer perceptron (MLP) hidden dimension to 32 to ensure a compact model size matching our construction.  In order to be able to better visualize the behavior of the MLP, our mechinterp experiments projected down to a single dimension after the attention sub-layer, so that the MLP matrices are all $4(D_f+1)\times1$ or $1\times 4(D_f+1)$ vectors. Since this would be incompatible with a second attention layer, we replace the second layer with a simple $(T+1)\times 1$ projection that performs the final linear combination, rather than using a position-aware attention mechanism. The reason we replaced the attention based recombination stage with a linear layer was precisely to be able to make the MLP function interpretable, and see if we found any oscillations.  The model contained just over 1000 parameters total. We observed that in most cases, $\bar{W}^{(1)}$ contained rows corresponding to the Fourier components of its learned function, i.e. depicting bright spots in columns corresponding to the bits in a Fourier component. For example, in \autoref{fig:mechinterp}, one can identify rows corresponding clearly to the components $x_9x_{14}x_{17}$  and $x_3x_8x_{11}$, while the third component, with the smallest coefficient, is less identifiable. 

\begin{figure}[H]
    \includegraphics[width=1\columnwidth]{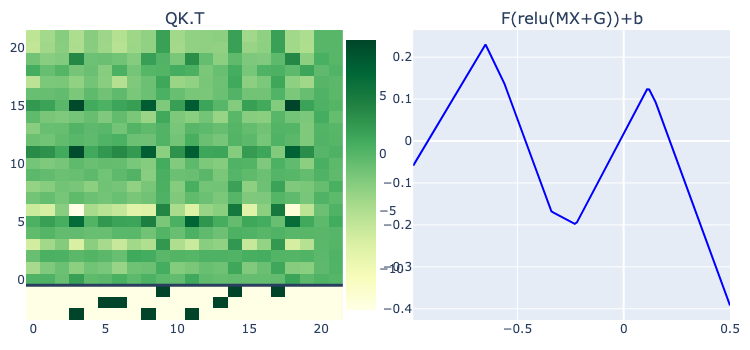}
    \caption{Left: A plot showing the combined attention matrix W for a function learned with an architecture matching our construction. Right: The learned 1-dimensional MLP function, which appears to exhibit cyclical behavior. While these aspects of this learned solution are similar to our construction, we emphasize that the learned solutions are often qualitatively different than our construction. This experiment is merely meant to support the feasibility of our construction -- not to argue that it is the norm. Close matches like this are rare. }
    \label{fig:mechinterp}
\end{figure}



\section{Result for Chain of Thought}\label{sec:cot}
We now apply our theoretical framework to understand the benefit of chain-of-thought for computing high-degree functions. It is well-documented that computing intermediate steps substantially boosts transformers' ability to learn to compute high-degree functions such as Parity \citep[e.g.][]{anil2022exploring, hahn-rofin-2024-sensitive, abbe2023generalization}.
As an explanation for this benefit, we now use our techniques to derive a favorable learning bound for Parity in the presence of chain-of-thought:
\begin{theorem}
\label{chain of thought}
Let $T \in \mathbb{N}$. Let $f_1, \dots, f_T$ where $f_i(x_1\dots x_T f_1\dots f_{i-1}) = \bigoplus_{j\leq i} x_j$. Let $\widehat{\Theta_{CoT}}(\alpha,\beta)$ be the solution returned by the general learning procedure on the training set of size $m$ jointly consisting of these functions using regularization parameters $\alpha, \beta$, and let $\delta_{CoT}(\alpha,\beta)$ be the expected final error for a chain of $T$ auto-regressive steps using the transformer parametrized by $\widehat{\Theta_{CoT}}$ to solve the Parity Task of length $T.$ Then there exist regularization parameters $\alpha_{CoT},\beta_{CoT}$ such that
$$\delta_{CoT}(\alpha_{CoT},\beta_{CoT}) \leq  4Te^{-\frac{m}{8\Sigma^2}}e^{\frac{m\sigma^2}{4\Sigma^2}\Big(2G_u(1,2) + P(\sigma,1, 2,T)\Big)+\frac{ L(1, 2,T) }{2\sigma^2}} $$
\end{theorem}
The proof is deferred to \autoref{chain of thought appendix} in the appendix. It follows a straightforward application of our generalization bound for $\omega=1$ and $D_f=2,T,$ and uses the fact that all of the functions $G_u(), P(), $ and $L()$ are superlinear in $T.$   We defer the reader to \cite{cabannes2024iteration} for a mechanistic interpretability study of similar, simple``iteration heads'' in transformers. In contrast, when learning to compute Parity without intermediate steps, a substantially less favorable bound remains:
\begin{theorem}
Let $\widehat{\Theta_{OP}}(\alpha, \beta)$ be the solution returned by the general learning procedure on the training set of size $m$ consisting of only the inputs and outputs of the complete Parity task of length $T,$ i.e. the functions $f(x_1\dots x_T)=\bigoplus_{j\leq T} x_j$.  Let $\delta_{OP}(\alpha, \beta)$ be the expected error for a single pass of a transformer parameterized by $\widehat{\Theta_{OP}}(\alpha, \beta)$ on the parity task of length $T.$ Then there exist regularization parameters $\alpha_{OP},\beta_{OP}$ such that
$$\delta_{OP} \leq  4e^{-\frac{m}{8\Sigma^2}}e^{\frac{m\sigma^2}{4\Sigma^2}\Big(2G_u(1,T) + P(\sigma,1, T,T)\Big)+\frac{ L(1, T,T) }{2\sigma^2}} $$
$$\leq  4e^{-\frac{m}{8\Sigma^2}}e^{\frac{m\sigma^2T}{4\Sigma^2}\Big(2G_u(1,2) + P(\sigma,1, 2,T)\Big)+\frac{T L(1, 2,T) }{2\sigma^2}} $$
In other words, our bound on the error for Parity increases exponentially with length when using the one-pass approach, whereas using the CoT approach, the error increases only linearly with T.
\end{theorem}

It is worth comparing to prior theoretical accounts for the benefit of chain-of-thought for Parity.
\citet{abbe2024how} argue that functions with high Globality Degree require a chain-of-thought to be learned efficiently, though the general version of this claim remains conjectural.
\citet{hahn-rofin-2024-sensitive} show that transformers face a very steep loss landscape when learning Parity, but do not show a converse positive result showing that chain-of-thought makes learning easy.
\citet{kim2024transformers} show that chain-of-thought helps transformers learn \emph{sparse (subset)} parities, whereas our result applies to the full Parity function applying to all input bits.

\section{Further Discussion}

\paragraph{Implications}
Recent work shows empirical and theoretical links between degree and generalization of transformers \cite{bhattamishra2022simplicity,abbe2023generalization,hahn-rofin-2024-sensitive}. Our result uses these properties in a generalization bound, supporting the degree as a complexity measure for transformers. Our work also introduces the Fourier sparsity as a secondary factor in generalization.
To our knowledge, this is the first generalization bound that relies on explicitly upper bounding the gradients and hessian of a given class of transformers. In \autoref{comparison}, we discuss the cases in which our approach could yield a tighter guarantee than a purely norm-based approach \cite{edelman2022inductive}.

We believe our overall approach might offer a promising path to explaining generalization for functions for which there is a known canonical transformer construction. For other kinds of algorithmic problems, one could use a canonical transformer construction that solves the problem (using e.g. the RASP program for the task \cite{weiss2021thinking}), and then analyze its norms and the trace of its Hessian. This might lead to more function-specific bounds than what is possible using the existing covering-number-based bounds \citep{edelman2022inductive, trauger2024sequence}. For example, one could analyze how the sharpness of the loss (the trace of the Hessian) for a small transformer recognizing Dyck-(k,D) varies with the key complexity parameters $k,D,$ and use this to create a sharpness-based generalization bound based on these complexity parameters. We leave these explorations of our approach applied to other formal languages and algorithmic tasks to future work.

Beyond the theoretical, our work has potential implications for safety-critical applications. For example, consider a transformer-based perception system in an autonomous vehicle that must evaluate a high-degree Boolean function over abstract features---a conjunction of many distinct conditions about pedestrians, weather, road surface, visibility, and traffic signals. Our theoretical results suggest that such high-degree conjunctive decisions are precisely the ones for which generalization is most fragile without chain-of-thought-style intermediate reasoning. This could inform how safety-critical AI systems are designed, trained, and evaluated---for instance, by motivating the decomposition of complex decision logic into lower-degree intermediate steps.

\subsection{Related Work}
The work of \cite{chiang2022overcoming} showed that, even though transformers can be explicitly constructed to learn highly sensitive (and high degree) boolean functions such as parity, hence are not limited by their expressivity, these functions remained difficult for a transformer to learn. In particular, they showed that the reason for this difficulty learning mechanistically speaking is that they induce a strongly oscillating loss curve in the immediately vicinity of some of the key parameters. Still other work such as \cite{liu2023shortcuts} showed an explicit construction for how parity was learned by transformers using a ``shortcut", which explicitly memorized the mod-2 function in the MLP. These two observations, when combined, led to the key insight behind our work: that for higher degree boolean functions, not only are the norms higher, but the oscillating function approximated by the MLP will have higher curvature, which should lead to a larger magnitude of the hessian than for lower-degree functions.

Some of our methods are inspired by techniques in \citet{deorao2023ptimization}, but substantial changes were needed. Their construction did not have an MLP or second attention layer like ours, and differed in some other ways, but demonstrated key techniques such as using finite differences to calculate the Gradients and the Hessian. In addition, the techniques used in bounding the perturbed gradients share many commonalities with the techniques used by \citet{edelman2022inductive} .

\subsection{Limitations}
\label{limitations}
One of the main limitations of our bound is the perturbation term $P(\sigma)$ appearing in the trace of the perturbed loss hessian. The analytic bound on the perturbation terms (see supplementary materials) is far from tight, as shown by comparing with our empirical estimates in \autoref{perturbation analysis}. Still, ideally we would want our bound to be fully-analytic that is still tight enough to be practical. One possible approach would be to use parameter- or parameter-group specific perturbation variances, so that parameters for which the scale is smaller or the overall sensitivity of the transformer is higher can have a smaller variance, without causing the KL term to blow up. This is a key way in which PAC-Bayes bounds are made tight, but would have significantly complicated our analysis. Another approach might be to use higher order approximations to the sharpness, moving the perturbation analysis to higher orders of $\sigma.$

Another particular weakness for the perturbation analysis is that the $|\Theta|$ term becomes prohibitive, even if we are using tricks like random projections. Using automated solvers to approximate the trace of the loss hessian for larger networks directly, rather than going through the operator norm, would be more scalable as it would not incur the $|\Theta|$ factor.

Beyond these mathematical limitations, we note several structural restrictions of the current analysis. First, our construction and bound apply to Boolean functions with all-positive Fourier coefficients, constant degree across components, and sparsity $\omega \leq T$. Second, the domination assumption (Section~\ref{lowsharpnessvalidation}) is an oracle-style structural assumption whose plausibility is supported empirically in the tested regime but has not been established in full generality. Third, our main learning experiments use a model that is intentionally larger than the minimal theoretical construction. While none of these restrictions are likely to be qualitative barriers---extending to negative coefficients, mixed degrees, and larger architectures would change the constants in our bounds but not the fundamental dependence on degree and sparsity---they define the scope of the current results. We also note that there is substantial evidence that circuits resembling our construction (e.g., approximately Boolean logic implemented via shortcut solutions \citep{liu2023shortcuts}) appear inside real transformers learning Boolean functions, suggesting that the same framework could be adapted to analyze specific circuits within larger architectures.

\paragraph{Discussion}
We have shown that transformers represent Boolean functions with low degree and low sparsity in flat minima, and deduced a generalization bound via PAC-Bayes methods. This result expands the understanding of the learning abilities and biases of the transformer architectures.
Our bound could be instantiated in practice without a priori knowing the target boolean function that has been learned by the transformer. While there are known approaches to estimating properties of the Fourier spectrum such as Monte Carlo methods or Random Fourier Features, these methods are considered expensive. Luckily, both of these properties (Fourier degree and sparsity) have known efficient property testing algorithms \cite{alon2003lowdegreetesting,gopalan2011sparsitytesting}, and therefore can be approximately upper bounded with high probability given very few input-output samples. More detail is given in \autoref{propertytesting} in the appendix.

\section*{Acknowledgments}
We thank Michael Hahn for detailed feedback on the manuscript, and Puneesh Deora and Pierre Alquier for their responsiveness and helpfulness.

\bibliography{biblio_conference_deduped}
\bibliographystyle{icml2026}

\appendix
\onecolumn

\tableofcontents

\section{Experimental Implementation Details}
\label{experimentaldetails}

For our experiments, we trained our simple transformer on 5 randomly sampled functions from each (sparsity,degree) complexity class, with the degree ranging from 1 to 5, and the sparsity in $\omega\in \{1,7,14,20\}$. We chose as our context length $T=20, $ and set $m=8192 $ , as this kept the problem small while still enabling convergence. All functions are trained until they have a training loss below 0.02. We found that there was significant variability in the hardness of learning of these random functions even within each function class, and our model did not converge randomly for some degree-5 functions, in which case we selected a new random function in the class.

In order to instantiate our bound, we needed to estimate the value of the subgaussian constant. We did so for several different functions trained to nearly perfect training loss, by analyzing the distribution of the validation losses, calculating the moment generating function, and then finding the smallest $\sigma$ such that a centered gaussian with standard deviation $\sigma$ has an MGF that strictly dominates that of our losses. We found $\Sigma$ empirically to fall in the [.01,.1] range. Thus we use $\Sigma=0.01$ as a reasonable setting for this parameter when instantiating our bound.

Training was performed using a dataset of 8192 randomly generated binary strings of length 20 (not including the CLS token) on an A100 GPUs using the DistributedDataParallel package with the "nccl" backend. We found that the rate and likelihood of convergence was strongly dependent on the batch size: the simple model struggled to learn degree-5 functions with batch sizes above 256. We used a learning rate of $1\times 10^{-3},$ a weight decay of 0.0001, and dropout of $0.1$. 

Our transformer was written using the pytorch library, and was based on the standard implementation of a 2-layer, 1-head transformer from the transformers python library.  However we made the appropriate modifications to match our construction above, including the removal of the layernorm, explicit encoding of the positional embedding to match our one-hot positional encoding, and log scaling of the attention logits. during an ablation study, we found that the results of our experiments did not change qualitatively when we added or removed the layernorm,

One difference between our construction above and the transformer used in our experiments is that our transformer uses both a larger hidden dimension and a larger MLP width than the theoretical construction. The hidden dimension used was 51 (21 positional encoding dimensions including the CLS token, plus 30 more dimensions), and the feed forward dimension was 128. The main reason for this was simply that we found convergence to be much faster with a larger model. Using these hyperparameters, it took roughly two days to train all 11 random functions from each of the 20 complexity classes.  We were able to confirm that for smaller degrees and sparsities, convergence was achieved with the minimum number of dimensions described by our construction, i.e. $N$ dimensions for the positional encoding, and only $2$ additional dimensions. Indeed, this minimal transformer exactly matching our construction is what was used during our mechanistic interpretability study, as well as the analysis of the sharpness and norm of the theoretical construction. Still, due to the hardness of learning functions with larger degrees and sparsities, we found a larger model to be necessary to obtain our main results showing generalization gap over various degrees and sparsities in a reasonable amount of time.

\section{Complete Construction}
We define a transformer that takes bitstrings of length $T$ as input. We reserve a special output position at the end of the sequence (position $T+1$), which is where the final output will be read from. Therefore, the actual input has $T+1$ positions. For each position $t\in [T+1],$ we define our input vector to be

$$y_t=\{I[t=1],...,I[t=T+1],z_t\} \in  \mathbb{R}^{T+2}$$
where $z_t \in \{0,1\}$ is the value of the $t^{th}$ bit. For the output position, we always fix $z_{T+1}=0.$ Note that in our initial construction, the input dimension is $T+2.$ We are working towards a construction that approximates this one with fewer dimensions, which we will define at the end of this section.

In order to calculate the value of each Fourier component $\chi_{S_t}$ at each position in the first layer, we will  make use of a purely position-aware attention mechanism with $O(log(T))$ scaling. For each position in the first $\omega$ rows of the attention matrix, this puts weights of approximately $\frac{1}{D_f}$ on all positions $j \in S_t,$ and $0$ elsewhere. Inactive positions $t\in[\omega+1,T+1]$ receive uniform attention weights in the first layer; their contributions are zeroed out by the second attention layer, which places zero weight on all positions $t > \omega.$
I.e.,

\begin{equation}
  a_{i,t} =
    \begin{cases}
      2\log(T) & \text{if } i \in S_t \wedge t\leq\omega\\
      0 & \text{otherwise}\\
    \end{cases}
\end{equation}
so that for active positions $t\leq\omega$:
\begin{equation}
  \hat{a}_{i,t} =
    \begin{cases}
      \frac{1}{\frac{T+1-D_f}{T^2}+D_f} & \text{if } i \in S_t\\
      \frac{1}{T+1-D_f+D_fT^2} & \text{if } i \notin S_t\\
    \end{cases}
\end{equation}
and for inactive positions $t>\omega,$ the attention weights are uniform: $\hat{a}_{i,t}=\frac{1}{T+1}$ for all $i.$ In the limit of $T$ large, this becomes
\begin{equation}
  \hat{a}_{i,t} =
    \begin{cases}
      \frac{1}{D_f} & \text{if } i  \in S_t \wedge t\leq\omega \\
      0 & \text{if } i \notin S_t \wedge t\leq\omega\\
      \frac{1}{T+1}  & \text{if } t> \omega \\

    \end{cases}
\end{equation}

\subsection{$\bar{Q}^{(1)}$, $\bar{K}^{(1)}$}
We let $\bar{Q}^{(1)},\bar{K}^{(1)} \in \mathbf{R}^{(T+2) \times (T+1)}$ 
We define $\bar{Q}^{(1)}$ to be a diagonal matrix which holds a scaling factor of $2log(T)$ on each of the diagonals.
\begin{equation}
  \bar{Q}^{(1)}_{s,t}=
    \begin{cases}
      2log(T) & s=t\\
      0 & \text{otherwise} \\
    \end{cases}       
\end{equation}

We define $\bar{K}^{(1)}$ to be a matrix such that the first $\omega$ columns contain inclusion indicators for each component in the first $T$ rows, with all remaining columns equal to $0.$ The final, $(T+2)^{th}$ column and row are all $0s.$
\begin{equation}
  \bar{K}^{(1)}_{s,t}=
    \begin{cases}
      1 & \text{if } s \in \chi_{S_t} \wedge t \in [\omega] \\
      0 & \text{otherwise} \\
    \end{cases}
\end{equation}

Note that $$(y_i^T\bar{Q}^{(1)})_t =\sum_{s=1}^{T+1} 2\log(T)I[i=s]I[s=t]$$
$$=2\log(T)I[i=t\}]$$
and 
$$(y_j^T\bar{K}^{(1)})_t = \sum_{s=1}^{T}I[j=s]I[s \in \chi_{S_{t}} \wedge t \in [\omega]] $$
$$ =I[j \in \chi_{S_{t}} \wedge t \in [\omega]]$$

Thus
$$a_{i,j} = (y_i^T\bar{Q}^{(1)})(y_j^T\bar{K}^{(1)})^T=2\log(T)\sum_{t=1}^{\omega}  I[i=t]I[j \in \chi_{S_{t}}]$$
$$=2\log(T)I[ j\in \chi_{S_{i}} \wedge i\leq\omega],  $$
which matches the desired behavior for attention weights in the first layer. The output of this step is a $(T+1) \times (T+1)$ attention matrix, where each of the first $\omega$ rows corresponds to the $i^{th}$ Fourier component, and has $\frac{1}{|\chi_{S_r}|}$ in column $j$ if position $j$ is a member of the $i^{th}$ Fourier component. All other rows (inactive positions $t\in[\omega+1,T+1]$) have zero attention logits and therefore receive uniform attention weights.

\subsection{Random projections for $\log(T)$ width}
Suppose we add a projection matrix into our attention mechanism: define $a^*_{i,j} =y_i^TJJ^T\bar{Q}^{(1)}(\bar{K}^{(1)})^TJJ^Ty_j,$ where $J \in \mathbb{R}^{T+2 \times d+1}.$ is an additional inner projection matrix with entries in the first $T+1$ rows and d columns that are  i.i.d. normally distributed with standard deviation $\frac{1}{\sqrt{d}}$, and a $1$ in bottom right element, and all else $0$ in the last row and column, i.e. $J_{ij}\sim\frac{1}{\sqrt{d}}\mathcal{N}(0,1) .$ for $i\in [T+1],j\in [d],$  It is an immediate consequence of the Johnson-Lindenstrauss Lemma that for any fixed $\epsilon_p>0,$  for any integer $d>\frac{8\log(T)}{\epsilon_p^2}.$
$$ Pr\Big[|\langle (\bar{Q}^{(1)})^TJJ^Ty_i,(K^{(1)})^TJJ^Ty_j\rangle-\langle (Q^{(1)})^Ty_i,K^Ty_j\rangle| < \epsilon_p \Big]>1-2e^{-\frac{d}{2}(\frac{\epsilon_p^2}{2}-\frac{\epsilon_p^3}{3})} = \delta_p$$
This trick allows us to define our projected inputs, $x_t=J^Ty_t\in \mathbb{R}^n,$ where $n\in O(\log(T)),$ using modified attention matrices $Q^{(1)}=J^T\bar{Q}^{(1)}$ and $K^{(1)}=J^T\bar{K}^{(1)}$ so that the overall output of the network is approximately the same as for the original network. 

\subsection{$\bar{V}^{(1)}$}
For our value matrix, we define $\bar{V}^{(1)}\in\mathbb{R}^{(T+2) \times (T+2)}$

\begin{equation}
  \bar{V}^{(1)}_{s,i} =
    \begin{cases}
      1 & \text{if s=T+2 and i=T+2} \\
      0 & otherwise\\
    \end{cases}       
\end{equation}
Thus 
$$y_i^T\bar{V}^{(1)}=\sum_{s=1}^{T+1} y_{i,s}\bar{V}^{(1)}_{s,T+1} = concat(\mathbf{0}_{T+1},z_i\}$$
Now, for notational simplicity we write $\bar{W}^{(1)} = \bar{Q}^{(1)}(\bar{K}^{(1)})^T,$ and we define the output at position $t$ after the first attention layer as $\bar{u}_t,$ where we are representing $\bar{u}_t$ as a $(T+2)\times 1$ column vector. In order to express this desired quantity as a column vector, we find it useful to transpose our transformer construction, so that 
$$\bar{u}_t=(\bar{V}^{(1)})^TY^T\phi(Y(\bar{W}^{(1)})^Ty_t) =  concat(\mathbf{0}_{T+1}, \frac{k_t}{D_f} \} \in  \mathbb{R}^{(T+2)}$$
We can easily modify our definition of $\bar{V}$ to work with our dimension-reduced inputs instead. We can replace $\bar{V}^{(1)}$ with $V^{(1)}=J^T\bar{V}^{(1)}J \in \mathbb{R}^{(d+1)\times (d+1)}$ and rewrite the output of our attention layer in terms of our dimension-reduced inputs $X=YJ$ as 
$$ u_t=(\bar{V}^{(1)})^TJJ^TY^T\phi(YJJ^T(\bar{W}^{(1)})^TJJ^Ty_t) \approx \bar{u}_t $$
We work with the attention output in the dimension reduced space:
$$b_t := J^Tu_t \in \mathbb{R}^{d+1}$$
$$=({V}^{(1)})^TX^T\phi(XW^{(1)}x_t) $$

We note that using our definition of $\bar{V^{(1)}},$ the only effect of applying the random projection is to reduce the matrix's dimensions: since all of the upper left $(T+1)\times(T+1)$ block is 0 anyways, and this is the only portion subject to random projections, this block simply gets replaced with a $d\times d$  block of zeros. This approximation allows us to achieve a correct construction with a width of only $d\in O(\log(T))$.

\subsection{$M,F,\Gamma$ in MLP layer}
Next, we make use of a wide MLP layer which interpolates the "mod 2" function, or alternatively, interpolates the function $$m: \mathbb{R}\to \mathbb{R}, m(x)=\frac{1}{2}\Bigg(1+sin\Big(\pi\Big(D_f*x+\frac{1}{2}\Big)\Big)\Bigg)$$
on the domain $x\in \{0,\frac{1}{D_f},...,\frac{D_f-1}{D_f},1\}.$ Below, we sometimes write $m(M)$ where M is a vector or a matrix. In these cases, we can understand this to mean $m$ applied element-wise.

Note that the maximum value for $k_t$ is $D_f$, and thus the domain of the function needing to be memorized has cardinality at most $D_f+1$. In order to memorize such a function, using a construction similar to that of \cite{liu2023shortcuts} or \cite{chiang2022overcoming}, we require $4(D_f+1)$ MLP units. This creates a first layer that acts as an indicator function $I[k=i]$ in a lookup table, with the second layer of the MLP storing the function values.

\begin{theorem}\label{mlpcorrectness}
Let $\bar{M}\in \mathbb{R}^{(T+2) \times 4(D_f+1)}$ ,  $\bar{\Gamma}  \in \mathbb{R}^{4(D_f+1)}$, and $\bar{F}\in \mathbb{R}^{(4(D_f+1))\times (T+2)}.$ 
For each $i \in [D_f], h_i \in \{0,\frac{1}{D_f},...,\frac{D_f-1}{|\chi_{S_r}|},1\},$
Let
$$\bar{M}_{i,t}=0,\forall t, i\in [T+1]$$
$$\bar{M}_{i,t}=1,\forall t, i =T+2 $$
$$\bar{\Gamma}_{4i-3} = -h_i-\frac{2}{4D_f}$$
$$\bar{\Gamma}_{4i-2} = -h_i-\frac{1}{4D_f}$$
$$\bar{\Gamma}_{4i-1} = -h_i+\frac{1}{4D_f}$$
$$\bar{\Gamma}_{4i} = -h_i+\frac{2}{4D_f}$$

and let
$$\bar{F}_{i,t}=0,\forall i\in [4(D+f+1)], t\in[T+1]$$
$$\bar{F}_{4i-3,T+2}=  4m(h_i)D_f\forall i\in [D_f]$$
$$\bar{F}_{4i-2,T+2}=-4m(h_i)D_f\forall i\in [D_f] $$
$$\bar{F}_{4i-1,T+2}= -4m(h_i)D_f\forall i\in [D_f] $$
$$\bar{F}_{4i,T+2}=4 m(h_i)D_f \forall i\in [D_f]$$

Then the function
$$g_t = \bar{F}^T(\bar{M}^Tb_t+\bar{\Gamma})_+ $$
interpolates
$$\approx \frac{1}{2}\Bigg(1+sin\Big(\pi\Big(D_f*b_t+\frac{1}{2}\Big)\Big)\Bigg)=\frac{1}{2}\Bigg(1+sin\Big(\pi\Big(k_t+\frac{1}{2}\Big)\Big)\Bigg),$$
on the domain of possible values of $b_t, $ a grid of the possible values of our scaled prefix sums, $h_i$. 
\end{theorem}

Proof: 
First, we look at the product $\bar{M}^Tb_t.$ 
$$\big[\bar{M}^Tb_t\big]_{i} = \frac{k_t}{D_f}\implies \bar{M}^Tb_t=\frac{k_t}{D_f}\mathbb{1}_{4(D_f+1)}^T$$
It is easy to see that 
$$max\Big(\Big[\bar{M}^Tb_t+\bar{\Gamma}\Big]_{4i-2},0\Big) = \frac{(k_t-(i-\frac{6}{4}))_+}{D_f}$$
$$max\Big(\Big[\bar{M}^Tb_t+\bar{\Gamma}\Big]_{4i-2},0\Big) = \frac{(k_t-(i-\frac{5}{4}))_+}{D_f}$$
$$max\big(\Big[\bar{M}^Tb_t+\bar{\Gamma}\Big]_{4i-1},0\Big) = \frac{(k_t-(i-\frac{3}{4}))_+}{D_f}$$
$$max\Big(\Big[\bar{M}^Tb_t+\bar{\Gamma}\Big]_{4i},0\Big) = \frac{(k_t-(i-\frac{2}{4}))_+}{D_f}$$
Finally, we calculate the product as three separate sums:
$$max(b_t\bar{M}+\bar{\Gamma}^T,0)F  $$
$$=4D_fm(h_i)\sum_{i=1}^{D_f+1}  (b_t-\bar{\Gamma}_{4i-3})_+-(b_t-\bar{\Gamma}_{4i-2})_+ -  (b_t-\bar{\Gamma}_{4i-1})_+ +  (b_t-\bar{\Gamma}_{4i})_+$$
$$=4m(h_i)\sum_{i=1}^{D_f+1}  (k_t-(i-\frac{6}{4})_+- (k_t-(i-\frac{5}{4})_+-(k_t-(i-\frac{3}{4})_++(k_t-(i-\frac{2}{4}))_+$$
The above expression is only nonzero on the integers when $k_t=i-1,$ in which case its value is $\bar{F}^Tmax(\bar{M}^Tb_t+\bar{\Gamma},0) =m(h_i),$ proving the claim.  The function is a piecewise-linear `trapezoidal' wave alternating between 0 and 1.

Let us denote this trapezoidal wave function of our MLP by $h(b_t) = \bar{F}^T(\bar{M}^Tb_t+\bar{\Gamma})_+.$ The derivative $h'(b_t)$  is easily shown to be a square wave that alternates between $-4D_f, 0$ and $4D_f$. The second derivative of $h$ is 0 almost everywhere. We will see later that we avoid dealing with the second derivative entirely, by using first order Taylor approximations and finite differences rather than exact derivatives. After the MLP layer, we now have the even-parity indicator $\chi_{S_t}(x) = \frac{1+(-1)^{k_t}}{2}\in\{0,1\}$  stored in each post-MLP activation, $h_t,$ as desired.

Just as the other matrices were compressed using random projections, we similarly can define the dimension-reduced matrices $M=J^T\bar{M}$ and  $F=\bar{F}J.$ We note that in the case of the MLP matrices, the random projection portion of the $J$ matrix has no effect other than to shrink the size of the $0$ blocks of these matrices from having one dimension of size $T+1$ to that dimension now having size $d$, while maintaining their original values in the final row and column, respectively.

We will find it useful to refer to the MLP outputs in the dimension-reduced space,
$$g_t = J^T h_t \in \mathbb{R}^{(d+1)\times 1 }$$

\subsection{$\bar{W^{(2)}}$,$\bar{V^{(2)}}$}
To get out final output, we linearly combine the individual Fourier components from the MLP function, $g_t.$  To do so, we implement a final transformer ``half-layer'', or, a layer consisting of just an attention sub-layer (but no MLP sub-layer). In this sub-layer, The inner attention projection matrix for the full dimension inputs, $\bar{W}^{(2)}$, is a $(T+2)\times(T+2)$ matrix with only the penultimate row nonzero. This row contains $\log(c_tT^2),$ in the first $\omega$ entries, and 0 elsewhere. The value matrix in the second layer is $\bar{V}^{(2)}, $ which is a $(T+2)\times 1$ vector which selects the final dimension from the incoming vector and scales it by a factor of $D =\sum_{t=1}^{\omega}c_t.$

Note that the second layer's pre-softmax attention output at the $(t,s)$ position is zero in all but the last row, which has entries given by 
$a^{(2)}_{T+1,t} =g_{T+1}^T \hat{W^{(2)}} g_t= \log(c_tT^2) , \forall t\leq \omega.$  Taking the softmax of this first row of the attention matrix yields:
$$\hat{a}^{(2)}_{T+1,t} = \frac{e^{\log(c_t T^2) }}{T+1-\omega +\sum_{l=1}^{\omega} e^{\log(c_lT^2)}}=\frac{c_tT^2}{T+1-\omega + T^2 \sum_{l=1}^{\omega}c_l}=\frac{c_t}{\frac{T+1-\omega}{T^2}+D} \to \frac{c_t}{D}$$
Finally, the $\bar{V}^{(2)}$ is defined to have $0s$ everywhere except the final dimension, which has a value of $D$, so that the final activation is equal to: 

$$\sum_{t=1}^T \hat{a}^{(2)}_{T+1,t}g_t^T\bar{V}^{(2)}=\sum_{t=1}^T \frac{c_t}{D}g_{t,T+2}D=\sum_{t=1}^T c_tg_{t,T+2}$$
Recall from above that in the final dimension of the output of the MLP, $g_{t,T+2},$ is stored the parity of component $t$ for the input string. Thus the above formula matches the definition of our target function.

The dimension-reduced versions of $\bar{V}^{(2)}$ and $\bar{W}^{(2)}$ are obtained through the same overall pattern as the other matrices:
$$V^{(2)} = J^T \bar{V}^{(2)} , W^{(2)} = J^T\bar{W}^{(2)} J$$

Our complete transformer output can then be written in a series of the following ``blocks'' corresponding to the three sub-layers involved in our construction:

$$\mathcal{T}(X,W,V,M,F,\Gamma, W^{(2)},V^{(2)}) =(V^{(2)})^TG^T(\phi(G(W^{(2)})^Tg_{T+1}))  $$
where
$$g_t:=F^T\Big(M^Tb_t+\Gamma\Big)_+, \space G := \begin{bmatrix}
    g_1^T \\
    \dots \\
    g_{T+1}^T
\end{bmatrix}\in \mathbb{R}^{(T+1) \times (d+1)}$$
$$b_t =(V^{(1)})^TX^T \phi\big(X^T(W^{(1)})^Tx_t \big), \space B := \begin{bmatrix}
       b_1^T \\
    \dots \\
    b_{T+1}^T
\end{bmatrix} \in \mathbb{R}^{(T+1) \times (d+1)}
 $$
Note that we have a residual connection after the MLP, but not after the first attention layer; in other words, the matrix $B$ is missing the (randomly projected) positional encoding bits in the first $d$ dimensions, unlike $X$ and $G.$

\begin{theorem}\label{norm bounds appendix}
Norms of the weight matrix. 
Using the above construction, we have the following bound on the (vector) 2-norm of our weight matrix:
$$\|\mathbf{\Theta}\|_F^2 \leq L(\omega, D_f,T) := 16(D_f+1)D_f^2+8(D_f+1)+1+4\log^2(T)D_f\omega + 4\omega\log^2(T) + \omega$$
$$\in O\Big(D_f^3+\log(T)^2\omega D_f\Big)$$
\end{theorem}
\begin{proof}

First, we consider the MLP matrix $F.$ F consists of 1 column with maximum absolute value of $4D_f.$ In fact, every 4-tuple of parameters corresponding to a unique function value in the domain has total squared sum of $16D_f^2.$ There are $(D_f+1)$ such quadruplets, and therefore the total squared 2-norm of this matrix is exactly $16(D_f+1)D_f^2.$ 

Next, consider the bias vector $\Gamma.$ This vector has $4(D_f+1)$ entries, each of which is upper bounded by 1. Therefore $\|\Gamma\|_F^2 \leq 4(D_f+1).$

The matrix $M$ consists of 1 row, which is all 1s. Therefore, $\|M\|_F^2=4(D_f+1)$

For the matrix $\bar{V}^{(1)}, $ all entries except the bottom right are zero. For the dimension-reduced version $V^{(1)}=J^T\bar{V}^{(1)}J,$ the same is true, and that the matrix is now $(d+1)\times (d+1)$ with a $1$ in the bottom right. Thus $\|V^{(1)}\|=1.$


The matrix $\bar{W}^{(1)}\in \mathbb{R}^{(T+2)\times (T+2)}$ has in each of its first $\omega$ columns,  $D_f-$many elements  with value $2log(T)$ representing the positional encodings of each Fourier component. All other columns have all zeros, yielding a Frobenius norm of $\|\bar{W}^{(1)}\|_F=2log(T)\sqrt{D_f\omega}.$   The dimension-reduced matrix $W^{(1)}=J^T\bar{W}^{(1)}J$ will have a Frobenius norm whose expectations is $\|\bar{W}^{(1)}\|_F+o(\frac{1}{k}).$ To see this, we start with the definition of the Frobenius norm:
$$\|J^T\bar{W}^{(1)}J\|_F^2 = \sum_{r,s=1}^d \Big(J_{:,r}^T\bar{W}^{(1)}J_{:,s}\Big)^2=\frac{1}{d^2}\sum_{r,s=1}^d \Big(\hat{J}_{:,r}^T\bar{W}^{(1)}\hat{J}_{:,s}\Big)^2,$$
where $\hat{J}$ is the un-scaled projection matrix, i.e. $\hat{J} = \sqrt{d} J.$ To evaluate this, we consider the diagonal and off-diagonal terms separately. When $r=s, $ the Isserlis/Wick theorem tells us that the fourth moment of the quadratic guassian form is given by
$$\mathbb{E}\big[\big(\hat{J}_{:,r}^T\bar{W}^{(1)}\hat{J}_{:,r}\big)^2\big]=2\|\bar{W}^{(1)}\|_F^2 + Tr(M)^2.$$
For $r\neq s, $ we have 
$$\mathbb{E}\big[\big(\hat{J}_{:,r}^T\bar{W}^{(1)}\hat{J}_{:,s}\big)^2\big]=\|\bar{W}^{(1)}\|_F^2 .$$
Since there are $d(d-1)$ terms where $r\neq s$ and $k$ terms where $r=s,$ we put these together to obtain
$$\|\hat{J}^T\bar{W}^{(1)}\hat{J}\|_F^2=\frac{1}{d^2}\Big(d\big(2\|\bar{W}^{(1)}\|_F^2 + Tr(M)^2\big)+d(d-1)\big(\|\bar{W}^{(1)}\|_F^2\big)\Big)$$
$$=\|\bar{W}^{(1)}\|_F^2+\frac{1}{d}\Big(\|\bar{W}^{(1)}\|_F^2 + Tr(\bar{W}^{(1)})^2\Big).$$
Thus, for $d$ large, we can approximate
$$\|W^{(1)}\|_F^2=\|J^T\bar{W}^{(1)}J\|_F^2 \approx \|\bar{W}^{(1)}\|_F^2 =4log^2(T)D_f\omega$$

We can take a similar approach to bounding $\|W^{(2)}\|^2_F.$ By the same arguments, this is well approximated by $\|\bar{W}^{(2)}\|^2_F.$  The matrix $\bar{W}^{(2)}$ has all $0s$ except for $\omega $ entries of $\log(c_t T^2)$ in the $T+1^{th}$ (out of $T+2$) column. Thus the squared Frobenius norm of this matrix is bounded as 
$$\|W^{(2)}\|_F^2\lessapprox \sum_{t=1}^\omega log^2(c_t T^2) \leq 4\omega log^2( T)$$
Finally, we bound $\|V^{(2)}\|_F^2.$ Note that, like $\bar{V}^{(2)},$ $V^{(2)}$ is zero in all but the final dimension, on which $J^T$ acts as an identity map. Since the only nonzero entry of $\bar{V}^{(2)}$ has value $D \leq \omega,$ we have
$$\|V^{(2)}\|_F^2 \leq \omega$$

Adding these together, we have that 
$$\|\mathbf{\Theta}\|^2 \leq $$
$$\underbrace{16(D_f+1)D_f^2}_{\|F\|_F^2}+\underbrace{4(D_f+1)}_{\|\Gamma\|_F^2} +\underbrace{4(D_f+1)}_{\|M\|_F^2 + }+\underbrace{1}_{\|V^{(1)}\|_F^2}+\underbrace{4log^2(T)(D_f\omega )}_{\|W^{(1)}\|}+\underbrace{4\omega log^2(T)}_{\|W^{(2)}\|_F^2}+\underbrace{\omega}_{\|V^{(2)}\|_F^2}$$

$$\approx 16D_f^3 +4log^2(T)(D_f\omega )$$

$$\in O\Big(D_f^3+\log(T)^2\omega D_f\Big)$$
  \end{proof}

\section{Bounds on Norms of Parameters and Intermediate Variables}

\begin{theorem} \label{parameterbounds}
Let $\mathcal{T}(X,\Theta): \mathbb{R}^{(T+1) \times (d+1)} \to \mathbb{R}$ be the 1-layer transformer as defined above. Let $f: \{0,1\}^T\to \mathbb{R}$ be a function with a Fourier spectrum as described above: with maximum degree $D_f,$ and with a sparsity $\omega$ at most $T$. Then we have the following bounds on the input and parameter norms: 
 
$$\|X\|_{2,\infty} \lessapprox  \sqrt{2},$$
$$\|XV^{(1)}\|_{2,\infty} \lessapprox 1.$$
$$\|XV^{(1)}e_{d+1}\|_{\infty} \lessapprox 1.$$
$$\|G\|_{2,\infty} \lessapprox \sqrt{2}$$
$$\|G\|_{1,\infty} \lessapprox 1$$
$$\|g_{T+1}\| \lessapprox 1$$
$$\|GV^{(2)}\|_{\infty} \lessapprox \sqrt{\omega}$$
$$\|GV^{(2)}\| \lessapprox \omega$$
$$\|V^{(2)}\|\lessapprox \sqrt{\omega}$$
$$\|B\|_{2,\infty} \lessapprox  1$$
$$\|B\|_{1,\infty}  \lessapprox 1.$$
$$\|(W^{(2)})^Tg_{T+1}\|\lessapprox 2\sqrt{\omega} \log(T)$$
$$\|G(W^{(2)})^T\|_{2,\infty}\lessapprox 2log(T) $$
$$\|W^{(2)}\|_{2} \lessapprox 2\sqrt{\omega}\log(T)$$
\end{theorem}

\begin{proof}
$\|X\|_{2,\infty}$ is the maximum row 2-norm. $X=YJ$ is a $((T+1) \times (d+1))$ matrix where each entry in the $t^{th}$ row is a gaussian R.V. with standard deviation $\frac{1}{\sqrt{d}},$ with the exception of the final column, which stores the bit value. Therefore the 2-norm of each row of $YJ$ is upper bounded by
$$\|Y_{t,:}J\| = \sqrt{z_t^2 + \sum_{j=1}^d \Big(\sum_{i=1}^dY_{t,i}J_{i,j}}\Big)^2=\sqrt{z_t^2 + \sum_{j=1}^d J_{t,j}^2}$$
$$\leq \sqrt{1 + \sum_{j=1}^d J_{t,j}^2}$$
Note that $\sum_{j=1}^d J_{t,j}^2$ is a chi-squared R.V. with $d$ degrees of freedom. By the standard chi-squared tail bound \citep{laurentmassart2000} (see also \autoref{chidistributedboundlemma} in the supplementary material), with probability $1-\delta, $ such an R.V. is upper bounded as

$$\|\epsilon_{t,:}\|\leq \frac{1}{\sqrt{d}} \sqrt{d+2\sqrt{dlog(\frac{1}{\delta})}+2log(\frac{1}{\delta})}$$

To upper bound the maximum of $T$ such variables, we apply a union bound:
$$\max_{t \in T}\|\epsilon_{t,:}\|\leq \frac{1}{d} \Big(d+2\sqrt{dlog(\frac{T}{\delta})}+2log(\frac{T}{\delta})\Big)$$
Note that per our definition of $d,$ $d>\frac{8log(T)}{\epsilon_p^2}.$ Therefore, with probability at least $1-\delta,$
$$\max_{t \in T}\|\epsilon_{t,:}\|\leq 1+2\epsilon_p\sqrt{\frac{\log(\frac{T}{\delta})}{8 \log(T)}}+2\epsilon_p^2 \frac{\log(\frac{T}{\delta})}{8log(T)}$$
$$= 1+\frac{\epsilon_p}{\sqrt{2}}\sqrt{1-\frac{\delta}{ \log(T)}}+\frac{\epsilon_p^2}{4} \Big(1-\frac{\delta}{\log(T)}\Big)$$
$$\lessapprox 1+\frac{\epsilon_p}{\sqrt{2}}+\frac{\epsilon_p^2}{4}$$
Since we can make $\epsilon_p$ arbitrarily small, the above expression approach 1, and then we can approximate $\|X\|_{2,\infty} \lessapprox \sqrt{1+1} = \sqrt{2}$ (in the case that the bit value $z_t$ is 1).  To approximately upper bound $\|X\|_{1,\infty},$ the maximum element of $YJ,$ we simply consider the composition of the matrix $X,$ which has the upper left $(T\times d) $ sub-matrix as gaussian random variables with standard deviation $\frac{1}{\sqrt{d}},$ and then a final column of the bit values. Since $d\in O(\log(T)),$ the probability of all terms being less than 1 is overwhelmingly likely. The expectation of the maximum of m random variables is known to be $\sigma\sqrt{2log(m)} + o(1)$ With $\sigma=\frac{1}{\log(T)}$ and $m=Tlog(T),$ this becomes $\frac{1}{\log(T)\sqrt{2log(Tlog(T))}}\to 0$ for large $T.$ Thus the expectation is in $o(1).$ Therefore we approximate $\|X\|_{1,\infty} \lessapprox 1.$

To bound $\|XV^{(1)}\|_{2,\infty},$ note that $Y\bar{V}^{(1)}$ is a matrix similar to Y, in that the final column contains all of the bit values of $Y$, but all previous columns which held the positional encodings are $0$. Therefore, $Y\bar{V}^{(1)}$ could have a maximum row 2-norm of $1.$. The matrix $XV^{(1)}=YJJ^T\bar{V}^{(1)}$ approximates $Y\bar{V}^{(1)},$ since the error from our random projections approximation $\epsilon_p$ can be made arbitrarily small. In particular, 
$$\|XV^{(1)}\|_{2,\infty}=\|YJJ^T\bar{V}^{(1)}\|_{2,\infty}\leq\|Y\bar{V}^{(1)}\|_{2,\infty}+\|YJJ^T\bar{V}^{(1)}-Y\bar{V}^{(1)}\|_{2,\infty} $$
$$\leq \|Y\bar{V}^{(1)}\|_{2,\infty}+\|Y\||_{2,\infty}|JJ^T\bar{V}^{(1)}-\bar{V}^{(1)}\|_2\leq\|Y\bar{V}^{(1)}\|_{2,\infty}+\epsilon_p\|Y\|_{2,\infty}\|\bar{V}^{(1)}\|_2,$$
$$\leq \|Y\bar{V}^{(1)}\|_{2,\infty}+\sqrt{2}\epsilon_p$$
where the second to last inequality follows from JLL, and the last step follows from our bound on $\|Y\|_{2,\infty},$ and the fact that $\bar{V}^{(1)}$ has only a single nonzero value, which is 1, and therefore its maximum singular value can be no larger than 1. Thus,
$$\|XV\|_{2,\infty} \lessapprox 1$$
Similarly, we consider the $\|XV^{(1)}e_{d+1}\|_{\infty},$ which is approximated by $\|Y\bar{V}^{(1)}e_{d+1}\|_{\infty}$ due to JLL. the maximum value in the final column of $\bar{Y}V^{(1)},$ as described above, is 1. Thus 
$$\|XV^{(1)}e_{d+1}\|_{\infty}\lessapprox 1$$
We now consider the norms relating to our MLP output matrix, $G\in\mathbb{R}^{(T+1)\times(d+1)}.$  Let  $\bar{H}\in\mathbb{R}^{(T+1)\times (T+2)}$ be the full-dimension MLP output matrix using the full-dimensional parameter matrices. The maximum column $2-$norm of $\bar{H}$ is at most $\sqrt{T+1},$ since there could be $1s$ in any element of the final column. The maximum row 2-norm is $\sqrt{2}$ because there are at most 2 nonzero bits per row. The output of the dimension reduced transformer $H\in\mathbb{R}^{(T+1)\times (T+2)}$ can be made arbitrarily close to $\bar{H},$ and thus we have $$\|H\|_{1,2} \lessapprox \sqrt{T+1}$$
$$\|H\|_{2,\infty}\lessapprox \sqrt{2}$$ To bound these norms for the dimension-reduced MLP output matrix $G,$ we note that $G=HJ.$ and that the overall structure of $H$ is exactly the same as that of $Y$ above: each row has two possible nonzero elements: one for the positional encoding, and another variable bit value in the $T+2^{th}$ column. therefore we can apply the exact same arguments to conclude that 
$$\|G\|_{2,\infty}=\|HJ\|_{2,\infty}\lessapprox \|H\|_{2,\infty} = \sqrt{2}$$
We can apply this argument once more to approximately upper bound the maximum element of $G:$
$$\|G\|_{1,\infty}  =  \|HJ\|_{1,\infty} \lessapprox \|H\|_{1,\infty} = 1.$$
'
We will also find it useful to bound the 2-norm of the final post-MLP activation, $\|g_{T+1}\|.$ Note that the $d+1^{th}$ dimension for the output position $T+1$ is always $0$ (since the bit value for this position is fixed at $0$). Therefore $\|g_{T+1}\|$ is a chi-distributed variable with $d$ degrees of freedom and standard deviation $\frac{1}{\sqrt{d}}$. Thus with high probability, for large $d$ it becomes tightly concentrated around $1.$ We thus approximate $\|g_{T+1}\| \lessapprox 1.$

We consider the maximum row 2-norm of the matrix $GV^{(2)}.$ We can argue the same way as above in our bound on $\|XV^{(1)}\|_{2,\infty} ,$  that $\|GV^{(2)}\|_{\infty} =\|HJJ^T\bar{V}^{(2)}\|_{\infty}$ approximates $\|H\bar{V}^{(2)}\|_{\infty}$ arbitrarily well. Since $H\bar{V}^{(2)}$ is a vector containing a copy of the final column of $H,$ which are the MLP outputs in $\{0,1\},$ but scaled by a factor of $D,$ we can conclude that $\|H\bar{V}^{(2)}\|_{\infty} \leq D$. Therefore we conclude that 
$$\|GV^{(2)}\|_{\infty} \lessapprox D $$
Now, note that $$D=\sum_{t=1}^{\omega} c_t \leq \sum_{t=1}^{\omega} |c_t|\leq \sqrt{\omega} $$
Thus 
$$\|GV^{(2)}\|_{\infty} \lessapprox\sqrt{\omega} $$
Noting that the final column of $G$ is nonzero for only $\omega$ entries, we use Cauchy-Schwarz to conclude that 
$$\|GV^{(2)}\| \lessapprox \omega $$
Bounding $\|V^{(2)}\|$ is trivial, since $\|V^{(2)}\|=\|J^T\bar{V}^{(2)}\|\leq D\|J_{T+2,:}\| \leq D,$ where we have used the fact that the final, $T+2^{th}$ column of $J$ is $e_{d+1}^T,$ which has norm 1. Thus $\|V^{(2)}\|\leq D\leq \sqrt{\omega}.$

Note that the matrix $U$ has a similar overall structure as both $Y$ and $H$ above, except that due to the lack of residual connection after the first attention layer, $B$ does not contain any positional encodings. Following the similar arguments with this assumption, we can bound 
$$\|B\|_{2,\infty}=\|UJ\|_{2,\infty}\lessapprox \|U\|_{2,\infty} = 1$$
$$\|B\|_{1,\infty}  =  \|UJ\|_{1,\infty} \lessapprox \|U\|_{1,\infty} = 1.$$

Similarly, JLL says that $\|G(W^{(2)})^T\|_{2,\infty}$ approximates $\|H(\bar{W}^{(2)})^T\|_{2,\infty}$ arbitrarily well. Since the first $T$ columns and final column of $(\bar{W}^{(2)})^T$ are $0s,$ the same is true of $H(\bar{W}^{(2)})^T.$ In the penultimate column of this matrix, we have the vector $(\log(c_1T^2),\dots,\log(c_{\omega}T^2) ,0,\dots,0).$ Each row therefore has a 2-norm that is bounded by $\log(T^2)=2log(T).$ Thus $$\|G(W^{(2)})^T\|_{2,\infty}\lessapprox 2log(T)$$Following similar arguments, it follows from JLL that   $\|(W^{(2)})^Tg_{T+1}\|$ approximates $\|(\bar{W}^{(2)})^Th_{T+1}\|$ arbitrarily well. Recall that the matrix $(\bar{W}^{(2)})^T$ is only nonzero in the penultimate column, which has $\log(c_t T^2) $ in its first $\omega$ entries. $h_{T+1}$ is only nonzero in its positional encoding bit, which is the penultimate bit, and therefore the vector $(W^{(2)})^Tg_{T+1}=(\log(c_1T^2),\dots,\log(c_{\omega}T^2) ,0,\dots,0).$ It follows that $\|(W^{(2)})^Tg_{T+1}\|\leq \sqrt{\sum_{t=1}^{\omega} \log(c_t T^2)^2}\leq\sqrt{\sum_{t=1}^{\omega} 4log( T)^2}\leq 2\sqrt{\omega} \log(T)$

$\|W^{(2)}\|_{2}=\|J^TW^{(2)}J\|_{2},$ the operator norm of $W^{(2)},$ approximates $\|\bar{W}^{(2)}\|_{2}.$ Recall that $\bar{W}^{(2)}$ is nonzero only in its penultimate row, which has $\log(c_t T^2)$ in the first $\omega$ columns. Let $l:=(\log(c_1T^2),\dots,\log(c_{\omega}T^2),0,\dots,0)$ be the vector representing this nonzero row. it is easy to show that $(\bar{W}^{(2)})^T\bar{W}^{(2)} = ll^T.$ There is one nonzero eigenvalue of this rank-1 outer product, which is $\|l\|^2.$ Thus the maximum singular value is $\|l\|.$ We can easily bound this using Cauchy-Schwarz to conclude that $\|\bar{W}^{(2)}\|_{2} \leq 2\sqrt{\omega}\log(T).$
\end{proof}
\section{Theoretical Reasons why our Construction is not Pareto Optimal}
\label{non-pareto}
The low--sharpness interpolator assumption implicitly relies on the explicit construction not being Pareto--optimal with respect to (parameter norm, sharpness). If the construction were Pareto--optimal, one could not guarantee the existence of an interpolator with both strictly smaller norm and strictly smaller sharpness.

There is strong reason to believe that our construction is far from Pareto--optimal. Empirically (Figs.~\ref{fig:norm comparison}, \ref{fig:sharpness comparison}), a transformer with the same architecture trained on the same target functions achieves substantially lower norm and lower sharpness than the explicit construction. Beyond this empirical evidence, the construction admits strict analytic improvements.

\paragraph{(1) Structural domination via rescaling.}
Recall that the MLP block is given by
\[
g_t = F^T (M^T b_t + \Gamma)_+ .
\]
For any scalar $s>0$, define the reparameterization
\[
M^{(s)} = sM, 
\qquad 
\Gamma^{(s)} = s\Gamma,
\qquad 
F^{(s)} = \frac{1}{s}F,
\]
with all other matrices unchanged.

Since ReLU is positively homogeneous, $(su)_+ = s(u)_+$, we obtain
\[
(F^{(s)})^T\big((M^{(s)})^T b_t + \Gamma^{(s)}\big)_+
=
F^T(M^T b_t + \Gamma)_+,
\]
and therefore $\mathcal T(X,\Theta^{(s)}) = \mathcal T(X,\Theta)$ for all inputs $X$.

Under this reparameterization, the MLP contribution to the squared Frobenius norm becomes
\[
\|M^{(s)}\|_F^2 + \|\Gamma^{(s)}\|_F^2 + \|F^{(s)}\|_F^2
=
s^2 A + \frac{1}{s^2} B,
\]
where $A=\|M\|_F^2+\|\Gamma\|_F^2$ and $B=\|F\|_F^2$.  
The function $s^2A + B/s^2$ is strictly convex and minimized when the two terms are equal. Unless the original construction already satisfies $A=B$, there exists $s\neq 1$ that strictly decreases the parameter norm.

The same scaling applies to the gradients:
\[
\nabla_M \mapsto \frac{1}{s}\nabla_M,
\qquad
\nabla_\Gamma \mapsto \frac{1}{s}\nabla_\Gamma,
\qquad
\nabla_F \mapsto s\,\nabla_F.
\]
Hence the MLP contribution to the squared gradient norm becomes
\[
\frac{1}{s^2}G + s^2 H,
\]
where $G=\|\nabla_M\|_F^2+\|\nabla_\Gamma\|_F^2$ and 
$H=\|\nabla_F\|_F^2$.  
Again this expression is minimized when the two terms are balanced.  
In our construction, both the parameter magnitudes and the gradient bounds indicate that the second--layer MLP matrix contributes more heavily than the first--layer matrices. A single choice of $s>1$ therefore rebalances both quantities simultaneously, yielding a strict reduction in both norm and sharpness while leaving the realized function unchanged.

\paragraph{(2) Interpolators need only fit the training set.}
More fundamentally, the PAC--Bayes argument requires only an interpolator on the training set, not a globally exact representation of the target function. The explicit construction represents the function for all $x\in\{0,1\}^T$, including inputs such as the all--zeros or all--ones vectors. For training sets of size polynomial in $T$, such inputs occur with exponentially small probability.

Consequently, one may remove or attenuate the corresponding components of the MLP without affecting the training loss with high probability. This produces an interpolator that matches the training data exactly while strictly reducing the parameter norm and preserving sharpness on the training set.

Taken together with the empirical results, these structural and probabilistic considerations strongly suggest that the explicit construction is not Pareto--optimal in $(\text{norm},\text{sharpness})$ space, supporting the plausibility of the low--sharpness interpolator assumption.

\section{PAC-Bayes Bound Derivation}
\subsection{Average Sharpness Approximation}
\begin{theorem}\label{average sharpness approximation appendix}
Suppose our transformer $\mathcal{T}(X,\Theta+\epsilon)$ approximately represents $f$ using a parameter set $\Theta+\epsilon$, where $epsilon$ a normal perturbation $\epsilon\in\mathcal{N}(0,\sigma^2)^{|\Theta|},$ where $|\Theta|$ is the parameter count. Define $\hat{L}(f_{\mathbf{\Theta}})$ to be the empirical (training) average quadratic loss of the transformer with parameters $\Theta$ evaluated against the target function $f.$
The expected empirical loss under $\epsilon$ decomposes exactly as:
\begin{align*}
    \mathbb{E}_{\epsilon\sim \mathcal{N}(0,\sigma^2)^n}\Big[\hat{L}(f_{\Theta+\epsilon})\Big] =
    \frac{\sigma^2}{2}Tr\Big(\nabla^2 \big[\hat{L}(f_{\Theta})\big]\Big)
    + \frac{1}{2}\mathbb{E}_{\epsilon}\Big[\epsilon^T\Big(\nabla^2\big[\hat{L}(f_{\Theta+\zeta(\epsilon)})\big]
    - \nabla^2\big[\hat{L}(f_{\Theta})\big]\Big)\epsilon\Big]
\end{align*}
where $\zeta(\epsilon)\in \mathbb{R}^n$ are the ``remainder perturbations'', i.e. $\zeta_i(\epsilon) \in [0,\epsilon_i]$ is the Lagrange remainder point making the second-order Taylor expansion exact, conditional on each Gaussian perturbation $\epsilon_i$. Note that the classical Taylor expansion and Hessian are well-defined here because $\sigma$ is chosen small enough that no ReLU activations change sign under perturbation with high probability (see Appendix~\ref{mlpperturbationlemma}), so the network remains within a single linear region where it is smooth.
\end{theorem}

\begin{proof}
We start off by taking a second-order Taylor expansion of $L(f_{\mathbf{\Theta+\epsilon}})$ with the Lagrange form of the remainder:
$$\hat{L}(f_{\Theta+\epsilon})-\hat{L}(f_{\Theta}) = \nabla \hat{L}(f_{\Theta})\epsilon+\frac{1}{2}\epsilon^T\nabla^2 \big[\hat{L}(f_{\Theta+\zeta})\big]\epsilon ,$$
where $\zeta(\epsilon)\in \mathbb{R}^n$ are the "remainder perturbations", i.e. $\zeta_i \in [0,\epsilon_i]$ is some small perturbation of the parameters around the true minimum that makes the Taylor expansion exact, conditional on each Gaussian perturbation $\epsilon_i$. Since the unperturbed transformer exactly represents the target function, we have $\hat{L}(f_{\mathbf{\Theta}})=0$ Assuming our function has attained a local minimum, we can set the first term on the RHS above to 0. Therefore, we can simplify the equation to:
$$\hat{L}(f_{\Theta+\epsilon})=\frac{1}{2}\epsilon^T\nabla^2 \big[\hat{L}(f_{\Theta+\zeta})\big]\epsilon $$
We now add and subtract the Hessian at the unperturbed transformer:
$$\hat{L}(f_{\Theta+\epsilon})=\frac{1}{2}\epsilon^T\nabla^2 \big[\hat{L}(f_{\Theta})\big]\epsilon+\frac{1}{2}\epsilon^T\Big(\nabla^2 \big[\hat{L}(f_{\Theta+\zeta})\big]\ -\nabla^2 \big[\hat{L}(f_{\Theta})\big]\Big)\epsilon $$
Taking the expectation of both sides with respect to $\epsilon,$ we get
$$\mathbb{E}_{\epsilon \sim \mathcal{N}(0,\sigma^2)^{|\Theta|}}\Big[\hat{L}(f_{\Theta+\epsilon})\Big]$$
$$=\frac{1}{2}\mathbb{E}_{\epsilon \sim \mathcal{N}(0,\sigma^2)^{|\Theta|}}\Big[\epsilon^T\nabla^2 \big[\hat{L}(f_{\Theta})\big]\epsilon\Big]+\frac{1}{2}\mathbb{E}_{\epsilon \sim \mathcal{N}(0,\sigma^2)^{|\Theta|}}\Big[\epsilon^T\Big(\nabla^2 \big[\hat{L}(f_{\Theta+\zeta})\big]\ -\nabla^2 \big[\hat{L}(f_{\Theta})\big]\Big)\epsilon\Big]$$
Since $\nabla^2[\hat{L}(f_\Theta)]$ is a fixed matrix independent of $\epsilon$, Hutchinson's Trace Estimator applies to the first term: $\frac{1}{2}\mathbb{E}_\epsilon[\epsilon^T\nabla^2[\hat{L}(f_\Theta)]\epsilon] = \frac{\sigma^2}{2}Tr\Big(\nabla^2[\hat{L}(f_\Theta)]\Big)$. The second term cannot be further simplified via Hutchinson since $\zeta(\epsilon)$ depends on $\epsilon$. This yields:
$$\mathbb{E}_{\epsilon \sim \mathcal{N}(0,\sigma^2)^{|\Theta|}}\Big[\hat{L}(f_{\Theta+\epsilon})\Big] = \frac{\sigma^2}{2}Tr\Big(\nabla^2\big[\hat{L}(f_{\Theta})\big]\Big) + \frac{1}{2}\mathbb{E}_{\epsilon}\Big[\epsilon^T\Big(\nabla^2\big[\hat{L}(f_{\Theta+\zeta(\epsilon)})\big] - \nabla^2\big[\hat{L}(f_{\Theta})\big]\Big)\epsilon\Big]$$
\end{proof}

\subsection{Details of PAC-Bayes Bound Derivation}
\label{pac-bayes discussion}
Recall that $L_f(X,\Theta)$ is the  loss for our transformer learning the function $f$ on input $X.$ Define $L(f_{\mathbf{\Theta}})$ to be the global loss, while $\hat{L}(f_{\mathbf{\Theta}})$ is the empirical loss on our training set of size $m.$ Importantly, the PAC-Bayes bounds we invoke hold simultaneously for all posteriors $\mathcal{Q}$ with high probability over the training sample, so they remain valid when $\mathcal{Q}$ is chosen data-dependently---as in our case, where the posterior is a Gaussian centered at the learned interpolator $\hat{\Theta}$ \citep{alquier2021pacbayes}.

The original bound of Macallester (1998) claimed that starting with any data-independent prior distribution $P$ over hypotheses, and a data-dependent (specifically, the outcome of training a neural network) posterior distribution $Q,$ for any $\delta>0$, the following holds with probability at least $1-\delta:$

$$\mathbb{E}_{\Theta \sim Q}[{L}(f_{\Theta})] \leq \mathbb{E}_{\Theta\sim Q}[\hat{L}(f_{\Theta})] + \sqrt{\frac{KL(Q\|P)+ln(\frac{m}{\delta})}{2(m-1)}}$$

Neyshabur et al.\ (2017) instantiate McAllester's bound with a Gaussian posterior $Q = \mathcal{N}(\Theta, \sigma^2 I)$, writing $f_{\Theta+\epsilon}$ for the posterior distribution of predictors where $\epsilon \sim \mathcal{N}(0,\sigma^2)^n.$ Using $KL(Q\|P) = \|\Theta\|^2/(2\sigma^2)$ and decomposing the test loss into training loss plus sharpness gap gives:
$$\mathbb{E}_{\epsilon \sim \mathcal{N}(0,\sigma^2)^n}\big[L(f_{\Theta+\epsilon})\big] \leq \hat{L}(f_{\Theta}) + \underbrace{\mathbb{E}_{\epsilon \sim \mathcal{N}(0,\sigma^2)^n}\big[\hat{L}(f_{\Theta+\epsilon})\big]-\hat{L}(f_{\Theta})}_{sharpness} +\sqrt{\frac{1}{m}\Big(\underbrace{\frac{\|\Theta\|^2}{2\sigma^2}}_{norm}+\ln\frac{2m}{\delta}\Big)}$$
where $0<\delta<1$ is our probability of the bound being violated, $m$ is the number of training points, and $\sigma$ is the standard deviation of our parameter perturbation (assumed to be equal for all parameters in this version  of the bound). Note that the prior distribution $P$ is assumed to be a centered Gaussian with the same noise variance as the posterior distribution $Q$ around its mean, $\Theta$.

Our case is slightly different from this, as we are not using a binary loss bounded within $[0,1]$.Rather, we consider functions $f$ on a boolean domain, but with a real-valued range, and a natural choice of our loss is the quadratic loss, i.e. $L_f(X,\Theta)=(T(\Theta,X)-f(X))^2.$ Several variants of the original PAC-Bayes bounds by McAllester have been derived that allow us to adapt Neyshabur's perturbation bound to our present situation.

\cite{alquier2021pacbayes} provide a bound intended for losses that are possibly unbounded, but have a tail distribution that is sub-gaussian. Mathematically, we assume that for some positive constant $\Sigma,$ our loss function satisfies:
$$\mathbb{E}[e^{t[l(\Theta,x,f)-\mathbb{E}[l(\Theta,x,f)]]}]\leq e^{\frac{\Sigma^2 t^2}{2}}.$$
Then with probability at least $1-\delta, $ the following bound holds:

$$\mathbb{E}_{\Theta \sim Q}[{L}(f_{\Theta})] \leq \mathbb{E}_{\Theta \sim Q}[\hat{L}(f_{\Theta})] + \frac{\lambda \Sigma^2}{2m}+\frac{KL(Q\|P)+\ln\frac{1}{\delta}}{\lambda}$$
In order to estimate the sub-gaussian constant for a trained model, we calculate the MGF of our validation loss distribution, and find the lowest $\Sigma$ such that the Gaussian MGF still dominates the true loss MGF. In our case, we found that for models trained to a loss of $0.01$ on functions of various degrees and sparsities, $\Sigma$ generally falls in the range $[0.01,0.1]$. Instantiating with our Gaussian perturbation posterior $Q = \mathcal{N}(\Theta,\sigma^2 I)$ and prior $P = \mathcal{N}(0,\sigma^2 I)$, so that $KL(Q\|P)=\frac{\|\Theta\|^2}{2\sigma^2}$:
$$\mathbb{E}_{\epsilon \sim \mathcal{N}(0,\sigma^2)^n}\Big[L(f_{\Theta+\epsilon})\Big] \leq \mathbb{E}_{\epsilon \sim \mathcal{N}(0,\sigma^2)^n}\Big[\hat{L}(f_{\Theta+\epsilon})\Big] + \frac{\lambda \Sigma^2}{2m}+\frac{\frac{\|\Theta\|^2}{2\sigma^2}+\ln\frac{1}{\delta}}{\lambda}$$
Applying Theorem~\ref{average sharpness approximation appendix} to expand $\mathbb{E}_\epsilon\big[\hat{L}(f_{\Theta+\epsilon})\big]$:
$$\mathbb{E}_{\epsilon \sim \mathcal{N}(0,\sigma^2)^n}\Big[L(f_{\Theta+\epsilon})\Big] \leq \frac{\sigma^2}{2}Tr\Big(\nabla^2\big[\hat{L}(f_{\Theta})\big]\Big) + \frac{1}{2}\mathbb{E}_{\epsilon}\Big[\epsilon^T\Big(\nabla^2\big[\hat{L}(f_{\Theta+\zeta(\epsilon)})\big] - \nabla^2\big[\hat{L}(f_{\Theta})\big]\Big)\epsilon\Big] + \frac{\lambda \Sigma^2}{2m}+\frac{\frac{\|\Theta\|^2}{2\sigma^2}+\ln\frac{1}{\delta}}{\lambda}$$
Bounding $\|\Theta\|_F^2 \leq L(\omega,D_f,T)$ via Theorem~\ref{norm bounds appendix} then gives us the generalization bound:
$$\mathbb{E}_{\epsilon \sim \mathcal{N}(0,\sigma^2)}[{L}(f_{\Theta+\epsilon})] \leq \sigma^2\Big(G_u(\omega,D_f) + \frac{1}{2}P(\sigma)\Big)   + \frac{\lambda \Sigma^2}{2m}+\frac{\frac{L(\omega, D_f,T)}{2\sigma^2}+ln\frac{1}{\delta}}{\lambda}$$

minimizing the RHS with respect to $\lambda, $ we obtain 
$$\frac{\Sigma^2}{2m}=\frac{\frac{L(\omega, D_f,T)}{2\sigma^2}+ln\frac{1}{\delta}}{\lambda^2}\implies \lambda =\sqrt{ \frac{2m}{\Sigma^2}\Big(\frac{L(\omega, D_f,T)}{2\sigma^2}+ln\frac{1}{\delta}\Big)}$$
$$\implies \frac{\lambda \Sigma^2}{2m}+\frac{\frac{L(\omega, D_f,T)}{2\sigma^2}+ln\frac{1}{\delta}}{\lambda} =2\sqrt{ \frac{\Sigma^2}{2m}\Big(\frac{L(\omega, D_f,T)}{2\sigma^2}+ln\frac{1}{\delta}\Big)}$$
Hence our bound becomes:

$$\implies\mathbb{E}_{\epsilon \sim \mathcal{N}(0,\sigma^2)}[{L}(f_{\Theta+\epsilon})] \leq \sigma^2\Big(G_u(\omega,D_f) + \frac{1}{2}P(\sigma)\Big)  + 2\sqrt{ \frac{\Sigma^2}{2m}\Big(\frac{ L(\omega, D_f,T)}{2\sigma^2}+ln\frac{1}{\delta}\Big)},$$
The natural next step in simplifying our bound is to optimize over $\sigma.$  If we were to momentarily ignore the perturbation term $P(\sigma)$ and the $ln\frac{1}{\delta}$ term, we would end up with the following formula for the optimal $\sigma^*:$
$$\sigma^*=\Big(\frac{\Sigma^2L(\omega, D_f,T)}{4mG_u(\omega,D_f)^2 }\Big)^{\frac{1}{6}}$$
If we use the full bound including the perturbation term, the optimal $\sigma*$ ends up being smaller than without the perturbation term included. We return to the analysis of the optimization of $\sigma$ and tightness of the full bound after a brief digression that continues with the analysis of this truncated bound in order to contextualize our overall approach against the behavior of previous bounds for transformers on boolean domains. 

\begin{figure*}[tbh]
    \includegraphics[width=\textwidth]{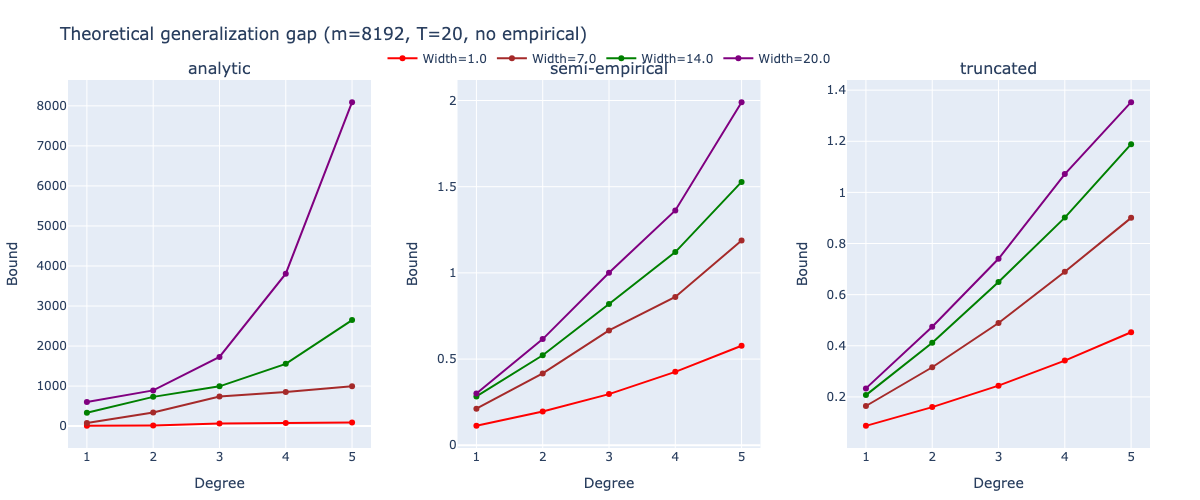}
    \caption{This plot shows a comparison of the bound with analytic worst-case bound on $P(\sigma)$ (left), the bound with the perturbed sharpness term $P(\sigma)$ simply truncated (middle), and the bound with the empirical worst case estimate of $P(\sigma)$ (right). Note that the fully analytic bound is much looser than the semi-analytic and truncated bounds, which are relatively close to each other. Note also that the fully analytic bound introduces a stronger dependency on both the degree $D_f$  and width $\omega$ than the other bounds, indicating that these strong dependencies are likely an artifact of our worst-case analysis of $P(\sigma).$ All bounds are optimized over $\sigma$ separately for each $(D_f,\omega)$ combination.   }
    \label{fig:boundcomparison}
\end{figure*}

\subsection{Comparison with Norm-Based Bound of Edelman et. al. (2022)}
\label{comparison}
If we were to plug the optimal $\sigma$  into our bound (still excluding the perturbation term), it would yield an optimized bound that grows like $O(m^{-\frac{1}{3}})$.  If we compare this to a purely norm-based capacity bound such as in \cite{edelman2022inductive} for sparse boolean functions, our norm term alone shares the $O(m^{-\frac{1}{2}})$ dependency. However, the sharpness term makes our bound $O(m^{-\frac{1}{3}}) $ overall, and it is therefore asymptotically dominated by the norm bound. As with many PAC-Bayes style bounds, the hope is that the \textit{local }flatness of the loss landscape around a learned minimum (low $G_u$) can be exploited in order to obtain relatively small sharpness constants even for large $\sigma,$ keeping both the sharpness and norm terms numerically small for small $m$. While offering a strong baseline, norm-based bounds depend exclusively on global complexity parameters such as parameter norms, and are independent of the curvature of the loss landscape. 

We now instantiate our semi-analytic generalization bound and compare it directly with the Edelman-style covering-number bound under identical architectural assumptions.  We consider a target function on $T$ bits that is a
linear combination of a tiling of $2$-parities.  In this case
\[
D_f = 2, \qquad \omega = \frac{T}{2}.
\]
This function is not Boolean-valued, but both our PAC-Bayes analysis and the norm-based analysis of
\cite{edelman2022inductive} apply unchanged.  Throughout this subsection we fix
\[
T = 20, 
\qquad 
\Sigma = 0.01,
\qquad
m = 10^6,
\qquad
\delta = 0.05,
\]
and we use the same architectural budgets for both bounds: a $1$-head, $2$-layer transformer without layer
normalization or random projections, so that the total hidden dimension including the $T+1$ positional encoding dimensions and the aggregation dimension is $T+2$.

\subsubsection{Our PAC-Bayes Bound}
As shown above, the relatively small size of the perturbation term relative to the unperturbed sharpness means that the semi-analytic bound remains close to the bound with the $P(\sigma)$ term simply truncated. Thus, for expository purposes and in order to provide a fully analytic analysis, we presented analyze the bound with the $P(\sigma)$ term truncated. With the above hyperparameter settings, our truncated PAC-Bayes bound specializes to
\begin{equation}
\label{eq:ours_main}
\mathrm{gap}_{\mathrm{ours}}(\sigma,m)
=
\underbrace{\sigma^2\, \big(G_u(\omega,D_f)\big)}_{\text{Sharpness term}}
+
\underbrace{
2\sqrt{\frac{\Sigma^2}{2m}
\left(
\frac{L(\omega,D_f,T)}{2\sigma^2}
+
\ln\frac{1}{\delta}
\right)}
}_{\text{Norm term}}.
\end{equation}
The sharpness coefficient is
\[
G_u(\omega,D_f)
=
4 + 4\omega\big(2 + D_f + 32D_f^2 + 32D_f^3\big).
\]
For $D_f=2$ and $\omega=T/2=10$,
\[
G_u = 4 + 4\cdot 10 (2+2+128+256) = 15524.
\]

The norm proxy appearing in~\eqref{eq:ours_main} is
\[
L(\omega,D_f,T)
=
16(D_f+1)D_f^2 + 8(D_f+1) + 1
+ 4\ln^2(T)D_f\omega + 4\omega\ln^2(T) + \omega,
\]
and for $(D_f,\omega,T)=(2,10,20)$ this evaluates to
\[
L \approx 1303.93
\]

Substituting these values into~\eqref{eq:ours_main} yields the explicit one-dimensional objective
\[
\mathrm{gap}_{\mathrm{ours}}(\sigma,10^6)
=
15524\,\sigma^2
+
2\sqrt{
\frac{10^{-4}}{2\cdot 10^6}
\Big(
\frac{ 1303.93}{2\sigma^2} + \ln 20
\Big)}.
\]

Optimizing over $\sigma>0$ numerically gives
\[
\sigma^\star \approx 2.27\times 10^{-3}.
\]
Evaluating each term at $\sigma^\star$:
\[
\text{Sharpness} = 15524(\sigma^\star)^2 \approx 0.080,
\qquad
\text{Norm} \approx 0.159.
\]
Thus our truncated bound yields
\begin{equation}
\label{eq:ours_final_num}
\mathrm{gap}_{\mathrm{truncated}}(m=10^6)
\approx
0.239.
\end{equation}
If we instead include our empirical worst-case estimate of the perturbation term in order to obtain our full semi-analytic bound, without even re-optimizing $\sigma,$ we obtain

\begin{equation}
\label{eq:ours_semi}
\mathrm{gap}_{\mathrm{semi-analytic}}(m=10^6)
\approx
0.327.
\end{equation}

\subsubsection{Edelman et. al. for our Construction }
In Theorem~\ref{norm bounds appendix}, we obtained explicit Frobenius norm bounds for each parameter
matrix in our construction.  To compare with the norm-based covering-number analysis of 
\cite{edelman2022inductive}, we convert each Frobenius bound to a $(2,1)$ bound using
\[
\|W\|_{2,1}
=
\sum_{j=1}^c \|W_{:,j}\|_2
\;\le\;
\sqrt{c}\,\|W\|_F,
\qquad
W\in\mathbb{R}^{r\times c}.
\]

From Theorem~\ref{norm bounds appendix} and the structure of the construction, we obtain:
\begin{align*}
\|F\|_{2,1} &= 4D_f\sqrt{D_f+1},\\
\|\Gamma\|_{2,1} &\le 2\sqrt{D_f+1},\\
\|M\|_{2,1} &= 4(D_f+1),\\
\|V^{(1)}\|_{2,1} &= 1,\\
\|W^{(1)}\|_{2,1}
&\le 
\sqrt{T+2}\cdot 2\log(T)\sqrt{D_f\omega + (T+1-\omega)},\\
\|W^{(2)}\|_{2,1}
&\le
\sqrt{T+2}\,\sqrt{2\omega\log(T)},\\
\|V^{(2)}\|_{2,1} &\le \sqrt{\omega}.
\end{align*}

For $(D_f,\omega,T,d) = (2,10,20,22)$, these yield
\[
C_{21}(20,22)
=
\|F\|_{2,1}
+\|\Gamma\|_{2,1}
+\|M\|_{2,1}
+\|V^{(1)}\|_{2,1}
+\|W^{(1)}\|_{2,1}
+\|W^{(2)}\|_{2,1}
+\|V^{(2)}\|_{2,1}
\approx 226.26.
\]

\subsubsection{Edelman-Style Covering-Number Bound}

The generalization bound of \cite{edelman2022inductive}, when instantiated with our architectural
hyperparameters and $(2,1)$-norm bounds, takes the form
\begin{equation}
\label{eq:edelman_main}
\mathrm{gap}_{\mathrm{Edelman}}(m)
\;\lesssim\;
C_{21}(T,d)\,
\sqrt{\frac{\ln(d\,m\,T)}{m}}.
\end{equation}
For $T=20,\ d=22,\ m=10^6$,
\[
\sqrt{\frac{\ln(d\,m\,T)}{m}}
=
\sqrt{\frac{\ln(4.4\times 10^8)}{10^6}}
\approx 4.46\times 10^{-3},
\]
and therefore
\begin{equation}
\label{eq:edelman_final_num}
\mathrm{gap}_{\mathrm{Edelman}}(m=10^6)
\approx
226.26 \times 4.46\times 10^{-3}
\approx 1.01.
\end{equation}

\subsubsection{Summary of the Comparison}

For the $T=20$ tiling of $2$-parities described above, and using identical architectural and norm
budgets, we obtain:

\[
\mathrm{gap}_{\mathrm{ours}}(m=10^6)
\approx 0.327,
\qquad
\mathrm{gap}_{\mathrm{Edelman}}(m=10^6)
\approx 1.01.
\]

Thus, in this concrete low-degree, width-$\omega=T/2$ example, our PAC-Bayes flatness-based bound is
numerically tighter than the Edelman-style covering-number bound, despite the latter using the same
model class and norm budgets. Note that this choice of target function intentionally leverages another advantage of our bound: that it naturally applies for functions with low-degree but non-sparse Fourier spectra.

As mentioned above, including the perturbation term introduces a significant numerical penalty in our bound, driving the optimal $\sigma$ lower, making our bound much less competitive. For target functions of high sharpness, the constants may be so large that there may not be any $m$ below which our bound has an advantage. If our perturbation analysis were tightened --  for example by using perturbations with variances that are parameter-specific -- we might be able to overcome this issue.

\section{Bounding the Trace of the Perturbed Loss Hessian}
\begin{theorem}
\label{ExactTraceBoundsAppendix}
    Let $f$ be a target boolean function of sparsity $\omega\leq T$ and maximum degree $D_f$, and let $\mathcal{T}(X,\Theta)$ be a transformer of context length T implementing it exactly according to our construction. Let $L_f(X,\Theta)=(T(\Theta,X)-f(X))^2$ be the unbounded, quadratic loss for our transformer learning the function $f.$  Then under the quadratic loss, the trace of the loss Hessian is given by:
$$Tr\Big(\nabla^2 L_f(X,\Theta)\Big)=2\| \nabla \mathcal{T}(X,\Theta)\|^2 \leq 2G_u(\omega, D_f)\in O(\omega D_f^3)$$
\end{theorem}
\begin{proof}
    For a quadratic loss, the gradient of the loss is given by:
$$\nabla L_f(X,\Theta+\zeta) = 2(\mathcal{T}(X,\Theta)-f(X)) \nabla \mathcal{T}(X,\Theta)$$
The Hessian is then given by:
$$\nabla^2L_f(X,\Theta)=+2 \nabla \mathcal{T}(X,\Theta)\nabla \mathcal{T}(X,\Theta)^T+ 2(\mathcal{T}(X,\Theta)-f(X))\nabla^2 \mathcal{T}(X,\Theta)$$
The second term on the RHS vanishes when the test loss is 0, i.e. $f(X)=\mathcal{T}(X,\Theta).$ Therefore
$$Tr\Big(\nabla^2L_f(X,\Theta)\Big)=+2Tr\Big(\nabla \mathcal{T}(X,\Theta)\nabla \mathcal{T}(X,\Theta)^T\Big)=  2\Big\|\nabla \mathcal{T}(X,\Theta)\Big\|^2$$
The second equality follows from the fact that the argument to the trace function is a rank-1 outer product, and thus has only a single nonzero eigenvalue making the operator norm and the trace both equal to $\Big\|\nabla \mathcal{T}(X,\Theta)\Big\|^2.$ In \ref{unperturbed gradient bounds appendix},  we establish that 
$\|\nabla \mathcal{T}(X,\Theta)\|^2 \leq G_u(\omega, D_f),$ and thus 

$$Tr\Big(\nabla^2L_f(X,\Theta)\Big) \leq 2G_u(\omega, D_f)\in O(\omega D_f^3)$$

\end{proof}
As we can see from the theorem above, the key ingredient in bounding the unperturbed transformer's loss is the norm of the gradients. We do so in the following theorem.
\begin{theorem}
\label{unperturbed gradient bounds appendix}
Suppose our transformer construction with context length $T$ is exactly representing some function $f$ with maximum degree $D_f$ and sparsity $\omega.$ Recall that $D=\sum_{t=1}^{\omega}c_t\leq\sqrt{\omega}$ is the sum of Fourier coefficients.
Then the following bounds on the gradient norms hold uniformly for all $X$:

$$\|\nabla_{\Theta}\mathcal{T}(X,\Theta)\|^2_F \lessapprox  2 + 8D^2 + 64\omega D_f^3 + 4D_f\omega +64\omega D_f^3+2+128\omega D_f^2$$
$$=4+4\omega(2+ D_f + 32D_f^2 + 32 D_f^3)$$
$$\triangleq G_u(\omega,D_f)\in O(\omega D_f^3)$$

\end{theorem}
\begin{proof}
Recall that our complete transformer function is given by
  
$$\mathcal{T}(X,W^{(1)},V^{(1)},M,F,\Gamma, W^{(2)},V^{(2)}) =(V^{(2)})^TG^T(\phi(G(W^{(2)})^Tg_{T+1}))  $$
where
$$g_t:=F^T\Big(M^Tb_t+\Gamma\Big)_+, \space G := \begin{bmatrix}
    g_1^T \\
    \dots \\
    g_{T+1}^T
\end{bmatrix}\in \mathbb{R}^{(T+1) \times (d+1)}$$
$$b_t =( V^{(1)})^TX^T\phi(X(W^{(1)})^Tx_t ) , \space B := \begin{bmatrix}
       b_1^T \\
    \dots \\
    b_{T+1}^T
\end{bmatrix} \in \mathbb{R}^{(T+1) \times (d+1)}
 $$
Before we calculate all of the gradients, we first calculate an intermediate derivative that will be useful throughout the following sections, and will demonstrate some of the common techniques used in this paper. We first take the gradient of the transformer with respect to $G.$ Note that we can write $g_{T+1} = (e_{T+1}^TG)^T = G^Te_{T+1}.$ $G$ still contains the one-hot positional encodings in the first $d$ dimensions from the residual stream, but these are not dependent on any of the parameter matrices, and thus could not contribute to the gradient of the output with respect to of any of the parameter matrices. Rather, it is only the gradients with respect to $G_{t,d+1}$ that will play a role when calculating the derivatives using the chain rule. The same is true of the final dimension of the first attention sub-layer's outputs $B_{t,d+1}:$ we can think of these as bottlenecks in the dependency graph for the overall transformer construction.

\begin{figure}[tbh]
    \includegraphics[width=1\columnwidth]{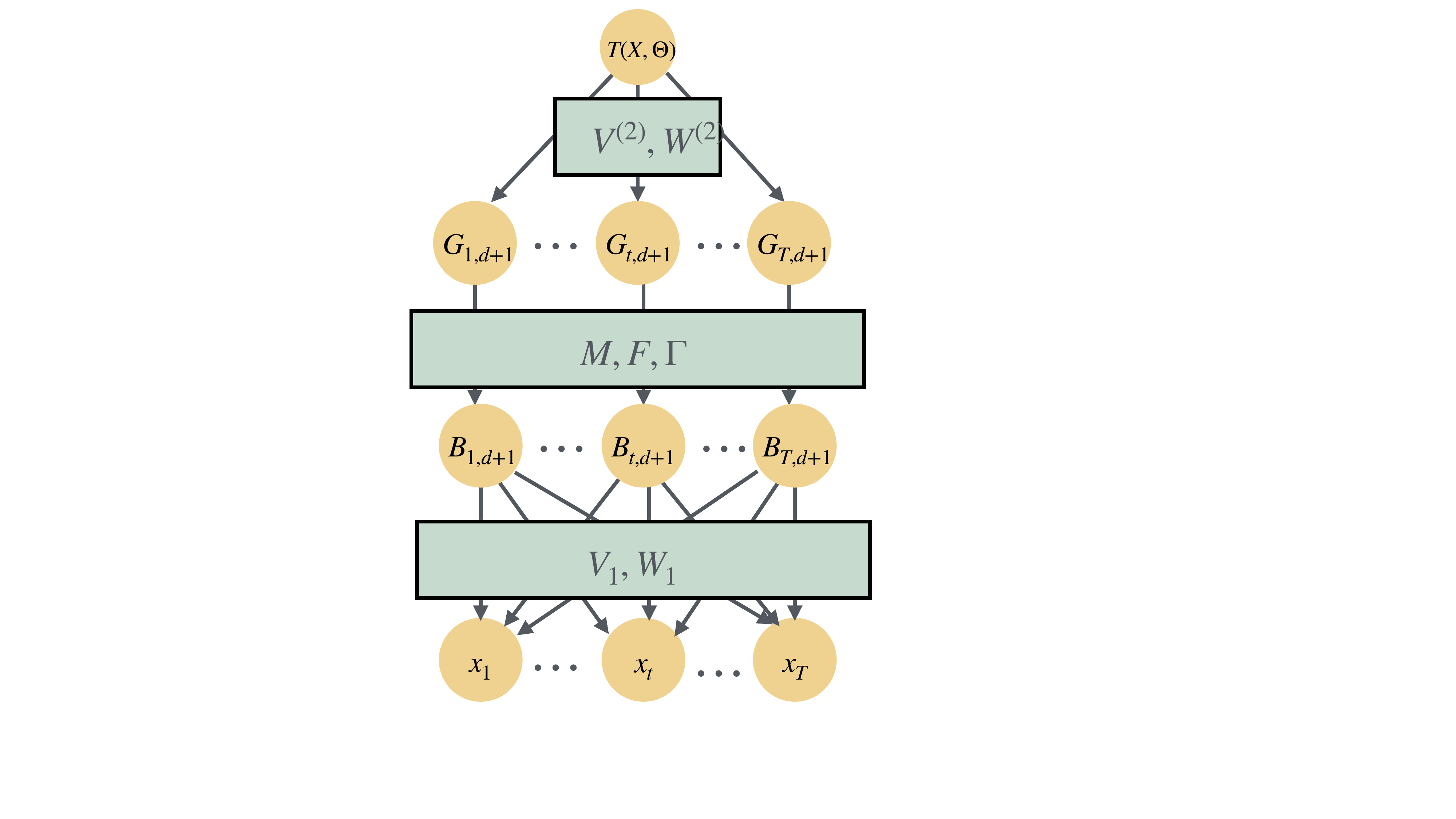}
    \caption{A diagram showing the high-level dependency graph of our 1.5-layer transformer construction. }
    \label{fig:causalgraph}
\end{figure}
    
    To give an example of how we can use these bottlenecks to create useful abstraction in the calculation of gradients, consider the gradient with respect to $W^{(1)}.$ 
    $$\nabla_{W_1} \mathcal{T}(X,\Theta) =  \sum_{t=1}^{T+1} \frac{\partial \mathcal{T}(X,\Theta)}{\partial G_{t,d+1}}\nabla_{W_1} G_{t,d+1} $$
Computing this derivative involves two tasks: first, calculating the derivative of the overall transformer output with respect to $G_{t,d+1}$ or $B_{t,d+1}$ downstream of the layer (inter-layer derivatives), and then the gradients of those key bottlenecks (scalars) with respect to the weight matrices in each layer (intra-layer gradients). To calculate the $\frac{\partial \mathcal{T}(X,\Theta)}{\partial G_{t,d+1}},$ we take the gradient with respect to the entire $G\in\mathbb{R}^{(T+1)\times (d+1)}$ matrix, and then only look at a specific row at a time, and only at the final column. This more general approach will allow us to analyze a perturbed version of our transformer, which might have have contributions from all dimensions. In these cases, we would have to consider the full Jacobian of $g_t $ with respect to $W^{(1)},$ $\mathcal{J}_{g_t}(W^{(1)}).$ 
$$\nabla_{W^{(1)}} \mathcal{T}(X,\Theta) =  \sum_{t=1}^{T+1}  \nabla^T_{g_t}\mathcal{T}(X,\Theta)\mathcal{J}_{g_t}(W^{(1)})$$
We begin with the following finite difference equation:
$$(V^{(2)})^T(G+\Delta)^T(\phi((G+\Delta)(W^{(2)})^T(G+\Delta)^Te_{T+1})$$
$$=(V^{(2)})^T\Delta^T\phi(G(W^{(2)})^Tg_{T+1})+(V^{(2)})^TG^T(\phi((G+\Delta)(W^{(2)})^T(G^T+\Delta^T )e_{T+1}) $$
First, we handle the term $$\phi\Big((G+\Delta)(W^{(2)})^T(G^T+\Delta^T )e_{T+1} \Big)=\phi\Big(G(W^{(2)})^T(G^T+\Delta^T )e_{T+1}  + \Delta (W^{(2)})^T(G^T+\Delta^T )e_{T+1}\Big)$$
$$=\phi(G(W^{(2)})^TG^Te_{T+1}+G(W^{(2)})^T\Delta^Te_{T+1} + \Delta (W^{(2)})^TG^Te_{T+1}+ \Delta (W^{(2)})^T\Delta^Te_{T+1} )$$
We can use a first-order Taylor approximation of $\phi$ to obtain:
$$\phi\Big((G+\Delta)(W^{(2)})^T(G^T+\Delta^T )e_{T+1} \Big) $$
$$= \phi(G(W^{(2)})^Tg_{T+1}) + \phi'(G(W^{(2)})^Tg_{T+1})(G(W^{(2)})^T\Delta^Te_{T+1} + \Delta (W^{(2)})^TG^Te_{T+1}) + O(\Delta^2)$$
Thus,
$$(V^{(2)})^T(G+\Delta)^T\phi\Big((G+\Delta)(W^{(2)})^T(G^T+\Delta^T )e_{T+1}\Big)- (V^{(2)})^TG^T\phi(G(W^{(2)})^TG^Te_{T+1})$$
$$=(V^{(2)})^T\Delta^T\phi(G(W^{(2)})^TG^Te_{T+1})$$
$$ + (V^{(2)})^TG^T\phi'(G(W^{(2)})^TG^Te_{T+1})(G(W^{(2)})^T\Delta^Te_{T+1} + \Delta (W^{(2)})^TG^Te_{T+1}) + O(\Delta^2)$$
We want to factor out the common factor of $\Delta^T$ (for reasons to be explained in the next paragraph), but note that there is an additive term of \[(V^{(2)})^TG^T\phi'(G(W^{(2)})^TG^Te_{T+1})\Delta (W^{(2)})^TG^Te_{T+1}\]. To bring out the $\Delta^T, $ we transpose this term (which is just a scalar), to get $e_{T+1}^TGW^{(2)}\Delta^T\phi'(G(W^{(2)})^TG^Te_{T+1})GV^{(2)}.$ Here we have used the well-known Jacobian of the softmax function, $    \phi'(x) = diag(\phi(x))-\phi(x)\phi(x)^T,$ to deduce that $\phi'(G(W^{(2)})^TG^Te_{T+1})$ is symmetric.) We can now rearrange some terms so that all of the $\Delta^T$ are on the right, and express them in terms of the trace operator:

$$=Tr\Bigg(\Big(\phi(G(W^{(2)})^TG^Te_{T+1})(V^{(2)})^T+e_{T+1}(V^{(2)})^TG^T\phi'(G(W^{(2)})^TG^Te_{T+1})G(W^{(2)})^T$$
$$+\phi'(G(W^{(2)})^TG^Te_{T+1})GV^{(2)}e_{T+1}^TGW^{(2)}\Big)\Delta^T\Bigg)$$
Now we can use the Taylor series for matrix derivatives, which says that 
$$f(A+\Delta) = f(A) +df(\Delta) + O(|\Delta|^2)=f(A) + Tr((\nabla_Af)\Delta^T)+O(|\Delta|^2). $$
We can extract the gradient from the above difference equation to conclude that 
$$\nabla_G \mathcal{T}(X,\Theta) = {\phi^{(2)}_{T+1}}(V^{(2)})^T+e_{T+1}(V^{(2)})^TG^T{\phi'^{(2)}_{T+1}}G(W^{(2)})^T+{\phi'^{(2)}_{T+1}}GV^{(2)}e_{T+1}^TGW^{(2)}$$
Where for brevity we have written ${\phi^{(2)}_{T+1}} := \phi(G(W^{(2)})^TG^Te_{T+1})$ for the post-softmax attention weights for the final position in the second attention layer. To take only the last column of the matrix $\nabla_G \mathcal{T}(X,\Theta)$, we right-multiply it by $e_{d+1}.$ To get only the $t^{th}$ row, we left-multiply by $e_t^T.$ The first term in this expression represents the change in the second attention layer's output under an infinitesimal change in the MLP output: if any of the MLP outputs change, we get a linear change in the final activation, with slope equal to the Fourier weight at that position, which are stored in the softmax scores, ${\phi^{(2)}_{T+1}}.$ The second and third terms represent changes in the final activation for the output position $T+1$ due to the shifting of attention scores. However, these terms can easily be see to be zero in the $d+1^{th}$ dimension: note that both the final column and final row of $W^{(2)}$ are both $0,$ and therefore $W^{(2)}e_{d+1}=(W^{(2)})^Te_{d+1} =  \mathbf{0}.$ Intuitively, this comes from the fact that the attention matrix in the second layer is purely position-aware. Thus, we have simply:
$$e_{t+1}^T\nabla_G \mathcal{T}(X,\Theta)e_{d+1} = c_t$$

For the gradient of the MLP output $G_{t,d+1}$ with respect to $W_1,$ we break the derivative down further, again using the chain rule:
$$\nabla_{W_1}G_{t,d+1} = \sum_{s=1}^{\omega}\frac{\partial G_{t,d+1}}{\partial B_{s,d+1}}\nabla_{W_1} B_{s,d+1}$$
However, since the MLP is applied element-wise,  the partial derivative is 0 unless $t=s, $ in which case the derivative is simply the slope of the MLP as a function of $B_{t,d+1}.$ In our construction, this is always in the range $[-4D_f,4D_f].$ More simply, we have 
$$\nabla_{W_1}G_{t,d+1} = g'_t\nabla_{W_1} B_{t,d+1},$$
Where we have used the shorthand $g'_t:=\frac{\partial G_{t,d+1}}{\partial B_{t,d+1}}.$ Note that since our MLP matrices $M$ and $F$ are only nonzero in their last row and column (respectively), only the final dimension of $B_{t,:}$ plays a role. When we consider perturbed versions of our construction, we will have to consider the full Jacobian of $g_t$ with respect to $b_t.$ Our total derivative would then be
$$\nabla_{W^{(1)}} \mathcal{T}(X,\Theta) =  \sum_{t=1}^{T+1}\sum_{i=1}^{d+1}  \nabla^T_{g_t}\mathcal{T}(X,\Theta)\Big[\mathcal{J}_{g_t}(b_t)\Big]_{:,i}\nabla_{W^{(1)}}b_{t,i}$$
For now, since we don't consider perturbations to our construction, we can write the final gradient in terms of these three basic constituents:
$$\nabla_{W_1} \mathcal{T}(X,\Theta) =  \sum_{t=1}^{\omega} \frac{\partial \mathcal{T}(X,\Theta)}{\partial G_{t,d+1}}g'_t\nabla_{W_1} B_{t,d+1} $$
As a sanity check, suppose our transformer implemented the final linear combination of Fourier components using a simple linear head, $C\in \mathbb{R}^{(T+1)\times (d+1)}$ instead. Mathematically, suppose that our transformer function were defined simply as 
$$\mathcal{T}(X,W,V,M,F,\Gamma,C)  = \bigg\langle\Big(M^Tb_t+\Gamma\Big)_+,C\bigg\rangle$$
where $C\in \mathbb{R}^{T\times (d+1)},$ with the Fourier coefficients $c_t$ in the last column, and $0s$ elsewhere. In this case, we would have$\nabla_{G}\mathcal{T}(X,\Theta) = C,$ and therefore
$$   \frac{\partial \mathcal{T}(X,\Theta)}{\partial G_{t,d+1}} = e_t^TCe_{d+1} =  c_t, \forall t\in [\omega].$$
This is the same as the gradient of our actual construction determined above, corroborating the fact that our additional attention layer is essentially performing the same task as a simple linear head (though with notably better properties when perturbations to the parameters are considered, as we demonstrate in \ref{perturbed gradient bounds appendix}.) 
To continue with our example, using our construction, we have
$$\Big[\nabla_{W^{(1)}} \mathcal{T}(X,\Theta)\Big]_{i,j} =\sum_{t=1}^T g'_t c_t \Big(e_{i}^T\nabla_{W^{(1)}} B_{t,d+1} e_{j}\Big)$$

We can already see that the different cross derivatives in the Hessian can be broken down into gradients of these major components of the derivative. The first component $g'_t $ only depends on the MLP matrices; the second term $c_t$ is a function only of the second layer attention matrices; the third term $e_{i}^T\nabla_{W^{(1)}} B_{t,d+1} e_{j}$ is the "intra-layer derivative", or the derivative of the attention output with respect to the specific matrix in question, $W^{(1)}.$

\subsection{Gradients}

\subsubsection{$\nabla_{V^{(2)}}\mathcal{T}(X,\Theta)$}
From basic rules of matrix calculus, we have
$$\nabla_{V^{(2)}}\mathcal{T}(X,\Theta)=G^T{\phi^{(2)}_{T+1}}$$
 We can use the matrix inequality $\|Av\|_2 \leq \|A\|_{1,2}\|v\|_1$ to obtain 
$$\|G^T{\phi^{(2)}_{T+1}}\| \leq \|G\|_{2,\infty}\|{\phi^{(2)}_{T+1}}\|_1 =\|G\|_{2,\infty}\leq \sqrt{2}$$

\subsubsection{$\nabla_{W^{(2)}}\mathcal{T}(X,\Theta)$}
To find the gradient $\nabla_{W^{(2)}} (V^{(2)})^TG^T\phi^{(2)}_{T+1},$ we write the following difference equations for the differential in the transformer function with respect to a small change in $W^{(2)}.$ 
$$(V^{(2)})^TG^T\phi(G(W^{(2)}+\Delta)^Tg_{T+1})$$
$$(V^{(2)})^TG^T\Big({\phi^{(2)}_{T+1}} + {\phi'^{(2)}_{T+1}}G\Delta^Tg_{T+1}\Big)+O(\Delta^2)=$$
$$(V^{(2)})^TG^T{\phi^{(2)}_{T+1}} +Tr\Big(g_{T+1}(V^{(2)})^TG^T{\phi'^{(2)}_{T+1}}G\Delta^T\Big)+O(\Delta^2)=$$
We use the Taylor series for matrix derivatives, from which we can recover the gradient from the above difference equations:
$$\nabla_{W^{(2)}}\mathcal{T}(X,\Theta)= g_{T+1}(V^{(2)})^TG^T{\phi'^{(2)}_{T+1}}G,$$
\subsubsection{$\nabla_{W^{(1)}} B_{t,i}:$}
Note that we have left the index $i$ free, despite the fact that in our exact construction, only the final dimension plays a role. This will become useful when analyzing the perturbed gradients in \ref{perturbed gradient bounds appendix}. Back to our present goal, we have already taken the gradient of our transformer function with respect to $G_{t,d+1},$ and the partial derivative of the same with respect to  $B_{t,d+1};$ what remains to calculate $\nabla_{W^{(1)}}\mathcal{T}(X,\Theta)$ is the gradient $\nabla_{W^{(1)}} B_{t,d+1}.$ To find it, we apply the method of finite differences. Using our definitions, we have
$$B_{t,i} =  e_{i}^T(V^{(1)})^TX^T\phi(X(W^{(1)}+\Delta)^T)x_t $$
$$= e_{i}^T(V^{(1)})^TX^T\big(\phi_t + \phi'_tX\Delta^Tx_t\big) + O(\Delta^2)$$
Therefore, 
$$ e_{i}^T(V^{(1)})^TX^T\phi(X(W^{(1)}+\Delta)^T)x_t -  e_{i}^T(V^{(1)})^TX^T\big(\phi(X(W^{(1)})^Tx_t) $$
$$= e_{i}^T(V^{(1)})^TX^T\phi'_tX\Delta^Tx_t=Tr\Big(x_t e_{i}^T(V^{(1)})^TX^T\phi'_tX\Delta^T\Big)$$
We can now extract the gradient to conclude that 
$$\nabla_{W_1} B_{t,i}=x_te_{i}^T(V^{(1)})^TX^T\phi'_tX$$

\subsubsection{$\nabla_{V^{(1)}} B_{t,i}:$}
The key gradient that remains to calculate $\nabla_{V^{(1)}} \mathcal{T}(X,\Theta)$ is to calculate $\nabla_{V_1} B_{t,d+1}. $ Note that 
$$B_{t,d+1} = e_{i}^T(V^{(1)})^TX^T\phi_t.$$
We now take finite differences:
$$e_{i}^T(V^{(1)}+\Delta)^TX^T\phi_t=e_{i}^TV_1^TX^T\phi_t+e_{i}^T\Delta^TX^T\phi_t$$
$$=e_{i}^T(V^{(1)})^TX^T\phi_t+Tr\Big(X^T\phi_te_{i}^T\Delta^T\Big)$$
Therefore, 
\begin{equation}\label{gradv}
\nabla_{V^{(1)}}B_{t,i}=\nabla_{V_1}(e_{i}^T(V^{(1)})^TX^T\phi_t)=X^T\phi_te_{i}^T
\end{equation}

\subsubsection{$\nabla_MG_{t,i}:$}
We now look at the gradient of $G_{t,i}$ with respect to the first matrix of the MLP, $M$. Recall that for each $t, $
$$G_{t,i} = e_{i}^TF^T(M^Tb_t + \Gamma)_+$$
Using the rules of matrix differentiation, we have
$$\nabla_MG_{t,i} =  \nabla_Me_{i}^T F^Tmax(M^Tb_t + \Gamma,0)=b_te_{i}^TF^Tdiag(I\big[M^Tb_t+\Gamma>0\big])$$

\subsubsection{$\nabla_{\Gamma}G_{t,i}$:}
The unique part of the dependency graph we need for this gradient is the gradient of $G_{t,i}$ with respect to $\Gamma$. Again, using basic matrix differentiation rules, we obtain
$$\nabla_{\Gamma} G_{t,i} = \nabla_{\Gamma} e_{i}^T F^T(M^Tb_t + \Gamma)_+=diag(I\big[M^Tb_t+\Gamma\geq0\big])Fe_{i}$$
\subsubsection{$\nabla_FG_{t,i}$: }
Finally, using the fact that the matrix derivative with respect to matrix $M$ of $x^TMy$ for vectors $x,y$ is given by $xy^T,$ we have 
$$\nabla_{F} G_{t,i} = \nabla_F e_{i}^T  F^T(M^Tb_t+\Gamma)_+ =\nabla_F (M^Tb_t+\Gamma)_+^TFe_{t,i}= \big(M^Tb_t + \Gamma\big)_+e_{i}^T$$
In summary, we have the following for the total gradient with respect to each of our parameter matrices:

$$\nabla_{V^{(2)}} \mathcal{T}(X,\Theta) = G^T{\phi^{(2)}_{T+1}} $$
$$\nabla_{W^{(2)}}\mathcal{T}(X,\Theta)=g_{T+1}(V^{(2)})^TG^T{\phi'^{(2)}_{T+1}}G$$
$$\nabla_{M} \mathcal{T}(X,\Theta) = \sum_{t=1}^{\omega} c_tb_te_{d+1}^TF^Tdiag(I\big[M^Tb_t+\Gamma>0\big])$$
$$\nabla_{F} \mathcal{T}(X,\Theta) = \sum_{t=1}^{\omega} c_t\big(M^Tb_t + \Gamma\big)_+e_{d+1}^T $$
$$\nabla_{\Gamma} \mathcal{T}(X,\Theta) = \sum_{t=1}^{\omega} c_tdiag(I\big[M^Tb_t+\Gamma\geq0\big])Fe_{d+1}$$
$$\nabla_{V^{(1)}} \mathcal{T}(X,\Theta) = \sum_{t=1}^{\omega} c_tg'_tX^T\phi_te_{d+1}^T$$
$$\nabla_{W^{(1)}} \mathcal{T}(X,\Theta) = \sum_{t=1}^{\omega} c_tg'_tx_te_{d+1}^TV_1^TX^T\phi'_tX $$
Recall that in our construction, the slope of the MLP output with respect to small changes around the points in our 1-D grid is $0$, and thus we have $g'_t = \Big|\frac{\partial G_{t,d+1}}{\partial B_{t,d+1}}\Big|=0.$ Thus we can make the simplification that 
$$\nabla_{V^{(1)}} \mathcal{T}(X,\Theta) = \mathbf{0}_{(d+1)\times(d+1)} $$
$$\nabla_{W^{(1)}} \mathcal{T}(X,\Theta) = \mathbf{0}_{(d+1)\times(d+1)} $$

We now state a useful lemma due to \citet{deorao2023ptimization}:
\begin{lemma}
\label{deoralemma}
Given an input matrix $X \in \mathbb{R}^{T\times d} $ a post-softmax Jacobian matrix $\phi'_t \in \mathbb{R}^{T\times T},$ and any vector $u_t\in \mathbb{R}^{d},$ the following bound holds:
$$\Big\| X^T \phi'_t Xu_t\Big\| =\Big\| u_t^TX^T \phi'_t X\Big\|\leq 2\|Xu_t\|_{\infty} \|X\|_{2,\infty}$$
The proof of this lemma is due to \cite{deorao2023ptimization} and we defer the reader to the proof of Eq. 42 in that manuscript for further details.
\end{lemma}

\subsection{Gradient Norms}

\subsubsection{$\|\nabla_{V^{(2)}}\mathcal{T}(X,\Theta)\|_F$}
We can use the matrix inequality $\|Av\|_2 \leq \|A\|_{1,2}\|v\|_1$ to obtain 
$$\|G^T{\phi^{(2)}_{T+1}}\|_F \leq \|G\|_{2,\infty}\|{\phi^{(2)}_{T+1}}\|_1 =\|G\|_{2,\infty}$$
Applying our bound on $\|G\|_{2,\infty} $ from \ref{parameterbounds}, we obtain:
$$\|\nabla_{V^{(2)}}\mathcal{T}(X,\Theta)\|_F\lessapprox \sqrt{2}$$
\subsubsection{$\|\nabla_{W^{(2)}}\mathcal{T}(X,\Theta)\|_F$}
We calculate the norm of  $\nabla_{W^{(2)}} \mathcal{T}(X,\Theta)$ by noting that the entire gradient term is a rank-1 outer product, and since the Frobenius norm of an outer product is the product of 2-norms, we have 
$$\|g_{T+1}(V^{(2)})^TG^T{\phi'^{(2)}_{T+1}}G\|_F = \|g_{T+1}\|\|(V^{(2)})^TG^T{\phi'^{(2)}_{T+1}}G\|$$
By \ref{deoralemma}, we can  bound the second multiplicative term as:
$$\|(V^{(2)})^TG^T\phi'_TG\|\leq 2\|(V^{(2)})^TG^T\|_{\infty} \|G\|_{2,\infty}$$
In \ref{parameterbounds}, we show that $\|GV^{(2)}\|_{\infty} \lessapprox D$ and $ \|G\|_{2,\infty}\lessapprox \sqrt{2}.$ It follows that
$$\|\nabla_{W^{(2)}}\mathcal{T}(X,\Theta)\|_F\lessapprox 2\sqrt{2}D$$

\subsubsection{$\|\nabla_{M}\mathcal{T}(X,\Theta)\|_F$}
We start with the triangle inequality:
$$\|\nabla_{M}\mathcal{T}(X,\Theta)\|_F \leq \sum_{t=1}^{\omega} \Big\|c_tb_te_{d+1}^TF^Tdiag(I\big[M^Tb_t+\Gamma>0\big])\Big\|_F$$
$$= \sum_{t=1}^{\omega} |c_t|\big\|b_t\big\|\Big\|diag(I\big[M^Tb_t+\Gamma>0\big])Fe_{d+1}\Big\| $$
In \ref{parameterbounds}, we show that $\|b_t\|\leq 1.$ Additionally we have $$|\Big\|diag(I\big[M^Tb_t+\Gamma>0\big])Fe_{d+1}\Big\| \leq \big\|Fe_{d+1}\big\|=4D_f\sqrt{4(D_f+1)}$$
We can now apply Cauchy-Schwarz to conclude:
$$\|\nabla_{M}\mathcal{T}(X,\Theta)\|_F \leq 8D_f \sqrt{\omega(D_f+1)} \approx 8D_f\sqrt{\omega D_f}$$
We will find it useful for future reference to state the norm of the intra-layer gradient term by itself. The intra-layer gradient is given by 
$$\nabla_{M}G_{t,i}=b_te_{i}^TF^Tdiag(I\big[M^Tb_t+\Gamma>0\big]).$$
Note that in the above calculation we restricted the sum to $t\in [\omega]$ because the inter-layer gradient $c_t = e_t^T\phi^{(2)}_{T+1}(V^{(2)})^T=0$ for $t>\omega,$ since the second attention layer places zero weight on all positions $t > \omega.$ In addition, the intra-layer gradient for positions $t>\omega$ is bounded by $O(D_f^{3/2})$ via the 1-Lipschitz property of the ReLU MLP, but since these terms are always multiplied by $c_t=0$ in the total gradient, they do not contribute.\footnote{The design decision to have inactive positions' contributions zeroed out by the second attention layer was made specifically to avoid a $T$ dependency in the gradient bounds.}

Thus, following the same arguments above yields
$$\Big\|\nabla_{M}G_{t,i}\Big\|_F \lessapprox8D_f^{\frac{3}{2}}I[i=d+1\wedge t\leq\omega]$$
\subsubsection{$\|\nabla_{F}\mathcal{T}(X,\Theta)\|_F$}
$$ \|\nabla_{F}\mathcal{T}(X,\Theta)\|_F \leq \sum_{t=1}^{\omega} |c_t|\Big\|\big(M^Tb_t + \Gamma\big)_+e_{d+1}^T\Big\|_F\leq \sum_{t=1}^{\omega} |c_t|\Big\|\big(M^Tb_t + \Gamma\big)_+\Big\|$$
$$\|\nabla_{F}\mathcal{T}(X,\Theta)\|_F\leq 2\sqrt{(D_f+1)\omega}\approx 2\sqrt{D_f \omega}$$
We once again provide the norm of just the intra-layer gradient for later use. This gradient is:
$$\|\nabla_{F}G_{t,i}\| \leq \|\big(M^Tb_t + \Gamma\big)_+e_{i}^T\|\lessapprox 2\sqrt{D_f}$$
Just as in the gradient with respect to $M,$ when $b_t=\mathbf{0}_{d+1},$ the above expression becomes $0.$ Thus we make this explicit:
$$\|\nabla_{F}G_{t,i}\| \lessapprox 2\sqrt{D_f}I[t\leq \omega]$$

\subsubsection{$\|\nabla_{\Gamma}\mathcal{T}(X,\Theta)\|_F$}
$$\|\nabla_{\Gamma}\mathcal{T}(X,\Theta)\|_F \leq \sum_{t=1}^{\omega} |c_t|\Big\|diag(I\big[M^Tb_t+\Gamma\geq0\big])Fe_{d+1}\Big\|$$
$$\|\nabla_{\Gamma}\mathcal{T}(X,\Theta)\|_F\leq 8D_f\sqrt{\omega(D_f+1)} \approx 8D_f\sqrt{\omega D_f} $$
Recall that $\nabla_\Gamma G_{t,i}= diag(I\big[M^Tb_t+\Gamma\geq0\big])Fe_{i}.$ Note that $Fe_i = \mathbf{0}_{4(D_f+1)}$ if $i\neq d+1.$ And just as in the gradient with respect to $M,$ when $b_t=\mathbf{0}_{d+1},$ the above expression becomes $0.$ Thus we have
$$\|\nabla_\Gamma G_{t,i}\|\lessapprox 8 D_f^{\frac{3}{2}}I[i=d+1 \wedge t\leq \omega]$$
\subsubsection{$\|\nabla_{V^{(1)}}\mathcal{T}(X,\Theta)\|_F$}
We showed that the gradient with  respect to $V^{(1)}$ was a 0-matrix due to the $\Bigg|\frac{\partial \mathcal{T}(X,\Theta)}{\partial B_{t,d+1}}\Bigg|$ component being $0$, and therefore $\|\nabla_{V^{(1)}}\mathcal{T}(X,\Theta)\|_F=0.$ Nonetheless, we will find it useful when calculating gradients for a perturbed construction to bound the norm of the intra-layer gradient. Recall that 
$$\nabla_{V^{(1)}}B_{t,i}=X^T\phi_te_{i}^T$$
The Frobenius norm of this rank-1 outer product is the product of the 2-norms:
$$\Big\|\nabla_{V^{(1)}}B_{t,i}\Big\|_F=\Big\|X^T\phi_t\Big\|\leq \|X^T\|_{1,2}\|\phi_t\|_1 = \|X\|_{2,\infty}\lessapprox \sqrt{2}$$
where we have applied our bound on $\|X\|_{2,\infty}$ in \ref{parameterbounds}.

\subsubsection{$\|\nabla_{W^{(1)}}\mathcal{T}(X,\Theta)\|_F$}
Similarly, it was already shown that the total gradient with respect to $W^{(1)}$ is the $0$-matrix, and therefore $\|\nabla_{W^{(1)}}\mathcal{T}(X,\Theta)\|_F=0,$ we still bound the norm of the intra-layer gradient for future reference. Recall that 
$$\nabla_{W_1} B_{t,i}=x_te_{i}^T(V^{(1)})^TX^T\phi'_tX$$
This is once again a rank-1 outer product, and its frobenius norm is thus bounded by 
$$\Big\|\nabla_{W_1} B_{t,i}\Big\|_F=\Big\|x_t\Big\|_2\Big\|e_{i}^T(V^{(1)})^TX^T\phi'_tX\Big\|_2\leq \Big\|X\Big\|_{2,\infty}\Big\|e_{i}^T(V^{(1)})^TX^T\phi'_tX\Big\|_2$$
Recall that the first $d$ rows and columns of $V^{(1)}$ are all $0s,$ and therefore $e_{i}^T(V^{(1)})^T = \mathbf{0} $ if $i\neq d+1.$ Thus we have 
$$\Big\|e_{i}^T(V^{(1)})^TX^T\phi'_tX\Big\|_2 = I[i=d+1]\Big\|e_{d+1}^T(V^{(1)})^TX^T\phi'_tX\Big\|_2$$
$$\leq  2I[i=d+1]\|X\|_{2,\infty}\|XV^{(1)}e_{d+1}\|_{\infty},$$
where the last inequality follows from \ref{deoralemma}. Using \ref{parameterbounds}, we conclude that 
$$\Big\|\nabla_{W_1} B_{t,i}\Big\|_F \lessapprox 2\sqrt{2}I[i=d+1].$$
In summary, we have
$$\|\nabla_{\Theta}\mathcal{T}(X,\Theta)\|^2_F \lessapprox  2 + 8D^2 + 64\omega D_f^3 + 4D_f\omega +64\omega D_f^3+2+128\omega D_f^2$$
$$=4+4\omega(2+ D_f + 32D_f^2 + 32 D_f^3)$$
$$\in O(\omega D_f^3)$$

$$\triangleq G_u(\omega, D_f)\in O(\omega D_f^3)$$
\end{proof}
\section{Chain Of Thought}

\begin{theorem}
\label{chain of thought appendix}
Let $T \in \mathbb{N}$. Let $f_1, \dots, f_T$ where $f_i(x_1\dots x_T f_1\dots f_{i-1}) = \bigoplus_{j\leq i} x_j$. Let $\widehat{\Theta_{CoT}}(\alpha,\beta)$ be the solution returned by the general learning procedure on the training set of size $m$ jointly consisting of these functions using regularization parameters $\alpha, \beta$, and let $\delta_{CoT}(\alpha,\beta)$ be the expected final error for a chain of $T$ auto-regressive steps using the transformer parametrized by $\widehat{\Theta_{CoT}}$ to solve the Parity Task of length $T.$ Then there exist regularization parameters $\alpha_{CoT},\beta_{CoT}$ such that
$$\delta_{CoT}(\alpha_{CoT},\beta_{CoT}) \leq  B_{CoT}(\sigma) ,$$ where
$$B_{CoT}(\sigma):=4Te^{-\frac{m}{8\Sigma^2}}e^{\frac{m\sigma^2}{4\Sigma^2}\Big(2G_u(1,2) + P(\sigma,1, 2,T)\Big)+\frac{ L(1, 2,T) }{2\sigma^2}}.$$

In contrast, Let $\widehat{\Theta_{OP}}(\alpha, \beta)$ be the solution returned by the general learning procedure on the training set of size $m$ consisting of only the inputs and outputs of the complete Parity task of length $T,$ i.e. the functions $f(x_1\dots x_T)=\bigoplus_{j\leq T} x_j$.  Let $\delta_{OP}(\alpha, \beta)$ be the expected error for a single pass of a transformer parameterized by $\widehat{\Theta_{OP}}(\alpha, \beta)$ on the parity task of length $T.$ Then there exist regularization parameters $(\alpha_{OP},\beta_{OP})$ such that
$$\delta_{OP} \leq  B_{OP}(\sigma),$$
where
$$B_{OP}(\sigma) := 4e^{-\frac{m}{8\Sigma^2}}e^{\frac{m\sigma^2}{4\Sigma^2}\Big(2G_u(1,T) + P(\sigma,1, T,T)\Big)+\frac{ L(1, T,T) }{2\sigma^2}}$$
For any fixed $\sigma,$ it follows that
$$B_{OP}(\sigma)\geq \frac{B_{CoT}(\sigma)^Te^{\frac{m}{8\Sigma^2}(T-1)}}{T}$$

In other words, our bound on the error for Parity increases exponentially with length when using the one-pass approach, whereas using the CoT approach, the error increases only linearly with T.
\end{theorem}
\begin{proof}
First, we show the existence of a construction which recognizes PARITY perfectly using a transformer that has learned a degree-2 boolean function, and therefore has low sharpness and parameter norms. 

Let $\mathcal{T_{CoT}}$ be a transformer with context length $2T$ matching our main construction, except that the positional encodings (and their dimension-reduced random projections) are cyclic with period $T. $ For example, positions $1$ and $T+1$ have the same positional encoding. Note that  the (pre-random-projection) input vectors $y_t\in \mathbb{R}^{T+2},$ and $\bar{W}^{(1)}\in \mathbb{R}^{(T+2)\times(T+2)}$ still, despite there being a longer context window now. Like our main construction, the matrix $\bar{W}^{(1)}$ has all zeros in the final row and column, so that attention patterns are purely position-aware. The pre-random-projection matrix $\bar{W}$ however looks like one with just a single component of degree 1, containing only the current position. 

We begin the CoT by passing the model the input of length $T$ along with a CLS token, which has the $d+1^{th}$ dimension for bit values set to $0$ as described previously in our main construction. Thus our input on the first step is $x_1\dots x_TCLS.$ Being that the $T+1^{th}$ and $1^{st}$ positions are the only ones with this positional encoding, the positional-identity map in $W^{(1)}$ creates an attention pattern for position $T+1$ which is uniform across those two positions, just as if we had a component of degree 2 across those positions when using $T+1$ absolute positional encodings (in fact, we will use this simulation argument later in the proof). from there, the transformer is exactly identical to our construction, using $D_f=2$ and $\omega=1.$ We read the final scalar output from the $T+1^{th}$ position, and if the output is closer to $0$ than it is to $1,$ output $0;$ otherwise output $1.$ We then write this output to position $T+2, $ and then repeat the process, considering the new current position to be $T+2.$ In other words, the input to step 2 of the COT is $x_1\dots x_TCLSf_1,$ where $f_1=0\bigoplus x_1=x_1.$ We then use our transformer to calculate $f_2=f_1\bigoplus x_2=x_1\bigoplus x_2, $ and so on. Applying this autoregressively, we will have 
$$f_i(x_1\dots x_T CLS f_1\dots f_{i-1}) = \bigoplus_{j\leq i} x_j$$
We have shown that our transformer correctly outputs the parity of the input string using COT. it remains to show that our bounds on the norm and sharpness for our construction for $D_f=2,\omega=1$ still hold in this slightly modified construction. Note that our matrix $\bar{W}^{(1)}$ actually resembles our construction with $D_f=1$ in the sense relevant to the parameter norm: the first  column of $\bar{W}^{(1)}$ has only a single nonzero entry. Thus the norm of the CoT construction will actually be smaller than that of our main construction using $D_f=2. $ As for the sharpness, we might expect this modification to the $W^{(1)}$ matrix to cause the perturbed gradients and Hessian to differ from the fixed-positional-encoding, degree-2 construction. However, we note that all derivatives in our construction depend on $W^{(1)} $ only indirectly, through the weights of the attention matrix, $\phi_t. $ Since we can simulate the exact attention patterns of our transformer at timestep $i$ with a transformer with $T+i+1$ fixed positional encodings and our usual degree-2 construction, these two constructions will have the exact same Hessian and gradients. 

Therefore, there are regularization parameters $\alpha_{CoT}, \beta_{CoT} \geq 0$ such that the general learning procedure returns a solution $\hat{\Theta}_{S}$ with the following properties:
$$\|\hat{\Theta}_{S}\|\leq L(1,2,T)$$
$$\mathbb{E}_{\epsilon \sim \mathcal{N}(0,\sigma^2)^n}\big[\hat{L}(f_{\hat{\Theta}_{S}+\epsilon})\big]\leq \mathbb{E}_{\epsilon \sim \mathcal{N}(0,\sigma^2)^n}\big[\hat{L}(f_{\Theta_{S}+\epsilon})\big]$$

According to our generalization bound, each time step has a generalization error bounded by
$$ \mathbb{E}_{\epsilon \sim \mathcal{N}(0,\sigma^2)}[{L}(f_{\hat{\Theta}_{S}+\epsilon})] \leq \sigma^2\Big(G_u(1,2) + \frac{1}{2}P(\sigma,1, 2,T)\Big) + 2\sqrt{ \frac{\Sigma^2}{2m}\Big(\frac{ L(1, 2,T)}{2\sigma^2}+ln\frac{1}{\delta}\Big)}\\$$
To have an incorrect step in CoT after taking the argmax, the raw error in the output must be at least $0.5,$ so the squared loss must be at least $0.25.$ By Markov's inequality, $P(\text{misclassification}) = P(\text{loss} \geq 0.25) \leq 4\mathbb{E}[\text{loss}].$ Setting $4\mathbb{E}[\text{loss}] $ equal to $\delta_S$ and solving, we have
$$\delta_S \leq 4e^{-\frac{m}{8\Sigma^2}\Big(1-\sigma^2\Big(2G_u(1,2) + P(\sigma,1, 2,T)\Big)\Big)^2+\frac{ L(1, 2,T)}{2\sigma^2}} $$
When our bound is non-vacuous, we have $\sigma^2\Big(2G_u(1,2) + P(\sigma,1, 2,T)\Big)<1, $ and we have 
$$-\frac{m}{8\Sigma^2}\Big(1-\sigma^2\Big(2G_u(1,2) + P(\sigma,1, 2,T)\Big)\Big)^2+\frac{ L(1, 2,T)}{2\sigma^2} $$
$$=-\frac{m}{8\Sigma^2}+\frac{2m\sigma^2}{8\Sigma^2}\Big(2G_u(1,2) + P(\sigma,1, 2,T)\Big) -\frac{m\sigma^4}{8\Sigma^2}\Big(2G_u(1,2) + P(\sigma,1, 2,T)\Big)^2$$
$$\leq-\frac{m}{8\Sigma^2}+\frac{2m\sigma^2}{8\Sigma^2}\Big(2G_u(1,2) + P(\sigma,1, 2,T)\Big) $$

Thus, we can use a simpler form of the upper bound which approximates the above:
$$\delta_S \leq 4e^{-\frac{m}{8\Sigma^2}}e^{\frac{m\sigma^2}{4\Sigma^2}\Big(2G_u(1,2) + P(\sigma,1, 2,T)\Big)+\frac{ L(1, 2,T)}{2\sigma^2}} $$
to bound the probability of the entire CoT having an error, we can use a union bound. Thus, the probability of the final answer being incorrect is upper bounded by 
$$\delta_{CoT} \leq  B_{CoT} := 4Te^{-\frac{m}{8\Sigma^2}}e^{\frac{m\sigma^2}{4\Sigma^2}\Big(2G_u(1,2) + P(\sigma,1, 2,T)\Big)+\frac{ L(1, 2,T) }{2\sigma^2}} $$
We compare this to our generalization bound for the solution learned by a regularized learner for calculating parity of the full string of length $T $ all at once. In that scenario, there exist (different) $\lambda_1, \lambda_2$
a generalization gap bounded by:
$$ \mathbb{E}_{\epsilon \sim \mathcal{N}(0,\sigma^2)}[{L}(f_{\Theta_{OP}+\epsilon})] \leq \sigma^2\Big(G_u(1,T) + \frac{1}{2}P(\sigma,1, T,T)\Big) + 2\sqrt{ \frac{\Sigma^2}{2m}\Big(\frac{ L(1, T,T)}{2\sigma^2}+ln\frac{1}{\delta}\Big)}\\$$
The probability of the final answer being correct in this one-pass approach is bounded by
$$\delta_{OP} \leq B_{OP} :=  4e^{-\frac{m}{8\Sigma^2}}e^{\frac{m\sigma^2}{4\Sigma^2}\Big(2G_u(1,T) + P(\sigma,1, T,T)\Big)+\frac{ L(1, T,T) }{2\sigma^2}} $$
Note that in this scenario where the output is a binary, an upper bound on the probability of an error is an upper bound on the expected error.  If we compare our bound for CoT with this bound, we can see that when using $CoT,$ the expected error only scales with $T,$ whereas the expected error in the one-pass scenario will scale exponentially with T. We know that $G_u$ is cubic in $D_f=T,$ and $P()$ carries a factor of $D_f^\frac{7}{2}$. Meanwhile, the norm term $L()$ carries a factor of $T.$ Therefore, assuming $m, \Sigma, \sigma$ are all held constant, increasing the degree by a factor of $T$ will exponentially increase the upper bound on the expected error compared to $\delta_S$, our upper bound on the error for a single step. 

Mathematically, we have that for fixed $\sigma,$ denoting the exponential cores $\bar{B}_{CoT} := B_{CoT}/(4T) = e^{-\frac{m}{8\Sigma^2}}e^{\frac{m\sigma^2}{4\Sigma^2}\Big(2G_u(1,2) + P(\sigma,1, 2,T)\Big)+\frac{ L(1, 2,T) }{2\sigma^2}}$ and $\bar{B}_{OP} := B_{OP}/4$ (stripping the constant Markov and union bound factors),
$$\bar{B}_{OP}e^{\frac{m}{8\Sigma^2}}= e^{\frac{m\sigma^2}{4\Sigma^2}\Big(2G_u(1,T) + P(\sigma,1, T,T)\Big)+\frac{ L(1, T,T) }{2\sigma^2}}$$
$$ \geq  e^{\frac{m\sigma^2T}{4\Sigma^2}\Big(2G_u(1,2) + P(\sigma,1, 2,T)\Big)+\frac{ TL(1, 2,T) }{2\sigma^2}}=\Big(\bar{B}_{CoT}e^{\frac{m}{8\Sigma^2}}\Big)^T$$
where the inequality uses $2G_u(1,T)+P(\sigma,1,T,T) \geq T\cdot\big(2G_u(1,2)+P(\sigma,1,2,T)\big)$ and $L(1,T,T)\geq T\cdot L(1,2,T),$ which follow from $G_u(\omega,D_f)\in O(\omega D_f^3)$ and the superlinear scaling of $P$ and $L$ in $D_f.$
Thus
$$\bar{B}_{OP}\geq \bar{B}_{CoT}^T e^{\frac{m}{8\Sigma^2}(T-1)}$$
Since $B_{OP} = 4\bar{B}_{OP}$ and $B_{CoT} = 4T\bar{B}_{CoT},$ the one-pass bound grows exponentially in $T$ relative to the CoT bound, with the constant factors of $4$ and $T$ absorbed into the exponential scaling.

\section{Efficient Black-Box Property Testing for $D_f, \omega$}
\label{propertytesting}
To motivate our use of property testing, we note that exactly calculating these spectral properties of either the true target function or the learned model function would traditionally involve a Fast Fourier-Walsh-Hadamard transform of said function. The computational complexity of this algorithm is $O(T 2^T),$ where again $T$ is the number of input bits. Thus, for even modest-length strings such as $T=50,$ computing these properties exactly is intractable. Therefore, upper bounding these properties approximately with a sub-exponential complexity would fill a real gap for the applicability of our generalization bound in a grey- or black-box setting, where it is known that the model has learned to perfectly fit a Boolean function over the training data, but the degree and sparsity for both the trained model and it's target function are unknown (which may not be the same, for example if the model is a low-degree interpolator as in the sense of \cite{abbe2023generalization}). For our experiments, we implemented two known property testing algorithms from the literature, with the assistance of OpenAI's $o3$ and $o4-$mini-high reasoning models to help find the simplest and most efficient known algorithms for both properties, and to ensure the algorithm was parallelized across our GPUs and adapted to the setting where the outputs are non-binary. For this experiment, we use a string length of $T=50.$ 

For the low-degree test, we used the classical test from the literature \cite{alon2003lowdegreetesting}. The basic idea of this tester is, for a given supposed degree, select $d+1$ points in the domain, and calculate the unique $d-$degree polynomial that interpolates any given subset of $d$ of these points. If the function's output on the remaining point matches the interpolated output, we accept; otherwise, we reject. We then repeat this experiment $O(\frac{1}{\epsilon}log(\frac{1}{\delta}))$ times to achieve the desired confidence level $\delta, $ with proximity parameter $\epsilon.$ This yields an overall query complexity that is in $O(\frac{d}{\epsilon}log(\frac{1}{\delta}))$ We note that while this test was originally developed for Boolean functions, the proof of the function still works with real-valued functions on a Boolean domain, using the same definition of distance. For the low-spectral-width test, we used the classical property testing algorithm of which calculates the complete Fast Fourier-Walsh-Hadamard transform on a random restriction of our function on a smaller number of coordinates \cite{gopalan2011sparsitytesting} . We define the distance from a width-$\omega$ function to be to total tail energy of the function beyond the top-$\omega$ components. To test whether the function is $\epsilon-$ far from a width-$\omega$ function, we check whether the random restriction of our function has tail energy (the sum of Fourier-Walsh coefficients beyond the top $\omega$) greater than $\frac{\epsilon}{2},$ and if so, we reject. Each trial of this test involves $O(\frac{\omega}{\epsilon})$ queries, which we repeat $O(log(\frac{1}{\delta}))$ times o achieve a confidence level of $\delta.$ This yields an overall query complexity that is in $O(\frac{\omega}{\epsilon}log(\frac{1}{\delta})).$ 

The degrees used in our property testing experiment were the same as that of our main experiment shown in e.g. \ref{fig:generalizationgap}. Since the degree was capped at $5,$ the low-degree property tester did not encounter any difficulty with computation time, and completed in a few minutes. In contrast, since the widths used in our experiments go up to $20,$ efficiency became an issue. We experimented with testers that were more easily parallelized and sharded across GPUs, such as the sketch-based testing algorithm described later in the same paper \cite{gopalan2011sparsitytesting}, or a property tester that avoids computing the entire granular Fourier Spectrum, and rather finds the energy in random hash buckets, which can be readily parallelized \citep{YZ2019_FastFourierSparsity}. However, both of these testers carry a worse dependency on the proximity parameter, $\epsilon. $ In light of the fact that we found it necessary to set $\epsilon$ as low as $10^{-4}-10^{-3}$ to obtain bounds that were faithful to the true sparsity, we opted for the simpler, less-parallelizable tester, which holds the entire sub-cube of the random restriction in contiguous memory, as it was ultimately the fastest implementation. 

In our experiments, we generated 8 random functions from each complexity class in our main experiment --that is, each of the degrees in $\{1,2,3,4,5\}$ and each of the sparsities in $\{1,7,14,20\},$ again with string length $50.$ We implemented our test by sweeping up through the degrees and widths until we hit an acceptance.  \autoref{upperboundondegree} shows the average ``first upper bound'' corresponding to the first acceptance, aggregating over all $4$ sparsities and all $8$ trials for each degree, while the right panel of \autoref{upperboundondegree} shows the average of all $5$ degrees and all 8 trials over each . Both plots show a $1-\sigma$ standard deviation bar as well. We found that our property testing-based bounds on the degree and  were both relatively tight, in the sense that they usually accepted at a level at or below the true degree. However, If the proximity parameter $\epsilon$ was not small enough, the upper bounds we found on the degree and width were artificially low. This ``wilting'' effect of our upper bounds below the true complexity parameters was more pronounced with larger proximity parameters. We found that for the low-degree test, a proximity parameter of $\epsilon= 0.001$ and $\delta=0.0001,$ the property test was able to recover upper bounds matching the true complexity parameters for most functions. For the low-spectral-width test, we found that a proximity parameter of $\epsilon=0.001$ and $\delta=.001$ were able to get close to the correct degree and width. Obtaining more faithful upper bounds on the degree and width of the target function or model function evidently requires smaller proximity and confidence parameters, and larger query complexity. 
\begin{figure}[tbh]
    \includegraphics[width=1\columnwidth]{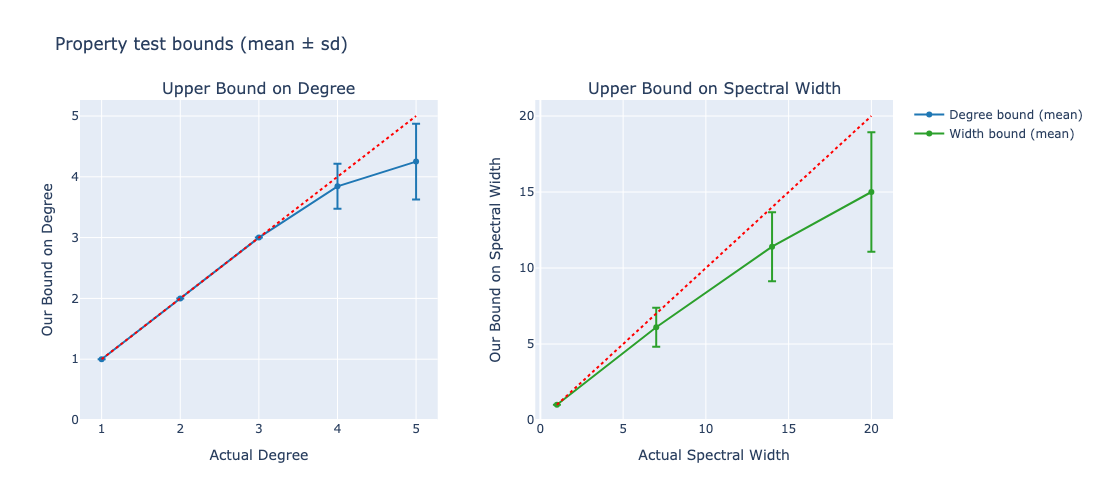}
    \caption{Left:  A plot showing the upper bound on the maximum degree of the target function obtained through property testing as a function of the actual maximum degree.  For each of the degrees used in our main experiments,  we show the mean of the bound as well as $\pm \sigma,$ aggregated over all degrees and all $10$ repetitions of the experiment.  Right: A plot showing the upper bound on the sparsity of the target function obtained through property testing. For each of the sparsity used in our main experiments,  we show the mean of the bound as well as $\pm \sigma, $ aggregated over all degrees and all $10$ repetitions of the experiment. Note that the bounds underestimate the maximum degree and sparsity slightly, due to the laxity of the test, $\epsilon$  The dotted red lines shows the line x=y. }
    \label{upperboundondegree}
\end{figure}

\end{proof}

This supplementary material is provided for the sake of rigor, in particular to include our analysis that upper bounds the term $P(\sigma;...), $ the Trace of the loss Hessian for the perturbed transformer construction. We include this bound in the supplement because, as referenced in the main paper, its derivation is quite involved, is quite loose due to its worst-case nature, and includes strong dependencies on the parameters $\omega, D_f, T, d$ that are not supported by our empirical studies. Therefore, the theoretical bound on $P(\sigma;...)$ derived herein was replaced in our main result with an empirical estimate of the worst-case perturbation to the loss Hessian. While we believe rigor demands an in-depth analysis of this perturbation term, including the analytic result in our main bound would have made it much less practical, making it non-vacuous only for very large samples, while adding little insight into the main drivers of complexity and learnability for transformer learning target functions on boolean domains. Note that our construction does not use a CLS sink; inactive positions $t \in [\omega+1, T]$ receive uniform attention weights in the first layer, but their contributions are zeroed out by the second attention layer, which places zero weight on all positions $t > \omega.$
\section{Perturbed Sharpness Bound}

\begin{theorem}
\label{full perturbed hessian bound}
    Given our perturbed transformer $\mathcal{T}(X,\Theta+\zeta), $ we can upper bound the trace of the hessian as follows:
$$Tr\Big(\nabla^2[L_f(\mathcal{T}(X,\Theta+\zeta))]\Big) \leq 2G_u(\omega, D_f) + P(\sigma,\omega,D_f,T,d) , $$
where $G_u(\omega,D_f)\in O(\omega D_f^3)$ is the squared gradient norm of the unperturbed network, and $P(\sigma,\omega,D_f,T,d)$ is the perturbation cost defined below.
\end{theorem}
\begin{proof}
We decompose the loss Hessian. Recall that for our quadratic loss,
$$\nabla^2[L_f(\mathcal{T}(X,\Theta+\zeta))]=2 \nabla \mathcal{T}(X,\Theta+\zeta)\nabla \mathcal{T}(X,\Theta+\zeta)^T+ 2(\mathcal{T}(X,\Theta+\zeta)-f(x))\nabla^2 \mathcal{T}(X,\Theta+\zeta)$$
At the exact construction $\Theta$ where $\mathcal{T}(X,\Theta)=f(x),$ the loss Hessian simplifies to $\nabla^2[L_f(\mathcal{T}(X,\Theta))]=2\nabla\mathcal{T}(X,\Theta)\nabla\mathcal{T}(X,\Theta)^T.$ Taking the trace of the perturbed loss Hessian:
$$Tr\Big(\nabla^2[L_f(\mathcal{T}(X,\Theta+\zeta))]\Big) = 2||\nabla \mathcal{T}(X,\Theta+\zeta)||^2+ 2(\mathcal{T}(X,\Theta+\zeta)-f(x))Tr\Big(\nabla^2 \mathcal{T}(X,\Theta+\zeta)\Big)$$
Applying the triangle inequality to the gradient norm:
$$||\nabla \mathcal{T}(X,\Theta+\zeta)||^2 \leq \Big(||\nabla \mathcal{T}(X,\Theta)||+||\nabla \mathcal{T}(X,\Theta+\zeta)-\nabla \mathcal{T}(X,\Theta)||\Big)^2$$
$$\leq G_u(\omega, D_f) + 2\sqrt{G_u}G_p + G_p^2$$
where in \ref{perturbed gradient bounds appendix} we demonstrate that
$$|| \nabla \mathcal{T}(X,\Theta+\zeta)-\nabla \mathcal{T}(X,\Theta)||\leq G_p(\sigma,\omega,D_f,T)$$
where $$G_p(\sigma,\omega, D_f,T) \in  O\Big(\sigma D_f^{\frac{7}{2}}\omega^{2}log^2(T)\Big)$$
and in \ref{unperturbed gradient bounds appendix} we show that $||\nabla_{\Theta}\mathcal{T}(X,\Theta)||^2_F \leq G_u(\omega, D_f)$ where $ G_u(\omega, D_f) \in O(\omega D_f^3).$

For the second term, in \ref{perturbed transformer appendix} it is shown that $|f(x)-\mathcal{T}(X,\Theta+\zeta)| \leq T_p(\sigma,\omega, D_f,T),$
where $T_p(\sigma, \omega, D_f,T) \in O(\sigma D_f^2\omega log^2(T)).$ The trace of the perturbed Hessian is bounded using the perturbed Hessian operator norm analysis in \ref{perturbedhessianbounds}:
$$\Big|Tr\Big(\nabla^2 \mathcal{T}(X,\Theta+\zeta)\Big)\Big|\leq |\Theta|\Big\|\nabla^2 \mathcal{T}(X,\Theta+\zeta)\Big\|_{op}\leq |\Theta|\Big(H_u+H_p\Big)$$
Combining these, the perturbation cost is:
$$P(\sigma,\omega,D_f,T,d) = 2\sqrt{G_u}G_p + G_p^2 + 2T_p|\Theta|(H_u+H_p)$$
\end{proof}

While the norm of the unperturbed transformer's Hessian did not play a role in the unperturbed Hessian of the loss due to a convenient simplification (the loss is exactly zero), it is a required component for the Hessian perturbation contribution, since the error is no longer exactly zero under perturbation.

\begin{theorem}
  \label{exact hessian bounds appendix}
    Suppose our transformer construction with context length $T$ is exactly representing some function $f$ with maximum degree $D_f$ and sparsity $\omega.$ Then the following bound on the operator norm of the hessian holds uniformly for all $X$.
$$||\nabla^2_{\Theta} \mathcal{T}(X,\Theta)||\lessapprox H_u(\omega,D_f,d)$$
where we have defined the function $H_u(\omega,D_f,d):=$
$$2\sqrt{2(d+1)}+ 16D_f\sqrt{2(D_f+1)\omega (d+1)}+ 16D_f\sqrt{(D_f+1)\omega(d+1)} $$
$$+ 4\sqrt{\omega (D_f+1)} + 2\sqrt{\omega}+ 8(D_f+1) \sqrt{d\omega}$$
$$\in O\Big(D_f^{\frac{3}{2}}\sqrt{\omega d}\Big) = O\Big(D_f^{\frac{3}{2}}\sqrt{\omega log(T)}\Big)$$
\end{theorem}
\begin{proof}

\section{Second Derivatives}
It is worth noting in advance of our calculation of the Hessian that, while conceivably there could be $\binom{7}{2} +7= 28$ terms or blocks in the Hessian, the above structure for the gradients implies that many of the second derivative terms will be $0.$ To understand which cross derivatives will vanish, we first recall that $\nabla_{V^{(1)}} \mathcal{T}(X,\Theta)$ and $\nabla_{W^{(1)}} \mathcal{T}(X,\Theta)$ are both $0$ matrices. This removes 13 cross derivatives. 

Clearly, $\nabla_{W^{(2)}} \mathcal{T}(X,\Theta)$ does not depend on $W^{(2)},$ so $\nabla^2_{W^{(2)}} \mathcal{T}(X,\Theta)$ is $0.$ But if we look closer at this gradient, we can see that almost all of their other cross derivatives are also $0$: recall that the first $d$ columns of $G,$ the entire vector $g_{T+1},$ and the entire second-layer attention vector ${\phi^{(2)}_{T+1}}$ for the output position $T+1$ (and therefore also ${\phi'^{(2)}_{T+1}} = diag({\phi^{(2)}_{T+1}}) - {\phi^{(2)}_{T+1}}({\phi^{(2)}_{T+1}})^T$)  depend only on positional information. Thus, the gradient is $0$ with respect to all other parameter matrices except for $V^{(2)}.$ This eliminates 4 additional cross derivatives. 

Similarly, it is clear that $\nabla_{V^{(2)}} \mathcal{T}(X,\Theta)$ does not depend on $V^{(2)}$  and therefore the second derivative with respect to $V^{(2)}$ will be $0$. The matrix $G^T {\phi^{(2)}_{T+1}}$  does indeed depend on the MLP outputs, but this eliminates one additional cross-derivative.

Regarding cross-derivatives among the MLP matrices, note that $\nabla_{F} \mathcal{T}(X,\Theta) $ does not depend on $F, $ and thus that second derivative is zero. Further, both $\nabla_{\Gamma} \mathcal{T}(X,\Theta) $ and $\nabla_{M} \mathcal{T}(X,\Theta) $ depend on the $M$ and $\Gamma$ only inside of an indicator function, which has a slope of $0$ almost everywhere. Thus, 4 more (unique) cross derivatives vanish, which means that we will need to consider no more than $28-13-4-1-4=6$ nonzero terms in the Hessian. 



1) \subsection{$\nabla^2_{W^{(2)}V^{(2)}} \mathcal{T}(X,\Theta):$}
The gradient with respect to $V^{(2)}$ is given by $\Big[\nabla_{V^{(2)}} \mathcal{T}(X,\Theta) \Big]_i =(G_{:,i})^T{\phi^{(2)}_{T+1}}.$ Recall that the output position $T+1$ places zero attention weight on itself (all weight is on positions $1,\dots,\omega$), so $[\phi^{(2)}_{T+1}]_{T+1}=0.$ Since the first $d$ columns of $G$ are independent of the parameter matrices, and $\phi^{(2)}_{T+1}$ depends only on positional information and the matrix $W^{(2)},$ all cross derivatives with  $V^{(2)}$ except this one will all be zero whenever $i\neq d+1.$ In this case however, the derivative is nonzero for $1\leq i\leq d+1.$
$$\nabla_{W^{
(2)}}\Big(\Big[\nabla_{V^{(2)}} \mathcal{T}(X,\Theta) \Big]_i\Big)  =\sum_{t=1}^{T+1} G_{t,i}\nabla_{W^{
(2)}}[\phi_{T+1}^{(2)}]_t$$
Note that 
$$[\phi_{T+1}^{(2)}]_t = e_t^T\phi(G(W^{(2)})^Tg_{T+1})$$
Taking finite differences, we have 
$$ e_t^T\phi(G(W^{(2)}+\Delta)^Tg_{T+1})=e_t^T\phi(G(W^{(2)})^Tg_{T+1})+e_t^T\phi'(G(W^{(2)})^Tg_{T+1})G\Delta^Tg_{T+1}$$
$$=e_t^T\phi(G(W^{(2)})^Tg_{T+1})+Tr\Big(g_{T+1}e_t^T\phi'^{(2)}_{T+1}G\Delta^T\Big)$$
From here we can conclude that $\nabla_{W^{
(2)}}[\phi_{T+1}^{(2)}]_t = g_{T+1}e_t^T\phi'^{(2)}_{T+1}G , $ and therefore
$$\nabla_{W^{
(2)}}\Big(\Big[\nabla_{V^{(2)}} \mathcal{T}(X,\Theta) \Big]_i \Big) =\sum_{t=1}^{T+1} G_{t,i}\nabla_{W^{
(2)}}[\phi_{T+1}^{(2)}]_t=\sum_{t=1}^{\omega} G_{t,i} g_{T+1}e_t^T\phi'^{(2)}_{T+1}G$$

Following the exact same arguments, we can immediately write the following 5 additional cross-derivatives involving $V^{(2)},$ noting that the cross-derivatives are all $0$ for $i < d+1.$

2) \subsection{$\nabla^2_{MV^{(2)}} \mathcal{T}(X,\Theta):$}
Recall that all cross derivatives with  $V^{(2)}$ will be zero whenever $i\neq d+1.$  When $i=d+1,$ we have 
$$e_{d+1}^TG^T{\phi^{(2)}_{T+1}} = \sum_{t=1}^{\omega} G_{t,d+1}c_t =\mathcal{T}(X,\Theta)$$

$$\nabla_{M } \Big(\big[\nabla_{V^{(2)}} \mathcal{T}(X,\Theta) \big]_{i}\Big)=\begin{cases}  
\mathbf{0}_{(d+1)\times 4(D_f+1)}  , &i\leq d \\ 
\\
\nabla_{M}\mathcal{T}(X,\Theta)=\sum_{t=1}^{\omega} c_tb_te_{d+1}^TF^Tdiag(I\big[M^Tb_t+\Gamma>0\big]), & i=d+1 \\ 
\end{cases}$$

3) \subsection{$\nabla^2_{FV^{(2)}} \mathcal{T}(X,\Theta):$}
$$\nabla_{F} \Big(\big[\nabla_{V^{(2)}} \mathcal{T}(X,\Theta) \big]_{i}\Big)=\begin{cases}  
\mathbf{0}_{4(D_f+1)\times (d+1) }  , &i\leq d \\ 
\\
\nabla_{F}\mathcal{T}(X,\Theta)=\sum_{t=1}^{\omega} c_t\big(M^Tb_t + \Gamma\big)_+e_{d+1}^T, & i=d+1 \\ 
\end{cases}$$
4) \subsection{$\nabla^2_{\Gamma V^{(2)}} \mathcal{T}(X,\Theta):$}
$$\nabla_{\Gamma} \Big(\big[\nabla_{V^{(2)}} \mathcal{T}(X,\Theta) \big]_{i}\Big)=\begin{cases}  
\mathbf{0}_{4(D_f+1)}  , &i\leq d \\ 
\\
\nabla_{\Gamma}\mathcal{T}(X,\Theta)=\sum_{t=1}^{\omega} c_te_{d+1}^T F^Tdiag(I\big[M^Tb_t+\Gamma\geq0\big]), & i=d+1 \\ 
\end{cases}$$

5) \subsection{$\nabla^2_{\Gamma F} \mathcal{T}(X,\Theta):$}
$$\nabla_{F} \mathcal{T}(X,\Theta) = \sum_{t=1}^{\omega} c_t\big(M^Tb_t + \Gamma\big)_+e_{d+1}^T $$
$$e_i^T \nabla_{F} \mathcal{T}(X,\Theta)e_j  =\sum_{t=1}^{\omega}  c_te_i^T\big(M^Tb_t + \Gamma\big)_+ e_{d+1}^Te_j $$
$$\nabla_{\Gamma} \Big(e_i^T \nabla_{F} \mathcal{T}(X,\Theta)e_j\Big) =I\big[j=d+1\big] \sum_{t=1}^{\omega}   c_te_iI\big[[M^Tb_t + \Gamma]_i>0\big]$$

6) \subsection{$\nabla^2_{FM} \mathcal{T}(X,\Theta):$}

From above, we have 
$$\nabla_{M} \mathcal{T}(X,\Theta) = \sum_{t=1}^{\omega} c_tb_te_{d+1}^TF^Tdiag(I\big[M^Tb_t+\Gamma>0\big])$$
$$\nabla_F\Big( \big[\nabla_{M} \mathcal{T}(X,\Theta)\big]_{i,j} \Big)= \sum_{t=1}^{\omega} c_t \nabla_F \Big(e_i^Tb_te_{d+1}^TF^Tdiag(I\big[M^Tb_t+\Gamma>0\big])e_j\Big)$$
$$\nabla_F\Big( \big[\nabla_{M} \mathcal{T}(X,\Theta)\big]_{i,j} \Big)= \sum_{t=1}^{\omega} c_t(e_i^Tb_t) \nabla_F \Bigg(e_j^Tdiag\Big(I\big[M^Tb_t+\Gamma>0\big]\Big)Fe_{d+1}\Bigg)$$
$$\nabla_F\Big( \big[\nabla_{M} \mathcal{T}(X,\Theta)\big]_{i,j} \Big)= \sum_{t=1}^{\omega} c_t(e_i^Tb_t) diag\Big(I\big[M^Tb_t+\Gamma>0\big]\Big)e_je_{d+1}^T$$
Note that we can take the scalar $e_i^Tb_t$ outside of the $\nabla_F$ operator in the third equation above because $b_t $ is upstream of $F.$  
\section{Bounding Maximum Eigenvalues}
We can express the complete hessian as a $\Big(\big(3(d+1)^2+(d+1)+(8d+6)(D_f+1)\big)\times \big(3(d+1)^2+(d+1)+(8d+6)(D_f+1)\big)\Big)$ block matrix. To bound the maximum eigenvalue of the Hessian, we want to find $\max_{||\mathbf{v}||=1}\langle \mathbf{v},\nabla^2_{\Theta}\mathcal{T}(X,\Theta)\mathbf{v}\rangle,$ where
$$\mathbf{v}=concat(\mathbf{s},\mathbf{p},\mathbf{q}, \mathbf{r},\mathbf{\eta},\mathbf{\gamma},\mathbf{\nu},),$$
$$=concat(s_1,\dots,s_{d+1},p_1,\dots,p_{d+1},q_1,\dots,q_{d+1},r_1,\dots r_{d+1},\eta_1,\dots,\eta_{d+1},\gamma_1,\nu_1,\dots,\nu_{4(D_f+1)}),$$

where $s_i\in \mathbb{R}$,$p_i,q_i,r_i,\nu_i \in \mathbb{R}^{d+1},\eta_i,\gamma_1 \in \mathbb{R}^{4(D_f+1)}.$ Written out  explicitly, we seek to upper bound 

\begin{equation}
\begin{split}
\bigl\|\nabla^2_{\Theta}\mathcal{T}(X,\Theta)\bigr\|
&= \max_{\|v\|=1}
  \begin{pmatrix}
    \mathbf{s}\\
    \mathbf{p}\\
    \mathbf{q}\\
    \mathbf{r}\\
    \boldsymbol{\eta}\\
    \boldsymbol{\gamma}\\
    \boldsymbol{\nu}
  \end{pmatrix}^T
  \begin{pmatrix}
    \nabla^2_{V^{(2)}V^{(2)}} & \nabla^2_{V^{(2)}W^{(2)}} & \cdots & \nabla^2_{V^{(2)}F} \\
    \nabla^2_{W^{(2)}V^{(2)}} & \nabla^2_{W^{(2)}W^{(2)}} & \cdots & \nabla^2_{W^{(2)}F} \\
    \vdots                    & \vdots                    & \ddots & \vdots                 \\
    \nabla^2_{F V^{(2)}}      & \nabla^2_{F W^{(2)}}      & \cdots & \nabla^2_{F F}
  \end{pmatrix}
  \begin{pmatrix}
    \mathbf{s}\\
    \mathbf{p}\\
    \mathbf{q}\\
    \mathbf{r}\\
    \boldsymbol{\eta}\\
    \boldsymbol{\gamma}\\
    \boldsymbol{\nu}
  \end{pmatrix}
\end{split}
\end{equation}

Therefore,
$$||\nabla^2_{\Theta} \mathcal{T}(X,\Theta)|| \leq \underbrace{2  \max_{||v||=1} concat(s_1,...s_{d+1})^T\Big(\nabla^2_{{W^{(2)}}V^{(2)}}\mathcal{T}(X,\Theta) \Big)concat(p_1,...,p_{d+1})}_{Term_1}$$

$$+\underbrace{2  \max_{||v||=1} concat(s_1,...,s_{d+1})^T\Big(\nabla^2_{MV^{(2)}}\mathcal{T}(X,\Theta) \Big)concat(\eta_1,...\eta_{d+1})}_{Term_2}$$
$$+\underbrace{2  \max_{||v||=1}concat(s_1,...,s_{d+1})^T \Big(\nabla^2_{\Gamma V^{(2)}}\mathcal{T}(X,\Theta)   \Big)\gamma_1}_{Term_3}$$
$$+\underbrace{2  \max_{||v||=1} concat(s_1,...,s_{d+1})^T\Big(\nabla^2_{FV^{(2)}}\mathcal{T}(X,\Theta) \Big)concat(\nu_1,...\nu_{4(D_f+1)})}_{Term_4}$$
$$+\underbrace{2 \max_{||v||=1} \gamma_1^T \Big(\nabla^2_{F \Gamma}\mathcal{T}(X,\Theta) \Big)concat(\nu_1,...\nu_{4(D_f+1)})}_{Term_{5}}$$
$$+\underbrace{2  \max_{||v||=1} concat(\eta_1,\dots,\eta_{d+1})^T\Big(\nabla^2_{F M}\mathcal{T}(X,\Theta) \Big)concat(\nu_1,...\nu_{4(D_f+1)})}_{Term_{6}}$$

\subsection{$Term_1:$}
We have $$Term_1 =2  \max_{||v||=1} concat(s_1,...s_{d+1})^T\Big(\nabla^2_{{W^{(2)}}V^{(2)}}\mathcal{T}(X,\Theta) \Big)concat(p_1,...,p_{d+1})$$
$$=2max_{||v||=1}\sum_{i=1}^{d+1} \sum_{t=1}^{\omega}  s_{i}Vec\Big(G_{t,i}g_{T+1}e_t^T{\phi'^{(2)}_{T+1}}G\Big)\mathbf{p}$$
$$=2max_{||v||=1}\sum_{i=1}^{d+1} \sum_{t=1}^{\omega}\sum_{k=1}^{d+1}  s_{i}G_{t,i}G_{T+1,k}e_t^T{\phi'^{(2)}_{T+1}}G\mathbf{p_k}$$
$$\leq 2\sum_{i=1}^{d+1} \sum_{t=1}^{\omega}\sum_{k=1}^{d+1}  |s_{i}|||G||_{1,\infty}^2\Big\|e_t^T{\phi'^{(2)}_{T+1}}G\Big\|||\mathbf{p_k}||$$
By \ref{deoralemma}, we know that $\Big\|e_t^T{\phi'^{(2)}_{T+1}}G\Big\|\leq 2||e_t||_{\infty} ||G||_{2,\infty} \leq 2||G||_{2,\infty}.$ By \ref{parameterbounds},
 we therefore have (collapsing the triple sum via Cauchy-Schwarz on the unit-vector constraint and substituting the parameter bounds from \ref{parameterbounds})
$$Term_1 \lessapprox 4\sqrt{2}d\omega$$

\subsection{$Term_2:$}
$$2  \max_{||v||=1} concat(s_1,...,s_{d+1})^T\Big(\nabla^2_{MV^{(2)}}\mathcal{T}(X,\Theta) \Big)concat(\eta_1,...\eta_{d+1})$$
(Since the output depends only on $V^{(2)}_{d+1,:},$ the Hessian block $\nabla^2_{MV^{(2)}}$ is zero for all rows $i\neq d+1,$ so only $s_{d+1}$ survives in all $V^{(2)}$ cross-terms below.)
$$=2\max_{||v||=1} s_{d+1}\sum_{t=1}^{\omega} Vec\Big(c_tb_te_{d+1}^TF^Tdiag(I\big[M^Tb_t+\Gamma>0\big])\Big)\mathbf{\eta}$$
$$=2\max_{||v||=1} s_{d+1}\sum_{t=1}^{\omega} \sum_{k=1}^{d+1} c_tb_{tk}e_{d+1}^TF^Tdiag(I\big[M^Tb_t+\Gamma>0\big])\mathbf{\eta}_k$$
$$\leq 2 |s_{d+1}|||B||_{1,\infty}\sum_{t=1}^{\omega} \sum_{k=1}^{d+1} |c_t| \Big\|e_{d+1}^TF^Tdiag(I\big[M^Tb_t+\Gamma>0\big])\Big\|||\mathbf{\eta}_k||$$
$$\leq 2 |s_{d+1}|||B||_{1,\infty}\sum_{t=1}^{\omega} \sum_{k=1}^{d+1} |c_t| ||F_{:,d+1}||||\mathbf{\eta}_k||$$
The entries of the final column of $F$ in our construction are in $[-4D_f,4D_f].$ So $||F_{:,d+1}||\leq 4D_f\sqrt{4(D_f+1)}.$ Using \ref{parameterbounds} to bound $||B||_{1,\infty},$ we get

$$Term_2 \leq 16D_f\sqrt{2(D_f+1)\omega (d+1)} $$

\subsection{$Term_3$:}
$$ Term_3 =2\max_{||v||=1} concat(s_1,...,s_{d+1})^T\Big(\nabla^2_{\Gamma V^{(2)}}\mathcal{T}(X,\Theta)\Big)\mathbf{\gamma_1}$$
$$= 2\max_{||v||=1} s_{d+1} \sum_{t=1}^{\omega}c_tVec\Big(diag(I\big[M^Tb_t+\Gamma\geq0\big])Fe_{d+1}\Big)\mathbf{\gamma_1}$$
$$=2\max_{||v||=1} s_{d+1} \sum_{t=1}^{\omega}\sum_{k=1}^{4(D_f+1)} c_te_k^Tdiag(I\big[M^Tb_t+\Gamma\geq0\big])Fe_{d+1}|\gamma_{1k}|$$
$$\leq 2|s_{d+1}| \sum_{t=1}^{\omega} |c_t| \sum_{k=1}^{4(D_f+1)} |(Fe_{d+1})_k||\gamma_{1k}| \leq 2|s_{d+1}| \sum_{t=1}^{\omega} |c_t| \|Fe_{d+1}\|\|\gamma_1\|$$
$$\leq 16D_f\omega\sqrt{D_f+1}$$
\subsection{$Term_4$:}
$$  Term_4= 2\max_{||v||=1}concat(s_1,...,s_{d+1})^T\Big(\nabla^2_{FV^{(2)}}\mathcal{T}(X,\Theta) \Big) concat(\mathbf{\nu}_1,...\nu_{4(D_f+1)})$$
$$=2\max_{||v||=1} s_{d+1}\sum_{t=1}^{\omega} c_tVec\Big(\big(M^Tb_t + \Gamma\big)_+e_{d+1}^T\Big)\mathbf{\nu}$$
$$=2\max_{||v||=1} s_{d+1}\sum_{t=1}^{\omega} \sum_{k=1}^{4(D_f+1)} c_t\Big[\big(M^Tb_t + \Gamma\big)_+\Big]_ke_{d+1}^T\mathbf{\nu_k}$$
$$\leq 2|s_{d+1}|\sum_{t=1}^{\omega} \sum_{k=1}^{4(D_f+1)} |c_t|\Bigg|\Big[\big(M^Tb_t + \Gamma\big)_+\Big]_k\Bigg|||\mathbf{\nu_k}||$$
Note that $\Big[\big(M^Tb_t + \Gamma\big)_+\Big]_k$ is upper bounded by 1, and therefore we have that 
$$Term_4\leq 4\sqrt{\omega (D_f+1)} $$

\subsection{$Term_5$:}
$$2\max_{||v||=1} \gamma_1^T\Big(\nabla^2_{F \Gamma}\mathcal{T}(X,\Theta) \Big) concat(\nu_1,...,\nu_{4(D_f+1)})$$
$$=2\max_{||v||=1} concat(\nu_1,...,\nu_{4(D_f+1)})^T\Big(\nabla^2_{ \Gamma F}\mathcal{T}(X,\Theta) \Big)\gamma_1 $$
$$=2 \max_{||v||=1} \sum_{i=1}^{d+1}\sum_{j=1}^{4(D_f+1)}\nu_{ij}I\big[i=d+1\big] \sum_{t=1}^{\omega}   c_tVec\Big(e_jI\big[[M^Tb_t + \Gamma]_j>0\big]\Big)\mathbf{\gamma_1}$$
$$=2 \max_{||v||=1} \sum_{i=1}^{d+1}\sum_{j=1}^{4(D_f+1)}\nu_{ij}I\big[i=d+1\big] \sum_{t=1}^{\omega} \sum_{k=1}^{4(D_f+1)}  I\big[k=j\big]c_tI\big[[M^Tb_t + \Gamma]_j>0\big]\gamma_{1k}$$
$$=2 \max_{||v||=1} \sum_{j=1}^{4(D_f+1)}\nu_{d+1,j} \sum_{t=1}^{\omega}  c_tI\big[[M^Tb_t + \Gamma]_j>0\big]\gamma_{1j}$$

$$\leq 2\sqrt{\omega} \max_{||v||=1} \sum_{j=1}^{4(D_f+1)}|\nu_{d+1,j} \gamma_{1j}|$$
$$\leq 4\sqrt{\omega(D_f+1)}$$
In the above chain of inequalities we used Cauchy-Schwarz multiple times, as well as the fact that the 2-norm of the Fourier coefficients is 1. 

\subsection{$Term_{6}$:}
$$2\max_{||v||=1} concat(\eta_1,\dots,\eta_{d+1})^T\Big(\nabla^2_{F M}\mathcal{T}(X,\Theta) \Big)concat(\nu_1,...\nu_{4(D_f+1)})$$
$$=2\max_{||v||=1}\sum_{j=1}^{4(D_f+1)} \sum_{i=1}^{d+1} \eta_{ij} \sum_{t=1}^{\omega} Vec\Big(  c_t(e_i^Tb_t) diag\Big(I\big[M^Tb_t+\Gamma>0\big]\Big)e_je_{d+1}^T\Big)\mathbf{\nu}$$
$$=2\max_{||v||=1}\sum_{j=1}^{4(D_f+1)} \sum_{i=1}^{d+1} \eta_{ij} \sum_{t=1}^{\omega} c_tB_{t,i}\sum_{k=1}^{4(D_f+1)}I\big[M^Tb_t+\Gamma>0\big]_j\nu_{k,d+1}$$
$$\leq 2||B||_{1,\infty}\sum_{j=1}^{4(D_f+1)} \sum_{i=1}^{d+1} |\eta_{ij}| \sum_{t=1}^{\omega} |c_t|\sum_{k=1}^{4(D_f+1)}|\nu_{k,d+1}|$$
$$\lessapprox  8(D_f+1) \sqrt{d\omega}$$

\subsection{Final Result}
Putting all 6 terms together, we have 

$$||\nabla^2_{\Theta} \mathcal{T}(X,\Theta)||\leq \underbrace{4\sqrt{2}\,d\omega}_{Term_1}+ \underbrace{16D_f\sqrt{2(D_f+1)\omega (d+1)}}_{Term_2}$$
$$ + \underbrace{16D_f\omega\sqrt{D_f+1}}_{Term_3} + \underbrace{4\sqrt{\omega (D_f+1)} }_{Term_4} + \underbrace{2\sqrt{\omega}}_{Term_5} + \underbrace{8(D_f+1) \sqrt{d\omega}}_{Term_{6}}$$
$$\approx 4\sqrt{2}\,d\omega + 16D_f\omega\sqrt{D_f+1}+32\sqrt{2}D_f^{\frac{3}{2}}\sqrt{\omega d} + 2\sqrt{\omega} + 8D_f\sqrt{\omega d}$$
$$ \approx 32\sqrt{2}D_f^{\frac{3}{2}}\sqrt{\omega d}$$
Finally, if we follow the more parameter and space efficient construction using random projections such that $d=O(log(T))$, we arrive at the final asymptotic result:

$$||\nabla^2_{\Theta} \mathcal{T}(X,\Theta)||\in O\Big(D_f^{\frac{3}{2}}\sqrt{\omega log(T)}\Big)$$
\end{proof}
\section{Perturbed Gradient Bounds}
We now consider the gradient of a slightly perturbed network, since the bound on the average direction sharpness is only exact when the trace of the loss hessian is evaluated at some neighboring point to our actual learned minimum $\Theta,$ $\Theta+\zeta.$ For convenience, denote the perturbed matrices by $\tilde{Q}=Q+\zeta_Q,\tilde{K}=K+\zeta_K,\tilde{V}=V+\zeta_V,\tilde{M}=M+\zeta_M,\tilde{F}=F+\zeta_F,\tilde{\Gamma}=\Gamma+\zeta_{\Gamma}.$ We use nested indexing for the perturbation matrices, denoting the $i,j^{th}$ element of the $K$ component of the perturbations, we write $\zeta_{K_{i,j}}.$ In this section we refer to these perturbations that make the Taylor's Theorem-based sharpness equation exact as the ``remainder perturbations''. We also reference the ``original perturbations'' or ``original Gaussian perturbations.'' We use $\epsilon$ to denote these Gaussian perturbations, and index them the same way as the remainder perturbations, i.e. the $i,j^{th}$ Gaussian perturbation in the PAC-Bayes bound is referred to as $\varepsilon_{K_{i,j}}.$ By Taylor's Theorem, we can assume that the remainder perturbations are dominated in absolute value by their corresponding Gaussian perturbations.

Note that our construction does not employ a CLS sink. Inactive positions $t\in[\omega+1,T]$ receive uniform attention weights $\frac{1}{T+1}$ in the first layer and may accumulate nonzero activations, but the second attention layer places zero weight on all positions $t>\omega,$ ensuring these have no effect on the output. For the unperturbed gradients, this means sums over positions run only to $\omega.$ For the perturbed gradients, the sums formally extend to $T+1,$ but the unperturbed intra-layer gradients $\nabla_M G_{t,i}, \nabla_F G_{t,i}, \nabla_\Gamma G_{t,i}$ remain zero for $t>\omega$ (since $b_t$ is only nonzero in the $d+1^{th}$ dimension for active positions), preventing any $T$-dependent blowup in the gradient perturbation bounds.

\subsection{Some Useful Lemmas}
Before proving our main result, we state some useful lemmas:
\begin{lemma}
    \label{chidistributedboundlemma}
    Let $\chi$  be a chi-distributed variable with standard deviation $\sigma $ and $d$ degrees of freedom. Define $$R(\delta,d):=\sqrt{1+2\sqrt{\frac{log(\frac{1}{\delta})}{d}}+\frac{2log(\frac{1}{\delta})}{d}}$$
    Then with probability at least $1-\delta$, the following bound holds:
    $$\chi \leq \sigma \sqrt{d} R(\delta,d)\approx \sigma \sqrt{d}$$

\end{lemma}
\begin{proof}
It is generally known that for a chi-squared distributed variable $Z_d$ with $d$ degrees of freedom, the following concentration bound holds: 
$$P\Big(Z_d>T+2\sqrt{dz}+2z\Big)\leq e^{-z}$$
Noting that $\sqrt{Z_d}$ is Chi-distributed, this immediately yields a bound for any chi-distributed RV $\chi_d$ with $d$ degrees of freedom. 
$$P\Big(\chi_d>\sqrt{d+2\sqrt{dz}+2z}\Big)\leq e^{-z}$$
Therefore, with probability at least $1-\delta,$ the following upper bound holds: 
$$||\epsilon_{V_{d+1,:}}||\leq \sigma \sqrt{d+2\sqrt{dlog(\frac{1}{\delta})}+2log(\frac{1}{\delta})}$$
We now introduce a new function for notational simplicity, which increases polylogarithmically with $\delta $ and decreases polynomially with $d:$
$$R(\delta,d):=\sqrt{1+2\sqrt{\frac{log(\frac{1}{\delta})}{d}}+\frac{2log(\frac{1}{\delta})}{d}}$$
Thus we can restate our bound on $||\epsilon_{V_{d+1,:}}|| $ as 
$$||\epsilon_{V_{d+1,:}}||\leq \sigma \sqrt{d}R(\delta,d) \approx \sigma \sqrt{d}$$
\end{proof}
The above term approaches $1$ as the number of degrees of freedom approaches infinity, which will allow us to simplify many of the expressions involving this function in the sequel. While this is clearly true when the scaling factor of $\delta$ from the union bound is constant, what happens if we are applying a union bound over $T $ instances of R.V.s, such that the degrees of freedom is $d=O(log(T))$? In this case we need to take care, since this polylogarithmic dependency on $T$ has created a polynomial dependency on $d,$ which may cause the terms involving the $log()$ function to not converge to 0 in the limit as $d\to \infty$. However, the following Lemma establishes that the variation above 1 in this term is controlled by our JLL approximation error, $\epsilon_p.$

\begin{lemma}
\label{chisquareunionlemma}
    Suppose $\{||\epsilon_i || \}_{i\in [T]}$ is a set of $T$ Chi-distributed variables with standard deviation $\sigma$ and $d$ degrees of freedom each, and that $d>\frac{8log(T)}{\epsilon_p^2} \iff T < e^{\frac{d \epsilon_p^2}{8}}$ By \ref{chidistributedboundlemma}  and the union bound, with probability $1-\delta, $ $||\mathbf{\epsilon}_i||\leq \sigma\sqrt{d}R(\frac{\delta}{T},d).$ Thus, $\lim_{T\to \infty} R(\frac{\delta}{T},d) =1 + O(\epsilon_p),$ and
$$\forall i, ||\mathbf{\epsilon}_i||\lessapprox \sigma \sqrt{d}$$
\end{lemma}
\begin{proof}
Note that the definition of $d$ is consistent with our dimension-reduced construction using JLL, 
$$ R(\frac{\delta}{T},d)= \sqrt{1+\frac{2}{\sqrt{d}}\sqrt{log(\frac{T}{\delta})}+\frac{2}{d}log(\frac{T}{\delta})}$$
$$ \leq \sqrt{1+2\epsilon_p \sqrt{\frac{log(\frac{T}{\delta})}{8log(T)}}+2\epsilon_p^2\frac{log(\frac{T}{\delta})}{8log(T)}}\leq \sqrt{1+\frac{\epsilon_p}{\sqrt{2}} \sqrt{1-\frac{log(\delta)}{log(T)}}+\frac{\epsilon_p^2}{4}\Big(1-\frac{log(\delta)}{log(T)}\Big)}$$
$$\lim_{T\to \infty}\sqrt{1+\frac{\epsilon_p}{\sqrt{2}} \sqrt{1-\frac{log(\delta)}{log(T)}}+\frac{\epsilon_p^2}{4}\Big(1-\frac{log(\delta)}{log(T)}\Big)}$$
$$= \sqrt{1+\frac{\epsilon_p}{\sqrt{2}} +\frac{\epsilon_p^2}{4}}=1+O(\epsilon_p)$$
We note that polynomial increases in the number of variables we apply the union bound to do not change the asymptotics, and in the limit we are sure to have $R(\frac{\delta}{poly(T)},d)\to 1 + O(\epsilon_p).$
\end{proof}

\begin{lemma}
    \label{attentionperturbationlemma}
    Let $b_t\in \mathbb{R}^{d+1}$ be the activation right after the first attention layer, at position $t.$ Then
$$||\tilde{\phi}_t-\phi_t||_1 \lessapprox 3\sigma d$$$$||\tilde{b}_t - b_t||\lessapprox 4 \sigma d$$
\end{lemma}
\begin{proof}

    $$||\tilde{b}_t - b_t|| = \Big\|(\tilde{V}^{(1)})^TX^T\phi\big(X(\tilde{W}^{(1)})^Tx_t \big)-(V^{(1)})^TX^T\phi\big(X(W^{(1)})^Tx_t \big)\Big\|$$
$$\leq \Big\|(\tilde{V}^{(1)})^TX^T\Big(\phi\big(X(\tilde{W}^{(1)})^Tx_t \big)-\phi\big(X(W^{(1)})^Tx_t \big)\Big)\Big\|+ \Big\|(\tilde{V}^{(1)}-V^{(1)})^TX^T\phi\big(X(W^{(1)})^Tx_t \big)\Big\|$$
$$\leq ||(\tilde{V}^{(1)})^TX^T||_{1,2}\Big\|\phi\big(X(\tilde{W}^{(1)})^Tx_t \big)-\phi\big(X(W^{(1)})^Tx_t \big)\Big\|_1+ \Big\|\zeta_{V^{(1)}}^TX^T\Big\|_{1,2}$$
$$\leq ||X{\tilde{V}^{(1)}}||_{2,\infty}
\Big\|\phi\big(X(\tilde{W}^{(1)})^Tx_t \big)-\phi\big(X(W^{(1)})^Tx_t \big)\Big\|_1+ \Big\|X\zeta_{V^{(1)}}\Big\|_{2,\infty}$$

Where we used the triangle inequality, and the fact that for any matrix $M$ and vector $v$, $||Mv||_2\leq ||M||_{1,2}||v||_1,$ and the fact that the 1-norm of the softmax is always exactly 1. We can follow a similar argument to Lemma A.6 in \cite{edelman2022inductive} and use the fact that Jacobian of the softmax has a $(1,1)$-norm that is bounded by 2 to conclude (continued from above):
$$\leq 2||X{\tilde{V}^{(1)}}||_{2,\infty}\Big\|X\big((\tilde{W}^{(1)})^T-(W^{(1)})^T\big)x_t \Big\|_{\infty}+ \Big\|X\zeta_V\Big\|_{2,\infty}$$
Let us take first the simpler term on the right, $ \Big\|X\zeta_{V^{(1)}} \Big\|_{2,\infty}. $ Recall that our dimension reduced-input matrix is $X=YJ,$ where J is a $(T+2 \times d+1)$ projection matrix. Then the $j^{th}$ element of the $t^{th}$ row of $X\zeta_V,$  is given by
$$\big[[X\zeta_{V^{(1)}}]_t\big]_j=\sum_{i=1}^{d}J_{t,i}\zeta_{V^{(1)}_{i,j}} + z_t\zeta_{V^{(1)}_{d+1,j}}$$
By the triangle inequality, we can thus bound the 2-norm of $[X\zeta_{V}]_t$ as 
$$||[X\zeta_{V^{(1)}}]_t|| \leq \sqrt{\sum_{j=1}^{d+1}\Big(\sum_{i=1}^d J_{t,i}\zeta_{V^{(1)}_{i,j}}\Big)^2} + \sqrt{\sum_{j=1}^{d+1}(z_t\zeta_{V^{(1)}_{d+1,j}})^2}$$
\[
\leq
\begin{Vmatrix}
       |J_{t,:}^T\zeta_{V^{(1)}_{:d,1}}| \\
      \dots \\
       |J_{t,:}^T\zeta_{V^{(1)}_{:d,d+1}}|  \\
     \end{Vmatrix} 
     +||\zeta_{V^{(1)}_{d+1,:}}||
|z_t|
\]
We can now use Cauchy-Schwarz to write 
\[
\begin{Vmatrix}
       |J_{t,:}^T\zeta_{V^{(1)}_{:d,1}}| \\
      \dots \\
       |J_{t,:}^T\zeta_{V^{(1)}_{:d,d+1}}|  \\
     \end{Vmatrix} 
     +||\zeta_{V^{(1)}_{d+1,:}}||
|z_t|
     \leq 
         ||J_{t,:}||
         \begin{Vmatrix}
           ||\zeta_{V^{(1)}_{:d,1}}|| \\
          \dots \\
           ||\zeta_{V^{(1)}_{:d,d+1}}||  \\
         \end{Vmatrix}
         +  ||\zeta_{V^{(1)}_{d+1,:}}||
    \]
  Note that   $||J_{t,:}||$ is a chi-distributed R.V. with standard deviation $\frac{1}{\sqrt{d}}$ and $d$ degrees of freedom, while $||\epsilon_{V^{(1)}_{d+1,:}}||$ is a chi-distributed R.V. with s.d. $\sigma$ and $d+1$ degrees of freedom, and $\begin{Vmatrix}
       ||\epsilon_{V^{(1)}_{:d,1}}|| \\
      \dots \\
       ||\epsilon_{V^{(1)}{:d,d+1}}||  \\
     \end{Vmatrix}$ is chi-distributed with s.d. $\sigma$ and $d(d+1)$ degrees of freedom. Using \ref{chidistributedboundlemma} and then applying a union bound over all $T+2$ random variables potentially involved in this expression over all $T$ positions, we can express our bound as
\[
||X\zeta_{V^{(1)}}||_{2,\infty} = \max_{t}||X_{t,:}\zeta_{V^{(1)}}|| \leq
   \max_t ||J_{t,:}||
     \begin{Vmatrix}
       ||\zeta_{V^{(1)}_{:d,1}}|| \\
      \dots \\
       ||\zeta_{V^{(1)}_{:d,d+1}}||  \\
     \end{Vmatrix}
     +  ||\zeta_{V^{(1)}_{d+1,:}}||
  \]
  $$ (w.p. >1-\delta, \forall t) \space \space\leq \sigma\sqrt{d(d+1)} R(\frac{\delta}{T+2},d)R(\frac{\delta}{T+2},d(d+1)) + \sigma \sqrt{d+1}R(\frac{\delta}{T+2},d+1)$$$$\approx  \sigma d $$
Where we have used \ref{chisquareunionlemma}. 
Now we consider another term in our bound for $|\tilde{b_t}-b_t|$ above, $\Big\|X\big((\tilde{W}^{(1)})^T-(W^{(1)})^T\big)x_t \Big\|_{\infty}=\Big\|X\zeta_{W^{(1)}}x_t \Big\|_{\infty}.$ To bound this term, we note that 
$$\Big\|X\zeta_{W^{(1)}}x_t \Big\|_{\infty} \leq ||X\zeta_{W^{(1)}}||_{2,\infty} ||x_t||,$$ and then bound these terms individually. Note that 
$$[x_s^T\zeta_{W^{(1)}}^T]_l=\sum_{i=1}^{d}J_{s,i}\zeta_{W^{(1)}_{i,l}} + z_s\zeta_{W^{(1)}_{d+1,l}}$$
From the triangle inequality, we can factor the above expression and replace the remainder perturbations $\zeta$ with the original gaussian perturbations $\epsilon$ that dominate them in absolute value:
 \[
 \Big\|X\big((\tilde{W}^{(1)})^T-(W^{(1)})^T\big) \Big\|_{2,\infty}
\leq
\max_s
||J_{s,:}||
\begin{Vmatrix}
       ||\epsilon_{W^{(1)}_{:d,1}}|| \\
      \dots \\
       ||\zeta_{W^{(1)}_{:d,d+1}}||  \\
     \end{Vmatrix} 
     +\Big\|\zeta_{W^{(1)}_{d+1,:}}\Big|\Big
||z_s|
\]note this is the exact same distribution of R.V.s that we saw when   bounding $||XV^{(1)}||_{2,\infty}$ above. We can again apply the union bound over all $T+2$ R.V.s in question and apply \ref{chisquareunionlemma} to conclude that
$$ ||X\zeta_{W}^T||_{2,\infty}\lessapprox \sigma d   $$
From \ref{parameterbounds} , we know that $||x_t|| \leq ||X||_{2,\infty}\lessapprox \sqrt{2},$ which holds simultaneously $\forall t.$ Therefore by applying Cauchy-Schwarz, we arrive at the following bound which holds w.p. $1-\delta:$
$$\forall  t: ||X\zeta_{W^{(1)}}^Tx_t||_{\infty}  \lessapprox \sqrt{2} \sigma d$$
It follows that $$||\tilde{\phi}_t-\phi_t||_1 \leq 2||X\zeta_{W^{(1)}}^Tx_t||_{\infty} \lessapprox 2\sqrt{2}\sigma d \leq 3 \sigma d.$$

Finally, we consider the term $||X\tilde{V}^{(1)}||_{2,\infty}.$ Recall that $x_t=J^Ty_t,$ and so the $l^{th}$ element of $x_t$ is $[x_t]_l = (J^T)_{l,t}$ for $l\in[d],$ $[x_t]_{d+1} =z_t.$  In addition, recall that the matrix $V^{(1)}=J^T\bar{V}^{(1)}\in \mathbb{R}^{(d+1)\times (T+2)}$ is a matrix of all $0$s except in the final row and column, which has a $1$. 

Thus the $(t,i)^{th}$ element of $X\tilde{V}^{(1)}$ is given by:
$$[x_t^T\tilde{V}^{(1)}]_i = z_t\tilde{V}^{(1)}_{d+1,i}+\sum_{l=1}^d J_{t,l}\tilde{V}^{(1)}_{l,i}=z_t(\zeta_{V^{(1)}_{d+1,i}} + I[i=d+1])+\sum_{l=1}^d J_{t,l}\zeta_{V^{(1)}_{l,i}}$$
By the triangle inequality, the 2-norm of $x_t^T\tilde{V}^{(1)}$ is thus bounded by
$$||x_t^T\tilde{V}^{(1)}|| \leq |z_t|(||\zeta_{\tilde{V}^{(1)}_{d+1,:d}}|| + ||e_{d+1}||)+\begin{Vmatrix}
       |J_{t,:}^T\zeta_{V^{(1)}_{:d,1}}| \\
      \dots \\
       |J_{t,:}^T\zeta_{V^{(1)}_{:d,d+1}}|  \\
     \end{Vmatrix} $$
We can once again use Cauchy-Schwarz and apply our union bound over T+2 chi-distributed variables involved in the above expression for all $t$ to conclude that 
$$||x_t^T\tilde{V}^{(1)}||\leq 1+\sigma \sqrt{d} R(\frac{\delta}{T+2},d)+\sigma \sqrt{(d(d+1)}R(\frac{\delta}{T+2},d(d+1))R(\frac{\delta}{T+2},d) \approx 1+\sigma d$$
Putting this all together, we have that with probability at least $1-3\delta,$
$$||\tilde{b}_t - b_t|| \leq 2||X{\tilde{V}^{(1)}}||_{2,\infty}\Big\|X\big((\tilde{W}^{(1)})^T-(W^{(1)})^T\big)x_t \Big\|_{\infty}+ \Big\|X\zeta_V\Big\|_{2,\infty} $$
$$\leq 2 \Big(1+\sigma d \Big)\Big(\sqrt{2} \sigma d\Big) + \sigma d \approx (2\sqrt{2}+1)\sigma d +2\sqrt{2}\sigma^2d^2$$
To simplify our analysis and the notation, we make the following simplification: we ignore the terms involving orders of $\sigma$ of 2 and above. The reason for this is that the \textit{additional} dependencies on $\omega,D_f,d$ that appear in terms involving e.g. $\sigma^2$ are no larger than the largest dependencies for terms involving $\sigma. $ Therefore, any choice of $\sigma$ that is small enough to make the first-order terms small will be small enough to make the higher order terms even smaller. While this argument is informal (we do not prove it explicitly for each and every bound), it can be easily verified, and dramatically simplifies the derivations in the sequel. One consequence of this is that, for any bounds involving terms that are the product of a perturbation term with another term that is a perturbed parameter matrix (i.e. $||\tilde{V}^TX^T||_{2}\Big\|X\big(\tilde{W}^T-(W^{(1)})^T\big)x_t \Big\|_1$ above), we can safely assume that the perturbation component will lead to an $O(\sigma^2)$ term that can be ignored.  Therefore, we could have made the approximation that 
$$||\tilde{V}^TX^T||_{2}\Big\|X\big(\tilde{W}^T-(W^{(1)})^T\big)x_t \Big\|_1 \approx ||(V^{(1)})^TX^T||_{2}\Big\|X\big(\tilde{W}^T-(W^{(1)})^T\big)x_t \Big\|_1,$$
and all of the perturbation terms that are first-order in $\sigma$ would have been the same. We will employ such approximations in the sequel. With this in mind, we use the first-order approximation, and round our constant up to the nearest whole number:
$$||\tilde{b}_t - b_t|| \leq 4\sigma d $$
\end{proof}
\begin{lemma}
    \label{mlpperturbationlemma}
        Let $g_t\in \mathbb{R}^{d+1}$ be the activation right after the MLP sublayer, including the residual connection, at position $t.$ Furthermore, assume that the standard deviation of the perturbations is such that 
    $$\sigma \leq \frac{1}{16d D_f}$$
Then $$||\tilde{g}_t - g_t||\lessapprox 32\sigma D_f^2$$
$$\in O(\sigma D_f^2)$$

\end{lemma}
\begin{proof}
    We will use the triangle inequality many times in the following arguments, starting with the observation that
$$\Big\|g\Big((V^{(1)})^TX^T\phi\big(X(W^{(1)})^Tx_t \big)\Big)-\tilde{g}\Big(\tilde{V}^TX^T\phi\big(X\tilde{W}^Tx_t \big)\Big)\Big\|  $$
$$\leq\Big\|g\Big(\tilde{V}^TX^T\phi\big(X\tilde{W}^Tx_t \big)\Big)-g\Big((V^{(1)})^TX^T\phi\big(X(W^{(1)})^Tx_t \big)\Big)\Big\|$$
$$+\Big\|\tilde{g}\Big(\tilde{V}^TX^T\phi\big(X\tilde{W}^Tx_t \big)\Big)-g\Big(\tilde{V}^TX^T\phi\big(X\tilde{W}^Tx_t \big)\Big)\Big| \Big|$$
To illustrate the two main components involved in our bound, we simplify notation and write 
$$\Big\|g(b_t)-\tilde{g}(\tilde{b}_t)\Big\| \leq \underbrace{\Big\|g(\tilde{b_t})-g(b_t)\Big\|}_{\text{Attn. Perturbation}}+ \underbrace{\Big\|\tilde{g}(\tilde{b}_t)-g(\tilde{b_t})\Big\|}_{\text{MLP   Perturbation}} $$
Note that in the attention perturbation term, both terms inside the norm involve the unperturbed MLP function $g,$ which has a slope of $0$ near $b_t.$ This term will be zero as long as none of the MLP activations changes. In reality, with some small probability, one or more of the neurons in the perturbed transformer may flip from being active to inactive. We opt to bound the probability of any such scenario occurring, effectively eschewing the need for a detailed analysis of it. Because of our choice of an MLP construction that has some width around the points in the domain for which the value is constant, this probability is low as long as $\sigma$ is small (unlike some other constructions of indicators from ReLUs using e.g. triangular waves). For any neuron to change its activation, it must be the case that $|\tilde{M}_{:,i}^T\tilde{b}_t+\tilde{\Gamma}_i-(M_{:,i}^Tb_t+\Gamma_i)| \geq  \frac{1}{4D_f}.$ In other words, for such an event to occur, the total perturbation in $M_{:,i}^Tb_t+\Gamma_i$ must be at least as large as the resolution of our MLP. 
Now, $$|\tilde{M}_{:,i}^T\tilde{b}_t+\tilde{\Gamma}_i -(M_{:,i}^Tb_t+\Gamma_i)|\lessapprox |M_{:,i}^T(\tilde{b}_t-b_t)|+|(\tilde{M}_{:,i}^T-M_{:,i}^T)b_t+(\tilde{\Gamma}_i-\Gamma_i)|$$
$$\leq ||\tilde{b}_t-b_t||+||\zeta_{M_{:,i}}||+|\zeta_{\Gamma_i}|$$
In \ref{attentionperturbationlemma} we bounded $||\tilde{b}_t-b_t||\lessapprox 4 \sigma d.$ In addition, we can bound each of the $||\zeta_{M_{:,i}}||,|\zeta_{\Gamma_i}|$  as chi-distributed variables with $d+1$ and $1$ degree of freedom, respectively. Taking a union bound over all $4(D_f+1)$ instances of these variables, we have With probability $1-\delta,$
$$\forall i \in [4(D_f+1)]: |\epsilon_{\Gamma_i}|\leq \sigma R(\frac{\delta}{8(D_f+1)},1)$$
$$\forall i \in [4(D_f+1)]: ||\zeta_{M_{:,i}}|| \leq \sigma \sqrt{d+1} R(\frac{\delta}{8(D_f+1)},d+1)\approx \sigma \sqrt{d} $$
Thus, the following bound holds for all $i\in [4(D_f+1)]:$
$$|\tilde{M}_{:,i}^T\tilde{b}_t+\tilde{\Gamma}_i -(M_{:,i}^Tb_t+\Gamma_i)|$$
$$\lessapprox  4\sigma d  +\sigma R(\frac{\delta}{8(D_f+1)},1) +\sigma \sqrt{d}$$
$$\approx 4\sigma d$$
In order for the RHS to be less than $ \frac{1}{4D_f},$ we must have 
$$\sigma \leq \frac{1}{16d D_f}$$
As long as $\sigma$ is small like this, then with high probability, none of the activations in our MLP will flip, and we conclude that
$$ \Big\|g(\tilde{b_t})-g(b_t)\Big\|=0$$
Now, the second major term we need to bound is the perturbation in the MLP: $\Big\|\tilde{g}(\tilde{b}_t)-g(\tilde{b}_t)\Big\|.$ Expanding out the MLP function in terms of the parameter matrices, and applying the triangle inequality, we obtain two main terms:
$$\Big\|\tilde{g}(\tilde{b}_t)-g(\tilde{b}_t)\Big\|=\Big\|\tilde{F}^T(\tilde{M}^T\tilde{b}_t+\tilde{\Gamma})_+ - F^T(M^T\tilde{b}_t+\Gamma)_+\Big\|$$
$$\leq \Big\|\tilde{F}^T(\tilde{M}^T\tilde{b}_t+\tilde{\Gamma})_+-\tilde{F}^T(M^T\tilde{b}_t+\Gamma)_+\Big\|+\Big\|\tilde{F}^T(M^T\tilde{b}_t+\Gamma)_+-F^T(M^T\tilde{b}_t+\Gamma)_+\Big\|$$
$$\lessapprox \underbrace{\Big\|F^T(\tilde{M}^T\tilde{b}_t+\tilde{\Gamma})_+-F^T(M^T\tilde{b}_t+\Gamma)_+\Big\|}_{Term_1}+\underbrace{\Big\|\tilde{F}^T(M^Tb_t+\Gamma)_+-F^T(M^Tb_t+\Gamma)_+\Big\|}_{Term_2}$$
Where we have replaced $\tilde{F}$ in the first term above with just $F,$ since this term is already considering perturbations to $M$ and $\Gamma,$ and therefore the perturbed versions of $F$ and $b_t$ contribute only a second order term to the overall perturbation. Recalling that the first $d$ columns of $F$ are $0,$ we then have 
$$Term_1 \leq \Big|(F_{:,d+1})^T(\tilde{M}^T\tilde{b}_t+\tilde{\Gamma})_+-(F_{:,d+1})^T(M^T\tilde{b}_t+\Gamma)_+\Big|$$
$$\leq  \Big\|F_{:,d+1}\Big\|\Big\|(\tilde{M}^T\tilde{b}_t+\tilde{\Gamma})_+-(M^T\tilde{b}_t+\Gamma)_+\Big\|$$
$$\leq  \Big\|F_{:,d+1}\Big\|\Big\|(\tilde{M}^T-M^T)\tilde{b}_t+(\tilde{\Gamma}-\Gamma)\Big\|$$
where we have used the fact that the $()_+$ function is 1-lipschitz. Recall from our definitions that $M\in \mathbb{R}^{(d+1) \times 4(D_f+1)}$ , while $ \Gamma \in \mathbb{R}^{4(D_f+1)}, $ and $F\in\mathbb{R}^{4(D_f+1)\times (d+1)}$ . Note that
$$\Big\|F_{:,d+1}\Big\| = \sqrt{\sum_{i=1}^{4(D_f+1)}(\pm 4D_f)^2}\leq 4D_f\sqrt{4(D_f+1)} \approx8D_f^{\frac{3}{2}}$$
Now, note that the $i^{th}$ element of $(\tilde{M}^T-M^T)\tilde{b}_t+(\tilde{\Gamma}-\Gamma)$ is given by $\zeta_{\Gamma_i}+ \zeta_{M_{i,:}} \tilde{b}_{t}.$ 
Therefore, by the triangle inequality
$$||(\tilde{M}^T-M^T)b_t+(\tilde{\Gamma}-\Gamma)||_2 \leq ||\zeta_{\Gamma}||+|| \zeta_{M_{:,d+1}}|| |B_{t,d+1}|\leq||\zeta_{\Gamma}||+||\zeta_{M_{:,d+1}}|| $$
Note that we have replaced $\tilde{b_t}$ with $b_t$ due to the aforementioned argument, that the difference is only in the terms involving $O(\sigma^2),$ which we assumed can be safely ignored for our purposes. We also use the fact that only the $d+1$ dimension is nonzero for $b_t.$   $||\epsilon_{\Gamma}||$ and $||\epsilon_{M_{:,d+1}}||$ are each chi-distributed R.V.s with $4(D_f+1)$  degrees of freedom. Thus, applying a union bound, the following two bounds hold simultaneously with probability at least $1-\delta:$
$$||\epsilon_{\Gamma}|| \leq \sigma  \sqrt{4(D_f+1)} R(\frac{\delta}{2},4(D_f+1))$$
$$||\epsilon_{M}|| \leq \sigma  \sqrt{4(D_f+1)} R(\frac{\delta}{2},4(D_f+1))$$
By the triangle inequality, with probability $1-\delta,$
$$||(\tilde{M}^T-M^T)b_t+(\tilde{\Gamma}-\Gamma)||_2 \leq 2\sigma\sqrt{4(D_f+1)}R(\frac{\delta}{2},4(D_f+1))\Big.$$
$$\lessapprox 4\sigma \sqrt{D_f}$$

We now consider $$Term_2= \Big\|\zeta_{F}^T(M^T\tilde{b}_t+\Gamma)_+\Big\| \leq \big\|\zeta_{F}\big\|_{2,\infty}\Big\|M^T\tilde{b}_t+\Gamma\Big\|_1$$

To bound $\big\| \zeta_F\big\|_{2,\infty}, $ we use the same arguments as before for bounding chi-distributed variables with $4(D_f+1)$ degrees of freedom, using a union bound over all $d+1$ rows of $\zeta_F$. With probability $1-\delta$:
$$||\zeta_{F}||_{2} \leq \sigma \sqrt{4(D_f+1)}R(\frac{\delta}{d+1},4(D_f+1))\approx 2\sigma \sqrt{D_f}$$
To bound $\big\|M^Tb_t+\Gamma\big\|_2,$ we have
$$\big\|M^T b_t+\Gamma\big\|_2 \approx \big\|M^Tb_t+\Gamma\big\|_2 \leq ||M|||b_t| + ||\Gamma||$$
$$\leq 2\sqrt{4(D_f+1)}\approx 4 \sqrt{D_f}.$$
We now summarize the bounds on all of the ingredients to our MLP perturbation, as well as the final bound on $\underbrace{\Big\|\tilde{g}(\tilde{b}_t)-g(\tilde{b_t})\Big\|}_{\text{MLP   Perturbation}}.$

$$\Big\|\tilde{g}(\tilde{b}_t)-g(\tilde{b_t})\Big\|\leq Term_1+Term_2 $$
$$\Big\|F_{:,d+1}\Big\|\Big\|(\tilde{M}^T-M^T)\tilde{b}_t+(\tilde{\Gamma}-\Gamma)\Big\|$$
$$+\big\|\zeta_{F}\big\|_{2,\infty}\Big\|M^T\tilde{b}_t+\Gamma\Big\|_1$$
$$+\Big(8D_f^{\frac{3}{2}}\Big)\Big(4 \sigma \sqrt{D_f}\Big)+\Big(2\sigma \sqrt{D_f}\Big)\Big(4\sqrt{D_f}\Big)$$
$$=8\sigma D_f(4D_f + 1) \approx 32\sigma D_f^2$$
Finally, we combine the two results to get the perturbation to the norm of the MLP outputs: the perturbation of the MLP applied to the perturbed attention output, and the perturbation of the unperturbed MLP applied to the perturbation in the attention output:
$$\Big\|g(b_t)-\tilde{g}(\tilde{b}_t)\Big\| \leq \underbrace{\Big\|g(\tilde{b_t})-g(b_t)\Big\|}_{\text{Attn. Perturbation}}+ \underbrace{\Big\|\tilde{g}(\tilde{b}_t)-g(\tilde{b_t})\Big\|}_{\text{MLP   Perturbation}} $$
$$\lessapprox 32\sigma D_f^2$$
$$\in O(\sigma D_f^2)$$
\end{proof}

\begin{remark}\label{inactive position remark}
The no-neuron-flip condition above (requiring a margin of $\frac{1}{4D_f}$ in the MLP pre-activations) is guaranteed for active positions $t\in [\omega]$ and $T+1,$ where $b_{t,d+1}$ lies on the grid $\{0,\frac{1}{D_f},\ldots,1\}.$ For inactive positions $t\in [\omega+1,T],$ the value $b_{t,d+1}=\frac{1}{T+1}\sum_s z_{s,d+1}$ may not lie on this grid, and the margin condition may fail. However, by the 1-Lipschitz property of ReLU, we can still bound
$$\|g(\tilde{b}_t)-g(b_t)\|\leq \|F\|_{op}\|M^T(\tilde{b}_t-b_t)\|\leq 16D_f^2|\tilde{b}_{t,d+1}-b_{t,d+1}|\lessapprox 64\sigma d D_f^2$$
Adding the MLP parameter perturbation, this gives $\|\tilde{g}_t-g_t\|\lessapprox 64\sigma d D_f^2 + 32\sigma D_f^2 \in O(\sigma d D_f^2)$ for inactive positions. Thus for the all-position bound:
$$\|\tilde{G}-G\|_{2,\infty}\lessapprox 64\sigma d D_f^2 \in O(\sigma d D_f^2)$$
We note that most downstream bounds involve $\phi_{T+1}^{(2)}$ or $\phi'^{(2)}_{T+1},$ which have support only on the active positions $[\omega],$ and thus depend only on the active-position bound $O(\sigma D_f^2).$ The one exception is the softmax perturbation (\ref{perturbedfinalactivationlemma}), where the $\infty$-norm of the logit perturbation runs over all positions, introducing an additional factor of $d=O(\log T).$
\end{remark}

\begin{theorem}
\label{perturbed transformer appendix}
Given a set of remainder perturbations $\zeta$  to the parameters $\Theta$ such that $|\zeta_i|\leq|\epsilon_i|,$ where $\epsilon \sim \mathcal{N}(0,\sigma^2)^n,$ the error of the perturbed transformer satisfies:
$$|f(x)-\mathcal{T}(X,\Theta+\zeta)| \leq T_p(\sigma,\omega,D_f,T) $$
where 
$$T_p(\sigma,\omega,D_f,T) \lessapprox  256\sigma d D_f^2 \omega log(T) \in O(\sigma D_f^2\omega log^2(T)) $$
\end{theorem}
\begin{proof}
We bound the magnitude of the perturbation of our transformer, $|\mathcal{T}(X,\Theta+\zeta))-\mathcal{T}(X,\Theta))|.$  
$$|f(x)-\mathcal{T}(X,\Theta+\zeta)|\leq |f(x)-\mathcal{T}(X,\Theta)| + |\mathcal{T}(X,\Theta)-\mathcal{T}(X,\Theta+\zeta)|= |\mathcal{T}(X,\Theta)-\mathcal{T}(X,\Theta+\zeta)|.$$
$$\Big|\mathcal{T}(X,\Theta+\zeta)-\mathcal{T}(X,\Theta)\Big|\leq\Big| (\tilde{V}^{(2)})^T\tilde{G}^T\tilde{\phi}_{T+1}^{(2)}-(V^{(2)})^TG^T\phi_{T+1}^{(2)}\Big|$$
$$\lessapprox \Big| (\tilde{V}^{(2)}-V^{(2)})^TG^T\phi_{T+1}^{(2)}\Big|+\Big|(V^{(2)})^T(\tilde{G}-G)^T\phi_{T+1}^{(2)}\Big|+\Big|(V^{(2)})^TG^T(\tilde{\phi}_{T+1}^{(2)}-\phi_{T+1}^{(2)})\Big| $$
$$\leq \Big\|G\zeta_{V^{(2)}}\Big\|_{\infty} + \Big\|(\tilde{G}-G)V^{(2)}\Big\|_{\infty}+\Big\|GV^{(2)}\Big\|_{\infty}||\tilde{\phi}_{T+1}^{(2)}-\phi_{T+1}^{(2)}||_1$$
Where we have used Holder's inequality three times, and used the fact that $||\phi_{T+1}^{(2)}||_1=1.$ Recall from \ref{parameterbounds} that $\Big\|GV^{(2)}\Big\|_{\infty}\lessapprox \sqrt{\omega}.$ We showed in \ref{perturbedfinalactivationlemma} that 
$$||\tilde{\phi}_{T+1}^{(2)}-\phi_{T+1}^{(2)}||_1\lessapprox 256\sigma d D_f^2\sqrt{\omega} log(T) $$
Note that
$$\Big\|G(\tilde{V}^{(2)}-V^{(2)})\Big\|_{\infty}\leq ||G||_{2,\infty}||\zeta_{V^{(2)}}|| \leq \sqrt{2} ||\zeta_{V^{(2)}}||$$
$ ||\zeta_{V^{(2)}}||$ is a chi-distributed variable with $d+1$ degrees of freedom, and this with probability $1-\delta,$
$$||\zeta_{V^{(2)}}||\leq \sigma \sqrt{d+1} R(\delta,d+1) \approx \sigma \sqrt{d}.$$
Therefore 
$$\Big\|G(\tilde{V}^{(2)}-V^{(2)})\Big\|_{\infty}\lessapprox \sigma \sqrt{2d}$$
Now, we have
$$\Big\|(\tilde{G}-G)V^{(2)}\Big\|_{\infty}\leq \Big\|(\tilde{G}-G)\Big\|_{2,\infty}||V^{(2)}|| \leq \sqrt{\omega}\Big\|(\tilde{G}-G)\Big\|_{2,\infty}$$
Where we have used our bound on $D=\sum_{t=1}^{\omega}$ from \ref{parameterbounds}. Note that this term appears with $\phi_{T+1}^{(2)}$ (which has support on $[\omega]$), so we may use the active-position bound $||\tilde{g}_t-g_t||\lessapprox 32\sigma D_f^2$ from \ref{mlpperturbationlemma}:
$$\Big\|(\tilde{G}-G)V^{(2)}\Big\|_{\infty}\lessapprox 32 \sqrt{\omega}\sigma D_f^2$$
Putting this all together, we have
$$\Big|\mathcal{T}(X,\Theta+\zeta)-\mathcal{T}(X,\Theta)\Big|\lessapprox  \sigma \sqrt{2d} +   32 \sqrt{\omega}\sigma D_f^2 + 256\sigma d D_f^2 \omega log(T) $$
$$\approx 256\sigma d D_f^2 \omega log(T) \in O(\sigma D_f^2\omega log^2(T))$$

\end{proof}

\begin{lemma}

    \label{mlpjacobianperturbationlemma}
    Let $g_t\in \mathbb{R}^{d+1}$ be the activation at position $t$ after the MLP, and let $b_t\in \mathbb{R}^{d+1}$ be the activation at position $t$ immediately after the attention sub-layer. Furthermore, assume that the standard deviation of the perturbations is such that 
    $$\sigma \leq \frac{1}{16d D_f}$$
Then with probability at least $1-\delta,$ the norm of the perturbation in the $i^{th}$ column of the Jacobian of $g_t$ with respect to $b_t$ is given by $$\Bigg|\Bigg|\Big[\mathcal{J}_{g_t}(b_t)\Big]_{:,i}-\Big[\mathcal{J}_{\tilde{g}_t}(\tilde{b}_t)\Big]_{:,i}\Bigg|\Bigg| \lessapprox 16\sigma D_f^2$$
\end{lemma}
\begin{proof}
Our expression for $g_t$ in terms of $b_t$ is given by 
$$g_t = F^T\Big(M^Tb_t + \tilde{\Gamma}\Big)_+ $$
Using basic rules of matrix calculus, the Jacobian is given by:

$$=F^TDiag\Big(I\Big[M^Tb_t + \Gamma>0\Big]\Big) M^T$$
Note that in our exact construction, only the last row of $F^T$ and the last column of $M^T$ are nonzero, meaning that only the bottom right entry of the Jacobian is nonzero (allowing us to make the drastic simplification in \ref{unperturbed gradient bounds appendix}). 

However, when we are dealing with perturbations, we must deal with the full Jacobian matrix. Consider the norm of the perturbation of the derivative of $g_t\in \mathbb{R}^{d+1}$ with respect to $B_{t,i}:$
$$\Bigg|\Bigg|\Big[\mathcal{J}_{g_t}(b_t)\Big]_{:,i}-\Big[\mathcal{J}_{\tilde{g}_t}(\tilde{b}_t)\Big]_{:,i}\Bigg|\Bigg|$$
$$=\Big\|\tilde{F}^TDiag\Big(I\Big[\tilde{M}^T\tilde{b_t} + \tilde{\Gamma}>0\Big]\Big) \tilde{M}^Te_{d+1} -F^TDiag\Big(I\Big[M^Tb_t + \Gamma>0\Big]\Big) M^Te_{d+1}\Big\|$$
Note that, if we can assume that none of the MLP neurons change from being positive to $0$ or vice-versa under the perturbation (an assumption formally justified already in \ref{mlpperturbationlemma}) , then we have $I\Big[\tilde{M}^T\tilde{b_t} + \tilde{\Gamma}>0\Big] = I\Big[M^T\tilde{b_t} + \Gamma>0\Big].$ From there, we can easily bound the above using the triangle inequality:
$$\lessapprox \Big\|F^TDiag\Big(I\Big[M^Tb_t + \Gamma>0\Big]\Big) (\tilde M-M)^Te_{d+1}\Big\|+\Big\|(\tilde{F}-F)^TDiag\Big(I\Big[M^Tb_t + \Gamma>0\Big]\Big)M^T e_{d+1}\Big\|$$
$$\leq \Big\|F^T\Big\|_{2,\infty} \Big\| (\tilde M-M)^Te_{d+1}\Big\|+\Big\|(\tilde{F}-F)^T\Big\|_{2,\infty} \Big\|M^T e_{d+1}\Big\|$$

The perturbation $|\tilde{g}'_t(\tilde{b_t})-g'_t(b_t)|$ does not depend on $t$, and thus we do not require a union bound to make it hold $\forall t$. Recall that $F$ is only nonzero in its final column, which has absolute values of $4D_f.$ Thus $||F||_{2,\infty} \leq 8\sqrt{D_f+1}D_f \approx 8D_f^{\frac{3}{2}}.$ Similarly, the only nonzero row of $M$ is the last row, which is all $1s,$ and therefore $||M^Te_{d+1}||\leq 2\sqrt{D_f+1} \approx 2 D_f^{\frac{1}{2}}.$ $ \Big\| (\tilde M-M)^Te_{d+1}\Big\|$ is a chi-distributed variable with $4(D_f+1)$ degrees of freedom, and therefore with probability $1-\delta,$
$$ \Big\| (\tilde{M}-M)^Te_{d+1}\Big\| \leq  \sigma \sqrt{4(D_f+1)} R(\delta,4(D_f+1))\approx 2\sigma \sqrt{D_f}$$
Finally, for $\Big\|(\tilde{F}-F)^T\Big\|_{2,\infty},$ we can take a union bound over $d+1$ chi-distributed variables with $4(D_f+1)$ degrees of freedom to obtain the following bound which holds with probability $1-\delta$
$$\Big\|(\tilde{F}-F)^T\Big\|_{2,\infty}\leq \sigma \sqrt{4(D_f+1)} R(\frac{\delta}{d+1},4(D_f+1)) \approx 2\sigma \sqrt{D_f}$$
Putting these pieces together, we have 
$$\Bigg|\Bigg|\Big[\mathcal{J}_{g_t}(b_t)\Big]_{:,i}-\Big[\mathcal{J}_{\tilde{g}_t}(\tilde{b}_t)\Big]_{:,i}\Bigg|\Bigg| \lessapprox 16 \sigma D_f^2 + 4\sigma D_f \approx 16\sigma D_f^2$$

This bound only holds as long as none of the MLP activations changes, which as we show in \ref{mlpperturbationlemma}, occurs with probability at least $1-\delta$ as long as 
$$\sigma \leq \frac{1}{16d D_f}$$
As long as $\sigma$ is small like this, then with high probability, none of the activations in our MLP will flip, and our simple bound on the norm of the perturbations of the columns of the Jacobian above holds.
\end{proof}
\begin{lemma}
    \label{unperturbedmlpgradientlemma}
    Let $g_t\in \mathbb{R}^{d+1}$ be the MLP activation at position $t.$ Then we can bound 
 $$||\nabla_{g_t}\mathcal{T}(X,\Theta)|| \leq  |c_t| + 4\sqrt{\omega}log(T)I[t=T+1]+4\omega log(T)I[t\in [\omega]]$$

\end{lemma}
\begin{proof}
        The more general formula for $\nabla_{g_t}\mathcal{T}(X,\Theta)$ which considers all of the dimensions of $g_t\in \mathbb{R}^{d+1}$ is given in \ref{unperturbed gradient bounds appendix} by the following:
$$\nabla^T_{g_t}\mathcal{T}(X,\Theta) = e_t^T{\phi^{(2)}_{T+1}}(V^{(2)})^T+e_t^Te_{T+1}(V^{(2)})^TG^T{\phi'^{(2)}_{T+1}}G(W^{(2)})^T+e_t^T{\phi'^{(2)}_{T+1}}GV^{(2)}e_{T+1}^TGW^{(2)}$$
By the triangle inequality,
$$\|\nabla^T_{g_t}\mathcal{T}(X,\Theta)\| $$
$$\leq|\frac{c_t}{D}|\Big\|V^{(2)}\Big\|  +I\big[t=T+1\big] \Big\|(V^{(2)})^TG^T{\phi'^{(2)}_{T+1}}G(W^{(2)})^T \Big\|+\Big\|e_t^T{\phi'^{(2)}_{T+1}}GV^{(2)}e_{T+1}^TGW^{(2)}\Big\|$$
Note that the only nonzero element of $V^{(2)}$ is the final dimension, which has value $D=\sum_{t=1}^{\omega} c_t \leq \sqrt{\omega}.$  Therefore $|\frac{c_t}{D}|\Big\|V^{(2)}\Big\| = |c_t|\Big\|e_{d+1}\Big\| \leq |c_t|.$ To bound the second term, we use \ref{deoralemma}, which tells us that 
$$ \Big\|(V^{(2)})^TG^T{\phi'^{(2)}_{T+1}}G(W^{(2)})^T \Big\|\leq 2\|GV^{(2)}\|_{\infty}\big\|G(W^{(2)})^T\big\|_{2,\infty}$$
We have already shown in \ref{parameterbounds} that $\|GV^{(2)}\|_{\infty} \leq \sqrt{\omega}.$ It is shown in \ref{parameterbounds} that $\|G(W^{(2)})^T\|_{2,\infty}\lessapprox 2log(T).$ Thus  $$ \Big\|(V^{(2)})^TG^T{\phi'^{(2)}_{T+1}}G(W^{(2)})^T \Big\|\leq 4\sqrt{\omega} log(T)$$
Now we bound 
$$\Big\|e_t^T{\phi'^{(2)}_{T+1}}GV^{(2)}e_{T+1}^TGW^{(2)}\Big\| \leq \Big\|e_t^T{\phi'^{(2)}_{T+1}}\Big\|\Big\|GV^{(2)}\Big\|\Big\|e_{T+1}^TGW^{(2)}\Big\|$$
We can use \ref{deoralemma} to bound  $\Big\|e_t^T{\phi'^{(2)}_{T+1}}\Big\|\leq 2I[t\in [\omega]] .$ Note that we can condition on $t\in [\omega]$ because the unperturbed softmax Jacobian matrix $\phi'^{(2)}_{T+1}$ is a (symmetric) outer-product matrix with nonzero entries only in the first $\omega$ rows and columns.  In \ref{parameterbounds} we showed that $\Big\|GV^{(2)}\Big\|\lessapprox \sqrt{\omega},$ and that $\|(W^{(2)})^Tg_{T+1}\|\lessapprox 2\sqrt{\omega} log(T).$ 
 
  Therefore 
$$\Big\|e_t^T{\phi'^{(2)}_{T+1}}GV^{(2)}e_{T+1}^TGW^{(2)}\Big\|\leq 4\omega log(T)I[t\in [\omega]]$$
In summary, we have
$$\|\nabla^T_{g_t}\mathcal{T}(X,\Theta)\| \leq |c_t| + 4\sqrt{\omega}log(T)I[t=T+1]+4\omega log(T)I[t\in [\omega]] $$
\end{proof}
\begin{lemma}
\label{perturbedfinalactivationlemma}
Let $\phi^{(2)}_{T+1}\in\mathbb{R}^{T+1}$ be the softmax score for position $T+1$ in the second attention layer of our transformer. Then
$$\|\tilde{\phi}^{(2)}_{T+1}-\phi^{(2)}_{T+1}\|_{\infty} \leq \|\tilde{\phi}_{T+1}^{(2)}-\phi_{T+1}^{(2)}\|_1\lessapprox  256\sigma d D_f^2\sqrt{\omega} log(T) \in O(\sigma D_f^2\sqrt{\omega} log^2(T))$$
\end{lemma}
\begin{proof}
By construction, $$\phi_{T+1}^{(2)}=\phi(G(W^{(2)})^Tg_{T+1}).$$
We seek an upper bound on the 1-norm of the perturbed final activation, which also serves as an upper bound on the $\infty-$norm:
$$\|\tilde{\phi}_{T+1}^{(2)}-\phi_{T+1}^{(2)}\|_{\infty}\leq\|\tilde{\phi}_{T+1}^{(2)}-\phi_{T+1}^{(2)}\|_1$$
By Lemma A.6 in \cite{edelman2022inductive}, the 2-boundedness of the (1,1)-norm of the softmax's Jacobian matrix implies that
$$\|\phi\big(\tilde{G}(\tilde{W}^{(2)})^T\tilde{g}_{T+1}\big)-\phi\big(G(W^{(2)})^Tg_{T+1}\big)\|_1\leq 2 \|\tilde{G}(\tilde{W}^{(2)})^T\tilde{g}_{T+1}-G(W^{(2)})^Tg_{T+1}\|_{\infty}$$
$$\lessapprox  2\|(\tilde{G}-G)(W^{(2)})^Tg_{T+1}\|_{\infty}+2\|G(\tilde{W}^{(2)}-W^{(2)})^Tg_{T+1}\|_{\infty}+2||G(W^{(2)})^T(\tilde{g}_{T+1}-g_{T+1})||_{\infty}$$
$$\leq  2||\tilde{G}-G||_{2,\infty}||(W^{(2)})^Tg_{T+1}||+2||G(\tilde{W}^{(2)}-W^{(2)})^T||_{2,\infty}||g_{T+1}||+2||G(W^{(2)})^T||_{2,\infty}||(\tilde{g}_{T+1}-g_{T+1})||$$
Note that $||\tilde{G}-G||_{2,\infty}$ was already bounded in \ref{mlpperturbationlemma}; as discussed in Remark \ref{inactive position remark}, the all-position bound (including inactive positions where the no-neuron-flip condition may not hold) is $||\tilde{G}-G||_{2,\infty}\lessapprox 64\sigma d D_f^2.$ In \ref{parameterbounds}, we show that $||(W^{(2)})^Tg_{T+1}||\leq 2\sqrt{\omega} log(T).$ To bound the term $||G(\tilde{W}^{(2)}-W^{(2)})^T||_{2,\infty},$ we follow  similar arguments to the proof of \ref{attentionperturbationlemma} to arrive at the bound
\[
 \Big\|G\big((\tilde{W}^{(2)})^T-(W^{(2)})^T\big) \Big\|_{2,\infty}
\leq
\max_s
||J_{s,:}||
\begin{Vmatrix}
       ||\epsilon_{W^{(2)}_{:d,1}}|| \\
      \dots \\
       ||\zeta_{W^{(2)}_{:d,d+1}}||  \\
     \end{Vmatrix} 
     +\Big\|\zeta_{W^{(2)}_{d+1,:}}\Big|\Big
||g_s|
\]
Continuing to follow that proof, we can again apply the union bound over all $T+2$ R.V.s in question and apply \ref{chisquareunionlemma} to conclude that 
$$ ||G\zeta_{W}^T||_{2,\infty}\lessapprox \sigma d   $$
We have already bounded $||g_{T+1}|| $ in \ref{parameterbounds} as $\lessapprox 1.$ Clearly, $||(\tilde{g}_{T+1}-g_{T+1})||\leq ||\tilde{G}-G||_{2,\infty}.$ Finally, we show in \ref{parameterbounds} that  $||G(W^{(2)})^T||_{2,\infty}\lessapprox 2log(T).$ 

Putting all of these bounds together, we have
$$||\tilde{\phi}_{T+1}^{(2)}-\phi_{T+1}^{(2)}||_{\infty} \leq ||\tilde{\phi}_{T+1}^{(2)}-\phi_{T+1}^{(2)}||_1 $$
$$\lessapprox 256\sigma d D_f^2\sqrt{\omega} log(T) + 2\sigma d +64\sigma D_f^2 log(T)$$
$$\approx 256\sigma d D_f^2\sqrt{\omega} log(T) \in O(\sigma D_f^2\sqrt{\omega} log^2(T))$$
\end{proof}

\begin{lemma}
    \label{perturbedmlpgradientlemma}
        Let $g_t,\tilde{g_t}\in\mathbb{R}^{d+1}$ be the unperturbed and perturbed (respectively) post-mlp activations of our transformer. Then the 2-norm of the difference of the gradients is given by 
        $$|\langle\nabla_{\tilde{g}_t} \mathcal{T}(X,\tilde{\Theta})-\nabla_{g_t} \mathcal{T}(X,\Theta),e_{i}\rangle| $$
$$\lessapprox 256\sigma d D_f^2\omega log(T)I[i=d+1] +1536\sigma d D_f^2 \omega^{\frac{3}{2}} log^2(T)I[i<d+1]$$

\end{lemma}
\begin{proof}
    Recall that the more general formula for $\nabla_{\tilde{g}_t}\mathcal{T}(X,\tilde{\Theta})$ which considers all of the coordinates of $\tilde{g}_t$ is given in \ref{unperturbed gradient bounds appendix} by the following:
$$\nabla^T_{\tilde{g}_t} \mathcal{T}(X,\tilde{\Theta}) = \underbrace{e_t^T {\tilde{\phi}^{(2)}_{T+1}}(\tilde{V}^{(2)})^T}_{Term_I}+\underbrace{e_t^T {\tilde{\phi}'^{(2)}_{T+1}}\tilde{G}\tilde{V}^{(2)}e_{T+1}^T\tilde{G}\tilde{W}^{(2)}}_{Term_{II}}+\underbrace{e_t^T e_{T+1}(\tilde{V}^{(2)})^T\tilde{G}^T{\tilde{\phi}'^{(2)}_{T+1}}\tilde{G}(\tilde{W}^{(2)})^T}_{Term_{III}}$$
We can decompose the perturbation to this gradient in terms of the perturbations to these three terms :
$$|\langle\nabla_{\tilde{g}_t} \mathcal{T}(X,\tilde{\Theta})-\nabla_{g_t} \mathcal{T}(X,\Theta),e_{i}\rangle| $$
$$= |(\tilde{Term_I} - Term_I)e_{i}| +|(\tilde{Term}_{II}-Term_{II})e_{i}|+|(\tilde{Term_{III}} - Term_{III})e_{i}|$$
We begin by decomposing the perturbation in $Term_I$ into two components, using the triangle inequality.
$$|(\tilde{Term_I} -Term_I)e_{i}| \leq \Big|e_t^T (\tilde{\phi}^{(2)}_{T+1}-\phi^{(2)}_{T+1})(V^{(2)})^Te_{i}\Big|+\Big|e_t^T \phi^{(2)}_{T+1}(\tilde{V}^{(2)}-V^{(2)})^Te_{i}\Big|$$
$$ \leq \sqrt{\omega}\Big\|\tilde{\phi}^{(2)}_{T+1}-\phi^{(2)}_{T+1}\Big\|_{\infty}I[i=d+1] +|c_t|\Big|(\tilde{V}^{(2)}-V^{(2)})^Te_{i}\Big|$$
We show in \ref{perturbedfinalactivationlemma} that 
$$\Big\|\tilde{\phi}^{(2)}_{T+1}-\phi^{(2)}_{T+1}\Big\|_{\infty}\lessapprox 256\sigma d D_f^2\sqrt{\omega} log(T)$$
In addition, we note that $\Big|(\tilde{V}^{(2)}-V^{(2)})^Te_{i}\Big|\leq \sigma R(\delta,1) . $ Thus we have
$$|(\tilde{Term_I} -Term_I)e_{i}| \lessapprox 256\sigma d D_f^2\omega log(T)I[i=d+1]+\sigma R(\delta,1) $$
Moving on to the perturbation of $Term_{II},$ we once again apply the triangle inequality: $$|(\tilde{Term_{II}}-Term_{II})e_{i}|\leq \Big|e_t^T {\tilde{\phi}'^{(2)}_{T+1}}\tilde{G}\tilde{V}^{(2)}e_{T+1}^T\tilde{G}\tilde{W}^{(2)}e_{i}-e_t^T {\phi'^{(2)}_{T+1}}GV^{(2)}e_{T+1}^TGW^{(2)}e_{i} \Big|$$
By the triangle inequality, this is bounded by 
$$\Big|e_t^T \Big({\tilde{\phi}'^{(2)}_{T+1}}\tilde{G}\tilde{V}^{(2)}e_{T+1}^T\tilde{G}-{\phi'^{(2)}_{T+1}}GV^{(2)}e_{T+1}^TG\Big)W^{(2)}e_{i}\Big|+\Big|e_t^T {\phi'^{(2)}_{T+1}}GV^{(2)}e_{T+1}^TG(\tilde{W}^{(2)}-W^{(2)})e_{i} \Big|$$
Note that the first term above terms vanishes when $i=d+1$: the term with only the unperturbed $W^{(2)}$ matrix vanishes, because the final, $(d+1)^{th}$ row and column are all all $0s.$ In the general case where $i<d+1,$ 
$$|(\tilde{Term_{II}}-Term_{II})e_i|\leq \Big|e_t^T {\tilde{\phi}'^{(2)}_{T+1}}\tilde{G}\tilde{V}^{(2)}e_{T+1}^T\tilde{G}\tilde{W}^{(2)}e_i-e_t^T {\phi'^{(2)}_{T+1}}GV^{(2)}e_{T+1}^TGW^{(2)}e_i \Big|$$
$$\lessapprox \underbrace{\Big\|e_t^T (\tilde{\phi}'^{(2)}_{T+1}-{\phi'^{(2)}_{T+1}})GV^{(2)}e_{T+1}^TGW^{(2)} \Big\|}_{Term_A}I[i<d+1]$$
$$+\underbrace{\Big\|e_t^T \phi'^{(2)}_{T+1}(\tilde{G}\tilde{V}^{(2)}-GV^{(2)})e_{T+1}^TGW^{(2)}\Big\|}_{Term_B}I[i<d+1]$$
$$+\underbrace{\Big\|e_t^T \phi'^{(2)}_{T+1}GV^{(2)}e_{T+1}^TG(\tilde{W}^{(2)}-W^{(2)})\Big\|}_{Term_C}$$
$$+\underbrace{\Big\|e_t^T \phi'^{(2)}_{T+1}GV^{(2)}e_{T+1}^T(\tilde{G}-G)W^{(2)}\Big\|}_{Term_D}I[i<d+1]$$
$$Term_A = \Big\|e_t^T (\tilde{\phi}'^{(2)}_{T+1}-{\phi'^{(2)}_{T+1}})GV^{(2)}e_{T+1}^TGW^{(2)} \Big\|\leq \Big\| (\tilde{\phi}'^{(2)}_{T+1}-{\phi'^{(2)}_{T+1}})e_t\Big\|_1\Big\|GV^{(2)}\Big\|_{\infty}\Big\|e_{T+1}^TGW^{(2)} \Big\|,$$
where we have used the symmetry of $\tilde{\phi}'_t-\phi'_t$ and Holder's inequality. Note that by \ref{parameterbounds}
, $\Big\|GV^{(2)}\Big\|_{\infty} \leq \sqrt{\omega}, $ and  $\Big\|e_{T+1}^TGW^{(2)} \Big\|=||(W^{(2)})^Tg_{T+1}||\lessapprox 2\sqrt{\omega}log(T).$ We showed in \ref{perturbedfinalsoftmaxjacobianlemma} that $\Big\|(\tilde{\phi}'^{(2)}_{T+1}-{\phi'^{(2)}_{T+1}})e_t\Big\|_1\leq  768\sigma d D_f^2\sqrt{\omega} log(T).$
Plugging this bound in, we have that
$$Term_A = \Big\|e_t^T (\tilde{\phi}'^{(2)}_{T+1}-{\phi'^{(2)}_{T+1}})GV^{(2)}e_{T+1}^TGW^{(2)} \Big\|\lessapprox  1536 \sigma d D_f^2 \omega^{\frac{3}{2}} log^2(T)$$
For $Term_B, $ we can use Holder's inequality:
$$Term_B=\Big\|e_t^T \phi'^{(2)}_{T+1}(\tilde{G}\tilde{V}^{(2)}-GV^{(2)})e_{T+1}^TGW^{(2)} \Big\|\leq \Big\|e_t^T \phi'^{(2)}_{T+1}\Big\|_1 \Big\|(\tilde{G}\tilde{V}^{(2)}-GV^{(2)})\Big\|_{\infty}\Big\|e_{T+1}^TGW^{(2)} \Big\|$$
Note that $ \Big\|e_t^T \phi'^{(2)}_{T+1}\Big\|_1\leq \Big\|\phi'^{(2)}_{T+1}\Big\|_{1,1}$ is the maximum absolute row sum of the softmax Jacobian. Again from Lemma A.6 in \cite{edelman2022inductive}, we know that this is bounded by $2.$ To bound $\Big\|(\tilde{G}\tilde{V}^{(2)}-GV^{(2)})\Big\|_{\infty}$ we start with the triangle inequality:
$$\Big\|(\tilde{G}\tilde{V}^{(2)}-GV^{(2)})\Big\|_{\infty} \lessapprox \Big\|(G(\tilde{V}^{(2)}-V^{(2)})\Big\|_{\infty}+\Big\|(\tilde{G}-G)V^{(2)}\Big\|_{\infty}$$
$$\Big\|G(\tilde{V}^{(2)}-V^{(2)})\Big\|_{\infty}\leq ||G||_{2,\infty}||\zeta_{V^{(2)}}|| \leq \sqrt{2} ||\zeta_{V^{(2)}}||$$
$ ||\zeta_{V^{(2)}}||$ is a chi-distributed variable with $d+1$ degrees of freedom, and this with probability $1-\delta,$
$$||\zeta_{V^{(2)}}||\leq \sigma \sqrt{d+1} R(\delta,d+1) \approx \sigma \sqrt{d}.$$
Therefore 
$$\Big\|G(\tilde{V}^{(2)}-V^{(2)})\Big\|_{\infty}\lessapprox \sigma \sqrt{2d}$$
Now, we have
$$\Big\|(\tilde{G}-G)V^{(2)}\Big\|_{\infty}\leq \Big\|(\tilde{G}-G)\Big\|_{2,\infty}||V^{(2)}||\leq D \leq \sqrt{\omega}\Big\|(\tilde{G}-G)\Big\|_{2,\infty}$$
Where we have used our bound on $D=\sum_{t=1}^{\omega}$ from \ref{parameterbounds}. Combining this with our bound on $||\tilde{g_t} - g_t||_2$ in \ref{mlpperturbationlemma}, this yields
$$\Big\|(\tilde{G}-G)V^{(2)}\Big\|_{\infty}\lessapprox 32 \sqrt{\omega}\sigma D_f^2$$
Thus 
$$\Big\|(\tilde{G}\tilde{V}^{(2)}-GV^{(2)})\Big\|_{\infty}\lessapprox \sigma\sqrt{2d} + 32 \sqrt{\omega}\sigma D_f^2\approx 32 \sqrt{\omega}\sigma D_f^2$$

We showed in \ref{parameterbounds} that $||(W^{(2)})^Tg_{T+1}||\lessapprox 2\sqrt{\omega} log(T).$ Combining these results together, we conclude that 
$$Term_B \leq 64 \sigma \omega D_f^2log(T)$$
Moving on to $Term_C, $ we have 
$$Term_C=||e_t^T \phi'^{(2)}_{T+1}GV^{(2)}e_{T+1}^T(G(\tilde{W}^{(2)}-W^{(2)})||$$
$$ \leq \Big\|e_t^T \tilde{\phi}'^{(2)}_{T+1}\Big\|_1\Big\|GV^{(2)}\Big\|_{\infty}\Big\|e_{T+1}^T(G(\tilde{W}^{(2)}-W^{(2)})\Big\|. $$
The first two terms have already been bounded, so we focus on the third term. To bound $\Big\|e_{T+1}^TG(\tilde{W}^{(2)}-W^{(2)})\Big\|,$ we follow a similar argument to that of \ref{perturbedfinalactivationlemma}.
$$ \Big\|g_{T+1}^T\big(\tilde{W}^{(2)}-W^{(2)}\big) \Big\|_{2}
 \leq 
 \Big\|G\big(\tilde{W}^{(2)}-W^{(2)}\big) \Big\|_{2,\infty}$$
Note that the perturbation matrix $\zeta_{W^{(2)}}=\tilde{W}^{(2)}-W^{(2)}$ has the same exact distribution if it is transposed. Therefore we can use the same bound from \ref{perturbedfinalactivationlemma} for $ \Big\|G\big(\tilde{W}^{(2)}-W^{(2)}\big)^T \Big\|_{2,\infty}$ to conclude that 
$$\Big\|G\big(\tilde{W}^{(2)}-W^{(2)}\big) \Big\|_{2,\infty} \lessapprox \sigma d$$
Therefore
$$Term_C \leq 2\sigma \sqrt{\omega} d$$
Now we consider $Term_D.$ First, note that 
$$ \Big\|(\tilde{G}-G)W^{(2)}\Big\|_{2,\infty}\leq ||\tilde{G}-G||_{2,\infty}||W^{(2)}||_{2}  $$
It is shown in \ref{parameterbounds} that $||W^{(2)}||_{2} \lessapprox 2\sqrt{\omega} log(T).$  We have already bounded $||\tilde{G}-G||_{2,\infty}\lessapprox 32\sigma D_f^2.$ Therefore, 
$$Term_D \lessapprox 64\sigma \sqrt{\omega}d D_f^2log(T)$$
Finally, we can bound
$$||\tilde{Term_{II}}-Term_{II}||\leq Term_A + Term_B + Term_C + Term_D $$
$$\lessapprox  1536 \sigma d D_f^2 \omega^{\frac{3}{2}} log^2(T)I[i<d+1] +64 \sigma \omega D_f^2log(T)I[i<d+1]$$
$$+2\sigma \sqrt{\omega} d+64\sigma \sqrt{\omega}d D_f^2log(T)I[i<d+1]$$
$$\approx 1536\sigma d D_f^2 \omega^{\frac{3}{2}} log^2(T)I[i<d+1] + 2\sigma \sqrt{\omega} d$$
For the perturbation in $Term_{III},$ we have 
 $$|(\tilde{Term_{III}}-Term_{III})e_i|\leq  I[t=T+1]\Big|\Big((\tilde{V}^{(2)})^T\tilde{G}^T{\tilde{\phi}'^{(2)}_{T+1}}\tilde{G}(\tilde{W}^{(2)})^T-(V^{(2)})^TG^T{\phi'^{(2)}_{T+1}}G(W^{(2)})^T\Big)e_i\Big|$$
 By the triangle equality, 
 $$\Big\|(\tilde{V}^{(2)})^T\tilde{G}^T{\tilde{\phi}'^{(2)}_{T+1}}\tilde{G}(\tilde{W}^{(2)})^T-(V^{(2)})^TG^T{\phi'^{(2)}_{T+1}}G(W^{(2)})^T\Big\|$$
 $$\leq \underbrace{\Big\|(\tilde{G}\tilde{V}^{(2)}-GV^{(2)})^T{\phi'^{(2)}_{T+1}}G(W^{(2)})^T\Big\|}_{Term_A}I[i<d+1]$$
 $$+ \underbrace{\Big\|(V^{(2)})^TG^T(\tilde{\phi'^{(2)}_{T+1}}-\phi'^{(2)}_{T+1})G(W^{(2)})^T\Big\|}_{Term_B}I[i<d+1]$$
$$\underbrace{\Big\|(V^{(2)})^TG^T\phi'^{(2)}_{T+1}G(\tilde{W}^{(2)}-W^{(2)})^T)\Big\|}_{Term_C}$$
  $$\underbrace{\Big\|(V^{(2)})^TG^T\phi'^{(2)}_{T+1}(\tilde{G}-G)(W^{(2)})^T)\Big\|}_{Term_D}I[i<d+1]$$

First we handle $Term_A.$ By \ref{deoralemma}, we have $\Big\|(\tilde{G}\tilde{V}^{(2)}-GV^{(2)})^T{\phi'^{(2)}_{T+1}}G(W^{(2)})^T\Big\|\leq 2\Big\|(\tilde{G}\tilde{V}^{(2)}-GV^{(2)})\Big\|_{\infty}\Big\|G(W^{(2)})^T\Big\|_{2,\infty} $
 In our above bound on the perturbation to $Term_{II},$ we showed that 
 $$\Big\|(\tilde{G}\tilde{V}^{(2)}-GV^{(2)})\Big\|_{\infty}\lessapprox 32 \sqrt{\omega}\sigma D_f^2$$
 In addition, we have from \ref{parameterbounds} that $\Big\|G(W^{(2)})^T\Big\|_{2,\infty}\lessapprox 2log(T)$
Thus $Term_A\lessapprox 128 \sqrt{\omega}\sigma D_f^2log(T)$
Next, we turn to $Term_B.$ 
$$Term_B \leq \Big\|(V^{(2)})^TG^T(\tilde{\phi'^{(2)}_{T+1}}-\phi'^{(2)}_{T+1})\Big\|_1\Big\|(G(W^{(2)})^T)^T\Big\|_{1,2}$$
$$\leq \Big\|(V^{(2)})^TG^T(\tilde{\phi'^{(2)}_{T+1}}-\phi'^{(2)}_{T+1})\Big\|_1\Big\|G(W^{(2)})^T\Big\|_{2,\infty}$$
In order to bound the term $\Big\|(V^{(2)})^TG^T(\tilde{\phi'^{(2)}_{T+1}}-\phi'^{(2)}_{T+1})\Big\|_1,$ we can use the same overall argument as above for bounding $\Big\| (\tilde{\phi}'^{(2)}_{T+1}-{\phi'^{(2)}_{T+1}})e_t\Big\|_1.$ The only difference in the argument is that each of the three main terms will have a factor of $|(V^{(2)})^TG^T\phi_{T+1}^{(2)}|$ out front (instead of $|e_t^T\phi_{T+1}^{(2)}|$). When we apply Holder's inequality to this, we gain a factor of $||GV^{(2)}||_{\infty}$ , which we have already bounded as $\lessapprox \sqrt{\omega}$. Therefore we conclude that 
$$\Big\|(V^{(2)})^TG^T(\tilde{\phi'^{(2)}_{T+1}}-\phi'^{(2)}_{T+1})\Big\|_1\lessapprox 384\sigma d D_f(d+2D_f)\omega log(T)$$

Using our bound on $\Big\|G(W^{(2)})^T\Big\|_{2,\infty}$ from \ref{parameterbounds}, it follows that
$$Term_B \lessapprox 1536\sigma d D_f^2\omega log^2(T)$$
Finally, for $Term_C $ we have by \ref{deoralemma} that 
$$\Big\|(V^{(2)})^TG^T\phi'^{(2)}_{T+1}G(\tilde{W}^{(2)}-W^{(2)})^T\Big\|\leq 2\Big\|GV^{(2)}\Big\|_{\infty}\Big\|G(\tilde{W}^{(2)}-W^{(2)})^T\Big\|_{2,\infty}$$
We have already bounded $\Big\|G(\tilde{W}^{(2)}-W^{(2)})^T\Big\|_{2,\infty}\lessapprox \sigma d$ in the proof of \ref{perturbedfinalactivationlemma}. Thus, we have
$$Term_C\lessapprox 2 \sigma \sqrt{\omega} d$$
For $Term_D,$ we similarly have
$$\Big\|(V^{(2)})^TG^T\phi'^{(2)}_{T+1}(\tilde{G}-G)(W^{(2)})^T\Big\|\leq 2\Big\|GV^{(2)}\Big\|_{\infty}\Big\|(\tilde{G}-G)(W^{(2)})^T\Big\|_{2,\infty}$$
Using standard matrix norm inequalities, we have
$$\Big\|(\tilde{G}-G)(W^{(2)})^T\Big\|_{2,\infty}\leq \Big\|\tilde{G}-G\Big\|_{2,\infty}\Big\|W^{(2)}\Big\|_{2}$$
We show in \ref{parameterbounds} that $\Big\|W^{(2)}\Big\|_{2}\lessapprox 2\sqrt{\omega} log(T). $ Therefore we have 
$$\Big\|(\tilde{G}-G)(W^{(2)})^T\Big\|_{2,\infty}\lessapprox 64 \sigma \sqrt{\omega} D_f^2 log(T) $$
Therefore
$$Term_D\lessapprox128\sigma \omega D_f^2 log(T)$$
Combining $Term_A,Term_B, Term_C,Term_D$ we have 
$$||\tilde{Term_{III}}-Term_{III}|| $$
$$\lessapprox 128\sigma \sqrt{\omega} D_f^2log(T)I[i<d+1]+1536\sigma d D_f^2\omega log^2(T)I[i<d+1]$$
$$+2 \sigma \sqrt{\omega} d + 128\sigma \omega D_f^2 log(T)I[i<d+1]$$
$$\approx 1536\sigma d D_f^2\omega log^2(T)I[i<d+1] + 2 \sigma \sqrt{\omega} d $$
Finally, we have
$$|(\nabla_{\tilde{g}_t} \mathcal{T}(X,\tilde{\Theta})-\nabla_{g_t} \mathcal{T}(X,\Theta))e_{d+1}| $$
$$\leq |(\tilde{Term_I} - Term_I)e_{d+1}| +|(\tilde{Term}_{II}-Term_{II})e_{d+1}|+|(\tilde{Term_{III}} - Term_{III})e_{d+1}|$$
$$\lessapprox 256\sigma d D_f^2\omega log(T)I[i=d+1]+\sigma R(\delta,1) $$
$$+1536\sigma d D_f^2 \omega^{\frac{3}{2}} log^2(T)I[i<d+1] + 2\sigma \sqrt{\omega} d$$
$$+1536\sigma d D_f^2\omega log^2(T)I[i<d+1] + 2 \sigma \sqrt{\omega} d $$
$$\approx 256\sigma d D_f^2\omega log(T)I[i=d+1] +1536\sigma d D_f^2 \omega^{\frac{3}{2}} log^2(T)I[i<d+1]$$

\end{proof}
\begin{theorem}
    \label{perturbedfinalsoftmaxjacobianlemma}
    Let $\tilde{\phi}'^{(2)}_{T+1},\phi'^{(2)}_{T+1}$ be the perturbed and unperturbed Jacobian of the softmax scores for the output position $T+1$ in the second attention layer in our transformer. Then the following bound holds for the maximum row/column 2-norm (noting that both $\tilde{\phi}'^{(2)}_{T+1},\phi'^{(2)}_{T+1}$ are symmetric) of $\tilde{\phi}'^{(2)}_{T+1}-{\phi'^{(2)}_{T+1}}:$
    $$\Big\|(\tilde{\phi}'^{(2)}_{T+1}-{\phi'^{(2)}_{T+1}})e_t\Big\|\leq \Big\|(\tilde{\phi}'^{(2)}_{T+1}-{\phi'^{(2)}_{T+1}})e_t\Big\|_1\leq  768\sigma d D_f^2\sqrt{\omega} log(T) \in O(\sigma D_f^2\sqrt{\omega} log^2(T))$$
\end{theorem}
\begin{proof}

    We use the well-known Jacobian of the softmax to obtain
$$(\tilde{\phi}'^{(2)}_{T+1}-{\phi'^{(2)}_{T+1}})=diag(\tilde{\phi}_{T+1}^{(2)})-\tilde{\phi}^{(2)}_{T+1}(\tilde{\phi}^{(2)}_{T+1})^T- diag(\phi^{(2)}_{T+1})+\phi^{(2)}_{T+1}(\phi_{T+1}^{(2)})^T$$
$$=diag(\tilde{\phi}_{T+1}^{(2)}-\phi^{(2)}_{T+1})+\tilde{\phi}_{T+1}^{(2)}(\tilde{\phi}_{T+1}^{(2)}-\phi^{(2)}_{T+1})^T+\phi^{(2)}_{T+1}(\tilde{\phi}_{T+1}^{(2)}-\phi^{(2)}_{T+1})^T$$
Taking the 1-norm of $(\tilde{\phi}'^{(2)}_{T+1}-{\phi'^{(2)}_{T+1}})e_t$ and then applying the triangle inequality, we have 
$$\Big\|(\tilde{\phi}'^{(2)}_{T+1}-{\phi'^{(2)}_{T+1}})e_t\Big\|$$
$$\leq \Big\|diag(\tilde{\phi}_{T+1}^{(2)}-\phi^{(2)}_{T+1})e_t\Big\|_1+\Big\|\tilde{\phi}_{T+1}^{(2)}(\tilde{\phi}_{T+1}^{(2)}-\phi^{(2)}_{T+1})^Te_t\Big\|_1+\Big\|\phi^{(2)}_{T+1}(\tilde{\phi}_{T+1}^{(2)}-\phi^{(2)}_{T+1})^Te_t\Big\|_1$$
The 1-norm of $diag(\tilde{\phi}_{T+1}^{(2)}-\phi^{(2)}_{T+1})e_t$ is upper bounded by  $||\tilde{\phi}_{T+1}^{(2)}-\phi^{(2)}_{T+1}||_{\infty}\leq||\tilde{\phi}_{T+1}^{(2)}-\phi^{(2)}_{T+1}||_{1}.$ Meanwhile,
$$\Big\|e_t^T\tilde{\phi}^{(2)}_{T+1}\big(\tilde{\phi}^{(2)}_{T+1}-\phi^{(2)}_{T+1}\big)^T\Big\|_{1}\approx\Big\|e_t^T\phi^{(2)}_{T+1}\big(\tilde{\phi}^{(2)}_{T+1}-\phi^{(2)}_{T+1}\big)^T\Big\|_{1}$$
$$=\big|e_t^T\phi^{(2)}_{T+1}\big|\Big\|\big(\tilde{\phi}^{(2)}_{T+1}-\phi^{(2)}_{T+1})^T)\Big\|_{1}$$
$$\leq \big\|\phi^{(2)}_{T+1}\big\|_1\Big\|\big(\tilde{\phi}^{(2)}_{T+1}-\phi^{(2)}_{T+1})^T)\Big\|_{1}=\Big\|\big(\tilde{\phi}^{(2)}_{T+1}-\phi^{(2)}_{T+1})^T)\Big\|_{1}$$
Putting this together, we have 
$$\Big\|(\tilde{\phi}'^{(2)}_{T+1}-{\phi'^{(2)}_{T+1}})e_t\Big\|\leq 3\Big\|\big(\tilde{\phi}^{(2)}_{T+1}-\phi^{(2)}_{T+1})^T)\Big\|_{1}$$

$\Big\|\tilde{\phi}^{(2)}_{T+1}-\phi^{(2)}_{T+1}\Big\|_1$ was already bounded with high probability in \ref{perturbedfinalactivationlemma}. Plugging in this bound, we have 
$$\Big\|(\tilde{\phi}'^{(2)}_{T+1}-{\phi'^{(2)}_{T+1}})e_t\Big\|\leq 768\sigma d D_f^2\sqrt{\omega} log(T)$$

\end{proof}
\begin{theorem}
\label{perturbed gradient bounds appendix}
Given a set of remainder perturbations $\zeta$  to the parameters $\Theta$ such that $|\zeta_i|\leq|\epsilon_i|,$ where $\epsilon \sim \mathcal{N}(0,\sigma^2)^n,$ the norm of the perturbed gradients satisfies:
$$|| \nabla \mathcal{T}(X,\Theta+\zeta)-\nabla \mathcal{T}(X,\Theta)||\leq G_p(\sigma,\omega,D_f,T)$$
where  $$G_p(\sigma,\omega,D_f,T) \in o(\sigma d D_f^{\frac{7}{2}}\omega^{2}log(T)^2)   $$

\end{theorem}
\begin{proof}

\subsection{$||\nabla_{\tilde{V}^{(2)}} \mathcal{T}(X,\tilde{\Theta})-\nabla_{V^{(2)}} \mathcal{T}(X,\Theta)|| $}
Consider the gradient $\nabla_{V^{(2)}} \mathcal{T}(X,\Theta) = G^T{\phi^{(2)}_{T+1}} $, the norm of which we upper bounded with $||\nabla_{V^{(2)}}\mathcal{T}(X,\Theta)||\leq \sqrt{2}$ above. If we perturb the network weights slightly, we have the following perturbed output:
$$\nabla_{\tilde{V}^{(2)}}\mathcal{T}(X,\tilde{\Theta})=\tilde{G}^T{\tilde{\phi}^{(2)}_{T+1}}$$
Taking the norm of the perturbation, we have
$$\Big\|\nabla_{\tilde{V}^{(2)}}\mathcal{T}(X,\tilde{\Theta})-\nabla_{V^{(2)}}\mathcal{T}(X,\Theta)\Big\|$$
$$=\Big\|\tilde{G}^T{\tilde{\phi}^{(2)}_{T+1}}-G^T{\phi^{(2)}_{T+1}}\Big\|\leq \Big\|(\tilde{G}-G)^T\phi^{(2)}_{T+1}\Big\|+\Big\|G^T(\tilde{\phi}^{(2)}_{T+1}-\phi^{(2)}_{T+1})\Big\|$$
$$\leq ||(\tilde{G} - G)^T||_{1,2}||\phi^{(2)}_{T+1}||_1+|| G^T||_{1,2}||\tilde{\phi}^{(2)}_{T+1}-\phi^{(2)}_{T+1}||_1$$
$$\leq ||\tilde{G} - G||_{2,\infty}+|| G||_{2,\infty}||\tilde{\phi}^{(2)}_{T+1}-\phi^{(2)}_{T+1}||_1$$
Note that $||\tilde{G}-G||_{2,\infty}$ appears here with $\phi^{(2)}_{T+1}$ (supported on $[\omega]$), so we use the active-position bound from \ref{mlpperturbationlemma}: $||\tilde{g}_t-g_t||\lessapprox 32 \sigma D_f^2$ for $t\in [\omega]\cup\{T+1\}.$
In addition, we showed in \ref{parameterbounds} that $||G||_{2,\infty}\lessapprox \sqrt{2},$ and in \ref{perturbedfinalactivationlemma} we showed that
$$||\tilde{\phi}^{(2)}_{T+1}-\phi^{(2)}_{T+1}||_1\lessapprox 256\sigma d D_f^2\sqrt{\omega} log(T)$$
Thus
$$||\nabla_{\tilde{V}^{(2)}} \mathcal{T}(X,\Theta+\zeta)-\nabla_{\tilde{V}^{(2)}} \mathcal{T}(X,\Theta)||\lessapprox 32\sigma D_f^2 + 256\sqrt{2}\sigma d D_f^2\sqrt{\omega} log(T)  $$
$$\approx  256\sqrt{2}\sigma d D_f^2\sqrt{\omega} log(T) $$
\subsection{$||\nabla_{\tilde{W}^{(2)}} \mathcal{T}(X, \tilde{\Theta})-\nabla_{W^{(2)}} \mathcal{T}(X, \Theta)||$}
Plugging our perturbed matrices into the formula for the unperturbed gradient of our transformer with respect to $W^{(2)},$
$$\nabla_{W^{(2)}}\mathcal{T}(X,\Theta)= g_{T+1}(V^{(2)})^TG^T{\phi'^{(2)}_{T+1}}G,$$
Noting that the RHS is a rank-1 outer product, we have
$$\Big\|\nabla_{W^{(2)}}\mathcal{T}(X,\Theta)\Big\|_F= ||g_{T+1}||\Big\|(V^{(2)})^TG^T{\phi'^{(2)}_{T+1}}G\Big\|,$$
By the triangle inequality, we have 
$$||\nabla_{\tilde{W}^{(2)}} \mathcal{T}(X, \Theta+\zeta)|| \leq ||\nabla_{W^{(2)}} \mathcal{T}(X, \Theta)|| + ||\nabla_{\tilde{W}^{(2)}} \mathcal{T}(X, \Theta+\zeta)-\nabla_{W^{(2)}} \mathcal{T}(X, \Theta)||$$
We can further decompose the perturbation term using the triangle inequality:
$$||\nabla_{\tilde{W}^{(2)}} \mathcal{T}(X, \Theta+\zeta)-\nabla_{W^{(2)}} \mathcal{T}(X, \Theta)|| \lessapprox\Big\|(\tilde{V}^{(2)})^T\tilde{G}^T{\tilde{\phi}'^{(2)}_{T+1}}\tilde{G}-(V^{(2)})^TG^T{\phi'^{(2)}_{T+1}}G\Big\|$$
$$\leq \Big\|(\tilde{V}^{(2)}-V^{(2)})^T\tilde{G}^T{\phi'^{(2)}_{T+1}}\tilde{G}\Big\|+\Big\|(V^{(2)})^T(\tilde{G}-G)^T{\phi'^{(2)}_{T+1}}\tilde{G}\Big\|+\Big\|(V^{(2)})^TG^T{\phi'^{(2)}_{T+1}}(\tilde{G}-G)\Big\|$$
$$+\Big\|(V^{(2)})^TG^T(\tilde{\phi}'^{(2)}_{T+1}-{\phi'^{(2)}_{T+1}})\tilde{G}\Big\|$$
From \ref{deoralemma}, and only keeping terms that are first order in $\sigma,$ we have
$$\Big\|(\tilde{V}^{(2)}-V^{(2)})^T\tilde{G}^T{\phi'^{(2)}_{T+1}}\tilde{G}\Big\|\leq 2 \Big\|\tilde{G}(\tilde{V}^{(2)}-V^{(2)})\Big\|_{\infty} ||\tilde{G}||_{2,\infty}\approx2 \Big\|G(\tilde{V}^{(2)}-V^{(2)})\Big\|_{\infty} ||G||_{2,\infty},$$
$$\Big\|(V^{(2)})^T(\tilde{G}-G)^T{\phi'^{(2)}_{T+1}}\tilde{G}\Big\|\leq 2\Big\|(\tilde{G}-G)V^{(2)}\Big\|_{\infty}||\tilde{G}||_{2,\infty}\approx 2\Big\|(\tilde{G}-G)V^{(2)}\Big\|_{\infty}||G||_{2,\infty}$$
$$\Big\|(V^{(2)})^TG^T{\phi'^{(2)}_{T+1}}(\tilde{G}-G)\Big\|\leq 2\Big\|GV^{(2)}\Big\|_{\infty}||\tilde{G}-G||_{2,\infty}$$
For the fourth term, by \ref{deoralemma} and \ref{perturbedfinalsoftmaxjacobianlemma},
$$\Big\|(V^{(2)})^TG^T(\tilde{\phi}'^{(2)}_{T+1}-{\phi'^{(2)}_{T+1}})\tilde{G}\Big\|\leq 2\Big\|GV^{(2)}\Big\|_{\infty}\Big\|(\tilde{\phi}'^{(2)}_{T+1}-{\phi'^{(2)}_{T+1}})e_t\Big\|_1||\tilde{G}||_{2,\infty}\lessapprox 768\sqrt{2}\sigma d D_f^2\sqrt{\omega} \log(T)$$

We have already bounded $||\tilde{G}-G||_{2,\infty} $ in \ref{mlpperturbationlemma} and $ ||G||_{2,\infty},\Big\|GV^{(2)}\Big\|_{\infty}$ in \ref{parameterbounds} . We show in \ref{perturbed transformer appendix} that the following two bounds hold:
$$\Big\|G(\tilde{V}^{(2)}-V^{(2)})\Big\|_{\infty}\lessapprox \sigma \sqrt{2d}$$
$$\Big\|(\tilde{G}-G)V^{(2)}\Big\|_{\infty}\lessapprox 32 \sqrt{\omega}\sigma D_f^2$$
 Putting this all together, we have
$$||\nabla_{\tilde{W}^{(2)}} \mathcal{T}(X, \Theta+\zeta)-\nabla_{W^{(2)}} \mathcal{T}(X, \Theta)|| \lessapprox 4\sigma\sqrt{d}  + 64\sqrt{2} \sigma \sqrt{\omega}D_f^2+64\sqrt{\omega}\sigma D_f^2 + 768\sqrt{2}\sigma d D_f^2\sqrt{\omega} \log(T) $$
$$\lessapprox 768\sqrt{2}\sigma d D_f^2\sqrt{\omega} \log(T)$$

\subsection{$||\nabla_{\tilde{W}^{(1)}} \mathcal{T}(X, \tilde{\Theta)}-\nabla_{W^{(1)}} \mathcal{T}(X, \Theta)||$}
We showed in \ref{unperturbed gradient bounds appendix} that 
$$\nabla_{W^{(1)}}\mathcal{T}(X,\Theta) =\sum_{t=1}^{\omega} c_tg'_tx_te_{d+1}^TV_1^TX^T\phi'_tX$$
However, this formula relied on the assumption that the dependency graph in our exact construction only involves the final dimensions of $g_t$ and $b_t.$ While true for our exact construction, this is no longer true when the parameters are perturbed. In the case of the perturbed construction, we need to use the more general formula provided in \ref{unperturbed gradient bounds appendix}, applied to the perturbed transformer:
$$\nabla_{\tilde{W}{^{(1)}}} \mathcal{T}(X,\tilde{\Theta}) =  \sum_{t=1}^{T+1}  \sum_{i=1}^{d+1} \nabla^T_{\tilde{g}_t}\mathcal{T}(X,\tilde{\Theta})\Big[\mathcal{J}_{\tilde{g}_t}(\tilde{b}_t)\Big]_{:,i}\nabla_{\tilde{W}^{(1)}}\tilde{B}_{t,i}$$
Note that the sums in these gradients now go all the way up to $T+1$ rather than stopping at $\omega.$ This is because in a perturbed transformer, the MLP outputs for ``empty'' positions beyond the $\omega$ required to compute our function may no longer all be $0,$ and may contribute to the transformers output and gradients.

In order to bound the perturbed gradient, we first bound the norm of the perturbation, using the triangle inequality:

$$\Big\|\nabla_{\tilde{W}{^{(1)}}} \mathcal{T}(X,\Theta+\zeta) -\nabla_{\tilde{W}{^{(1)}}} \mathcal{T}(X,\Theta) \Big\|_F$$
$$\lessapprox \underbrace{\sum_{t=1}^{T+1}  \sum_{i=1}^{d+1} \big|\big(\nabla^T_{\tilde{g}_t}\mathcal{T}(X,\tilde{\Theta})-\nabla^T_{g_t}\mathcal{T}(X,\Theta)\big) \big[\mathcal{J}_{g_t}(b_t)\big]_{:,i}\big|\Big\|\nabla_{W^{(1)}}B_{t,i}\Big\|_F}_{Term_1}$$
$$+\underbrace{\sum_{t=1}^{T+1}  \sum_{i=1}^{d+1} \big| \nabla^T_{g_t}\mathcal{T}(X,\Theta)\big(\big[\mathcal{J}_{\tilde{g}_t}(\tilde{b}_t)\big]_{:,i}-\big[\mathcal{J}_{g_t}(b_t)\big]_{:,i}\big)\big| \Big\|\nabla_{W^{(1)}}B_{t,i}\Big\|_F}_{Term_2}$$
$$+\underbrace{\sum_{t=1}^{T+1}  \sum_{i=1}^{d+1} \big|\nabla^T_{g_t}\mathcal{T}(X,\Theta)\big[\mathcal{J}_{g_t}(b_t)\big]_{:,i}\big|\Big\|\nabla_{\tilde{W}^{(1)}}\tilde{B}_{t,i}-\nabla_{W^{(1)}}B_{t,i}\Big\|_F}_{Term_3}$$

Recall that the only nonzero element of the unperturbed Jacobian $\mathcal{J}_{g_t}(b_t)$ is the last row and column, and therefore we can remove the sum over $i$ in $Term_1$ and $Term_3.$  For the final column of the Jacobian, we have $\big[\mathcal{J}_{g_t}(b_t)\big]_{:,d+1}=e_{d+1}\Big|\frac{\partial G_{t,d+1}}{\partial B_{t,d+1}}\Big|.$  Recall that in our construction, the slope of the MLP output with respect to small changes around the points in the grid of possible values for $B_{t,d+1}$ for is $0$, and thus we have $\Big|\frac{\partial G_{t,d+1}}{\partial B_{t,d+1}}\Big|=0.$ 
We show in \ref{mlpjacobianperturbationlemma} that
$$\Bigg|\Bigg|\Big[\mathcal{J}_{g_t}(b_t)\Big]_{:,i}-\Big[\mathcal{J}_{\tilde{g}_t}(\tilde{b}_t)\Big]_{:,i}\Bigg|\Bigg| \lessapprox 16\sigma D_f^2$$
In \ref{unperturbedmlpgradientlemma}, we show that 
 $$||\nabla^T_{g_t}\mathcal{T}(X,\Theta)|| \leq  |c_t| + 4\sqrt{\omega}log(T)I[t=T+1]+4\omega log(T)I[t\in [\omega]]$$
By Cauchy-Schwarz,
$$ \big| \nabla^T_{g_t}\mathcal{T}(X,\Theta)\big(\big[\mathcal{J}_{\tilde{g}_t}(\tilde{b}_t)\big]_{:,i}-\big[\mathcal{J}_{g_t}(b_t)\big]_{:,i}\big)\big|\leq \Big\|\nabla_{g_t}\mathcal{T}(X,\Theta)\Big\|\Bigg|\Bigg|\Big[\mathcal{J}_{g_t}(b_t)\Big]_{:,i}-\Big[\mathcal{J}_{\tilde{g}_t}(\tilde{b}_t)\Big]_{:,i}\Bigg|\Bigg|$$
$$\lessapprox 16\sigma D_f^2 ( |c_t| + 4\sqrt{\omega}log(T)I[t=T+1]+4\omega log(T)I[t\in [\omega]])$$
Taking the sum over $t$ and applying the indicator functions, we have the following:
$$\sum_{t=1}^{T+1}   \Big\|\nabla_{g_t}\mathcal{T}(X,\Theta)\Big\|\Bigg|\Bigg|\Big[\mathcal{J}_{g_t}(b_t)\Big]_{:,i}-\Big[\mathcal{J}_{\tilde{g}_t}(\tilde{b}_t)\Big]_{:,i}\Bigg|\Bigg| \Big\|\nabla_{W^{(1)}}B_{t,i}\Big\|_F$$
$$=16\sigma D_f^2 \sum_{t=1}^{T+1} \Bigg(  |c_t| + 4\sqrt{\omega}log(T)I[t=T+1]+4\omega log(T)I[t\in [\omega]]\Bigg)\Big\|\nabla_{W^{(1)}}B_{t,i}\Big\|_F$$
$$=16\sigma D_f^2  \Big(  \sqrt{\omega} + 4\sqrt{\omega}log(T) +4\omega^2log(T)\Big)\Big\|\nabla_{W^{(1)}}B_{t,i}\Big\|_F$$
$$\approx 64\sigma D_f^2\omega^2 log(T) \Big\|\nabla_{W^{(1)}}B_{t,i}\Big\|_F$$
Finally, we handle the $ \Big\|\nabla_{W^{(1)}}B_{t,i}\Big\|_F$ term. It was shown in \ref{unperturbed gradient bounds appendix} that
$$\Big\|\nabla_{W_1} B_{t,i}\Big\|_F \lessapprox 2\sqrt{2}I[i=d+1].$$
Thus we have
$$\Big\|\nabla_{\tilde{W}{^{(1)}}} \mathcal{T}(X,\Theta+\zeta) -\nabla_{\tilde{W}{^{(1)}}} \mathcal{T}(X,\Theta) \Big\|_F\lessapprox \sum_{i=1}^{d+1}  64\sigma D_f^2\omega^2 log(T) \Big\|\nabla_{W^{(1)}}B_{t,i}\Big\|_F$$
$$ \leq   128\sqrt{2}\sigma D_f^2\omega^2 log(T) $$
\subsection{$||\nabla_{\tilde{V}^{(1)}} \mathcal{T}(X, \tilde{\Theta})-\nabla_{V^{(1)}} \mathcal{T}(X, \Theta)||$}
We once again use the more general expression for our gradient with respect to $V^{(1)}$ that uses the full Jacobian. 
$$\nabla_{\tilde{V}{^{(1)}}} \mathcal{T}(X,\tilde{\Theta}) =  \sum_{t=1}^{T+1}  \sum_{i=1}^{d+1} \nabla^T_{\tilde{g}_t}\mathcal{T}(X,\tilde{\Theta})\Big[\mathcal{J}_{\tilde{g}_t}(\tilde{b}_t)\Big]_{:,i}\nabla_{\tilde{V}^{(1)}}\tilde{B}_{t,i}$$
After canceling out the terms involving the unperturbed Jacobian of $g_t$ with respect to $b_t$ and keeping only the term involving the perturbation of the Jacobian, we have:
$$\Big\|\nabla_{\tilde{V}{^{(1)}}} \mathcal{T}(X,\Theta+\zeta) -\nabla_{\tilde{V}{^{(1)}}} \mathcal{T}(X,\Theta) \Big\|_F$$
$$\lessapprox \sum_{t=1}^{T+1}  \sum_{i=1}^{d+1} \big| \nabla^T_{g_t}\mathcal{T}(X,\Theta)\big(\big[\mathcal{J}_{\tilde{g}_t}(\tilde{b}_t)\big]_{:,i}-\big[\mathcal{J}_{g_t}(b_t)\big]_{:,i}\big)\big| \Big\|\nabla_{V^{(1)}}B_{t,i}\Big\|_F$$
$$\approx 64\sigma D_f^2\omega^2 log(T)\sum_{i=1}^{d+1}  \Big\|\nabla_{V^{(1)}}B_{t,i}\Big\|_F$$
It was shown in \ref{unperturbed gradient bounds appendix} that
$$\Big\|\nabla_{V^{(1)}}B_{t,i}\Big\|_F\lessapprox \sqrt{2}$$
Thus
$$\Big\|\nabla_{\tilde{V}{^{(1)}}} \mathcal{T}(X,\Theta+\zeta) -\nabla_{\tilde{V}{^{(1)}}} \mathcal{T}(X,\Theta) \Big\|_F\lessapprox  64\sqrt{2}\sigma dD_f^2\omega^2 log(T) $$

\subsection{$||\nabla_{\tilde{M}} \mathcal{T}(X, \tilde{\Theta})-\nabla_{M} \mathcal{T}(X, \Theta)||$}
It was shown in \ref{unperturbed gradient bounds appendix} that 
$$\nabla_{M} \mathcal{T}(X,\Theta) =\sum_{t=1}^{\omega} c_tb_te_{d+1}^TF^Tdiag(I\big[M^Tb_t+\Gamma>0\big])$$
In the case of the perturbed construction, we once again need to use a slightly more general formula that accounts for the possibility of the transformer output depending on the first $d$ dimensions of $g_t$.
$$\nabla_{\tilde{M} }\mathcal{T}(X,\tilde{\Theta}) =  \sum_{t=1}^{T+1}  \sum_{i=1}^{d+1} \frac{\partial \mathcal{T}(X,\tilde{\Theta})}{\partial \tilde{G}_{t,i}}\nabla_{\tilde{M}}\tilde{G}_{t,i}$$
$$\Big\|\nabla_{\tilde{M}} \mathcal{T}(X,\Theta+\zeta)-\nabla_M \mathcal{T}(X,\Theta)\Big\|_F$$
$$\lessapprox \sum_{t=1}^{T+1}  \sum_{i=1}^{d+1} \Bigg|\frac{\partial \mathcal{T}(X,\tilde{\Theta})}{\partial \tilde{G}_{t,i}}-\frac{\partial \mathcal{T}(X,\Theta)}{\partial G_{t,i}}\Bigg|\Big\|\nabla_{M}G_{t,i}\Big\|_F+ \Bigg|\frac{\partial \mathcal{T}(X,\Theta)}{\partial G_{t,i}}\Bigg|\Big\|\nabla_{\tilde{M}}\tilde{G}_{t,i}-\nabla_{M}G_{t,i}\Big\|_F$$
It was shown in \ref{unperturbed gradient bounds appendix} that 
$$\nabla_{M}G_{t,i}=b_te_{i}^TF^Tdiag(I\big[M^Tb_t+\Gamma>0\big]),$$
and that the norm of this intra-layer gradient term can be bounded as:
$$\Big\|\nabla_{M}G_{t,i}\Big\|_F \lessapprox8D_f^{\frac{3}{2}}I[i=d+1\wedge t\leq \omega ]$$
In \ref{perturbedmlpgradientlemma}, it was shown that 
$$|\langle\nabla_{\tilde{g}_t} \mathcal{T}(X,\tilde{\Theta})-\nabla_{g_t} \mathcal{T}(X,\Theta),e_{d+1}\rangle| \lessapprox  256\sigma d D_f^2\omega log(T)$$
Note that, unlike the unperturbed gradient with respect to $M,$ we no longer have the unperturbed derivative with respect to $G_{t,d+1},$ $ \frac{\partial \mathcal{T}(X,\Theta)}{\partial G_{t,d+1}} =  c_t,$ to cancel out the gradient for the ``inactive'' positions after the first attention layer (with $t>\omega).$ However, since the unperturbed intra-layer gradient $\nabla_{M}G_{t,i}$ is also $0$ for $t > \omega, $ we can still avoid the $T$ dependency in our gradient perturbation.
Thus we have
$$\sum_{t=1}^{T+1}  \sum_{i=1}^{d+1} \Bigg|\frac{\partial \mathcal{T}(X,\tilde{\Theta})}{\partial \tilde{G}_{t,i}}-\frac{\partial \mathcal{T}(X,\Theta)}{\partial G_{t,i}}\Bigg|\Big\|\nabla_{M}G_{t,i}\Big\|_F$$
$$= 8D_f^{\frac{3}{2}}\sum_{t=1}^{\omega}  \Bigg|\frac{\partial \mathcal{T}(X,\tilde{\Theta})}{\partial \tilde{G}_{t,d+1}}-\frac{\partial \mathcal{T}(X,\Theta)}{\partial G_{t,d+1}}\Bigg|  $$
$$\leq 2048\sigma d D_f^{\frac{7}{2}} \omega^{2} log(T)$$
In addition, note that by Cauchy-Schwarz,
$$\Bigg|\frac{\partial \mathcal{T}(X,\Theta)}{\partial G_{t,i}}\Bigg|=\Big|\nabla^T_{g_t}\mathcal{T}(X,\Theta)e_i \Big| \leq \Big\|\nabla_{g_t}\mathcal{T}(X,\Theta)\Big\|  $$
We have already shown in \ref{perturbed gradient bounds appendix} that
$$\sum_{t=1}^{T+1} ||\nabla_{g_t}\mathcal{T}(X,\Theta)|| \lessapprox \sqrt{\omega} + 4\sqrt{\omega}log(T) +4 \omega^2 log(T) \approx  4\omega^2 log(T)$$
We now consider the term $\Big\|\nabla_{\tilde{M}}\tilde{G}_{t,i}-\nabla_{M}G_{t,i}\Big\|_F.$ Starting from our formula for the unperturbed gradient, we have  
$$\Big\|\nabla_{\tilde{M}}\tilde{G}_{t,i}-\nabla_{M}G_{t,i}\Big\|_F\leq \Big\|\tilde{b}_te_i^T\tilde{F}^Tdiag(I\big[\tilde{M}^T\tilde{b}_t+\tilde{\Gamma}>0\big])-b_te_i^TF^Tdiag(I\big[M^Tb_t+\Gamma > 0\big])\Big\|$$
$$\lessapprox||b_t||\Big\|e_i^TF^T\Big(diag(I\big[\tilde{M}^T\tilde{b}_t+\tilde{\Gamma}>0\big])-diag(I\big[M^Tb_t+\Gamma>0\big])\Big)\Big\|$$
$$+  ||b_t||\Big\|e_i^T(\tilde{F}-F)^Tdiag(I\big[M^Tb_t+\Gamma>0\big])\Big\|$$
$$+ ||\tilde{b}_t-b_t||\Big\|e_i^TF^Tdiag(I\big[M^Tb_t+\Gamma>0\big])\Big\|$$
Just as we saw in \ref{mlpperturbationlemma} and then again in \ref{mlpjacobianperturbationlemma}, with some small probability, some of the neurons in the perturbed transformer may flip from being active to inactive.  We use the same assumption that $\sigma \leq \frac{1}{16d D_f}$ to that this will happen with probability at most $\delta,$ and thus $diag(I\big[\tilde{M}^T\tilde{b}_t+\tilde{\Gamma}>0\big])-diag(I\big[M^Tb_t+\Gamma>0\big])= \mathbf{0}_{4(D_f+1)\times 4(D_f+1)}$ , and thus the first term will be $0.$

Moving to the second additive term, we have with probability $1-\delta$
$$\Big\|e_i^T(\tilde{F}-F)^Tdiag(I\big[M^Tb_t+\Gamma >0\big])\Big\|$$
$$\leq ||\zeta_F^T||_{2,\infty}\leq||\zeta_F||_{1,2}\leq||\epsilon_F||_{1,2}\leq \sigma \sqrt{4(D_f+1)}R(\frac{\delta}{d+1},4(D_f+1))$$
$$\approx 2\sigma D_f^{\frac{1}{2}}$$
For the last term, recall that all but the last column of $F$ is $0,$ we have 
$$\Big\|e_i^TF^Tdiag(I\big[b_tM+\Gamma^T>0\big])\Big\|\leq ||F||_{1,2}=4D_f\sqrt{4(D_f+1)} I[i=d+1] \approx 8D_f^{\frac{3}{2}}I[i=d+1]$$
Using this, and combining our bound for $||\tilde{b}_t-b_t||\lessapprox 4\sigma d$ from \ref{attentionperturbationlemma}, we have

$$\Big\|\nabla_{\tilde{M}}\tilde{G}_{t,i}-\nabla_{M}G_{t,i}\Big\|_F 
\lessapprox 2\sigma D_f^{\frac{1}{2}}  + 32\sigma d D_f^{\frac{3}{2}}I[i=d+1]$$
Therefore
$$\sum_{t=1}^{T+1}  \sum_{i=1}^{d+1} \Bigg|\frac{\partial \mathcal{T}(X,\Theta)}{\partial G_{t,i}}\Bigg|\Big\|\nabla_{\tilde{M}}\tilde{G}_{t,i}-\nabla_{M}G_{t,i}\Big\|_F $$
$$\leq  4\omega^2 log(T)  \sum_{i=1}^{d+1}  ||b_t||\Big\|e_i^T(\tilde{F}-F)^Tdiag(I\big[M^Tb_t+\Gamma>0\big])\Big\|$$
$$+ 4\omega^2 log(T)  \sum_{i=1}^{d+1} ||\tilde{b}_t-b_t||\Big\|e_i^TF^Tdiag(I\big[M^Tb_t+\Gamma>0\big])\Big\|$$
$$\leq    4\omega^2 log(T)\sum_{i=1}^{d+1} \Big(2\sigma D_f^{\frac{1}{2}}\Big)$$
$$+ 4\omega^2  log(T)  (4\sigma d)(8D_f^{\frac{3}{2}}) $$
$$\leq 8\sigma d \omega^2  D_f^{\frac{1}{2}}log(T) +128 \sigma d \omega^2 log(T)D_f^{\frac{3}{2}}$$
$$\lessapprox 128 \sigma d \omega^2  log(T)D_f^{\frac{3}{2}}$$

\subsection{$||\nabla_{\tilde{\Gamma}} \mathcal{T}(X, \tilde{\Theta})-\nabla_\Gamma \mathcal{T}(X, \Theta)||$}

We have for the perturbed transformer the following expression for the gradient:
$$\nabla_{\tilde{\Gamma} }\mathcal{T}(X,\tilde{\Theta}) =  \sum_{t=1}^{T+1}  \sum_{i=1}^{d+1} \frac{\partial \mathcal{T}(X,\tilde{\Theta})}{\partial \tilde{G}_{t,i}}\nabla_{\tilde{\Gamma}}\tilde{G}_{t,i}$$
We again decompose the perturbation in the norm of the gradient using the triangle inequality:
$$\Big\|\nabla_{\tilde{\Gamma}} \mathcal{T}(X,\Theta+\zeta)-\nabla_{\Gamma} \mathcal{T}(X,\Theta)\Big\|_F$$
$$\lessapprox \sum_{t=1}^{T+1}  \sum_{i=1}^{d+1} \Bigg|\frac{\partial \mathcal{T}(X,\tilde{\Theta})}{\partial \tilde{G}_{t,i}}-\frac{\partial \mathcal{T}(X,\Theta)}{\partial G_{t,i}}\Bigg|\Big\|\nabla_{\Gamma}G_{t,i}\Big\|_F+ \Bigg|\frac{\partial \mathcal{T}(X,\Theta)}{\partial G_{t,i}}\Bigg|\Big\|\nabla_{\tilde{\Gamma}}\tilde{G}_{t,i}-\nabla_{\Gamma}G_{t,i}\Big\|_F$$
We just upper bounded both of the inter-layer derivatives, $ \Bigg|\frac{\partial \mathcal{T}(X,\tilde{\Theta})}{\partial \tilde{G}_{t,i}}-\frac{\partial \mathcal{T}(X,\Theta)}{\partial G_{t,i}}\Bigg|$ and $\Bigg|\frac{\partial \mathcal{T}(X,\Theta)}{\partial G_{t,i}}\Bigg|.$ It remains to bound the perturbation in the intra-layer gradient. We show in \ref{unperturbed gradient bounds appendix} that 
$$||\nabla_\Gamma G_{t,i}||\lessapprox 8 D_f^{\frac{3}{2}}I[i=d+1 \wedge t\leq \omega]$$
The first sum becomes 
$$\sum_{t=1}^{T+1}  \sum_{i=1}^{d+1} \Bigg|\frac{\partial \mathcal{T}(X,\tilde{\Theta})}{\partial \tilde{G}_{t,i}}-\frac{\partial \mathcal{T}(X,\Theta)}{\partial G_{t,i}}\Bigg|\Big\|\nabla_{\Gamma}G_{t,i}\Big\|_F$$
$$\leq  8 D_f^{\frac{3}{2}}\sum_{t=1}^{\omega} 256\sigma d D_f^2\omega log(T)$$
$$=2048\sigma d D_f^{\frac{7}{2}} \omega^{2}log(T)$$
Now, turning to the second sum, we have 
$$||\nabla_{\tilde{\Gamma}}\tilde{G}_{t,i}-\nabla_{\Gamma}G_{t,i}||  $$
$$\leq  \Big\|diag(I\big[\tilde{M}^T\tilde{b}_t+\tilde{\Gamma}\geq0\big])\tilde{F}e_{i}-diag(I\big[M^Tb_t+\Gamma\geq0\big])Fe_{i}\Big\|$$
We again apply the triangle inequality, and apply our first-order (in $\sigma$) approximation to the perturbation, noting that the term involve the difference of indicators vanishes, assuming that $\sigma $ is small in the sense of \ref{mlpperturbationlemma}. Thus we obtain
$$||\nabla_{\tilde{\Gamma}}\tilde{G}_{t,i}-\nabla_{\Gamma}G_{t,i}||  \lessapprox \Big\|diag(I\big[M^Tb_t+\Gamma>0\big])(\tilde{F}-F)e_{d+1}\Big\|$$
$$\leq ||\zeta_{F}||_{1,2} \leq ||\epsilon_{F}||_{1,2}\leq \sigma \sqrt{4(D_f+1)} R(\frac{\delta}{d+1},4(D_f+1))\approx 2\sigma\sqrt{D_f}$$
Putting this together, for the second sum we have 
$$\sum_{t=1}^{T+1}  \sum_{i=1}^{d+1}\Bigg|\frac{\partial \mathcal{T}(X,\Theta)}{\partial G_{t,i}}\Bigg|\Big\|\nabla_{\tilde{\Gamma}}\tilde{G}_{t,i}-\nabla_{\Gamma}G_{t,i}\Big\|_F$$
$$\leq 4\sqrt{\omega} log(T) \sum_{i=1}^{d+1}\Big\|\nabla_{\tilde{\Gamma}}\tilde{G}_{t,i}-\nabla_{\Gamma}G_{t,i}\Big\|_F$$
$$\leq 4\sqrt{\omega} log(T)\Big(2\sigma\sqrt{D_f}\Big)(d+1)\approx 8\sigma D_f^{\frac{1}{2}}\sqrt{\omega} log(T)d$$
Putting the two sums together, we have 
$$\Big\|\nabla_{\tilde{\Gamma}} \mathcal{T}(X,\Theta+\zeta)-\nabla_{\Gamma} \mathcal{T}(X,\Theta)\Big\|_F$$
$$\leq 2048\sigma d D_f^{\frac{7}{2}} \omega^{2}log(T)  + 8\sigma D_f^{\frac{1}{2}}\sqrt{\omega} log(T)d$$

\subsection{$||\nabla_{\tilde{F}} \mathcal{T}(X, \tilde{\Theta})-\nabla_{F} \mathcal{T}(X, \Theta)||$}
We follow the same overall flow as the previous two gradient calculations. We showed in \ref{unperturbed gradient bounds appendix} that 
$$||\nabla_{F}G_{t,i}|| \leq ||\big(M^Tb_t + \Gamma\big)_+e_i^T||\lessapprox  2\sqrt{D_f}I[t\leq \omega]$$
In addition, we note that 

$$||\nabla_{\tilde{F}}\tilde{G}_{t,i}-\nabla_{F}G_{t,i}||_F  $$
$$\leq  \Big\|\big(\tilde{M}^T\tilde{b}_t + \tilde{\Gamma}\big)_+e_i^T-\big(M^Tb_t + \Gamma\big)_+e_i^T\Big\|_F=\Big\|\big(\tilde{M}^T\tilde{b}_t + \tilde{\Gamma}\big)_+-\big(M^Tb_t + \Gamma\big)_+\Big\|$$
Using the triangle inequality, we have
$$\lessapprox \Big\|\big(\tilde{M}^Tb_t + \tilde{\Gamma}\big)-\big(M^Tb_t + \Gamma\big)\Big\|+\Big\|\big(M^T(\tilde{b}_t-b_t) \Big\|$$
Where we have used the $1-$lipschitzness of the max function. It was shown in \ref{mlpperturbationlemma} that
$$||(\tilde{M}^T-M^T)b_t+(\tilde{\Gamma}-\Gamma)||_2 \lessapprox  4\sigma  D_f^{\frac{1}{2}}$$
and it is easy to see that $\Big\|M^T(\tilde{b}_t-b_t) \Big\|\leq \sqrt{4(D_f+1)}||\tilde{b}_t-b_t|| \lessapprox 8\sigma D_f^{\frac{1}{2}} d$ . Thus we obtain
$$\Big\|\nabla_{\tilde{F}}\tilde{G}_{t,i}-\nabla_{F}G_{t,i}\Big\|_F \lessapprox  4\sigma D_f^{\frac{1}{2}}+ 8\sigma D_f^{\frac{1}{2}} d \approx  8\sigma D_f^{\frac{1}{2}} d $$
Putting this all together, we get 
$$\Big\|\nabla_{\tilde{F}} \mathcal{T}(X,\Theta+\zeta)-\nabla_{F} \mathcal{T}(X,\Theta)\Big\|_F$$
$$\lessapprox \sum_{t=1}^{T+1}  \sum_{i=1}^{d+1} \Bigg|\frac{\partial \mathcal{T}(X,\tilde{\Theta})}{\partial \tilde{G}_{t,i}}-\frac{\partial \mathcal{T}(X,\Theta)}{\partial G_{t,i}}\Bigg|\Big\|\nabla_{F}G_{t,i}\Big\|_F+ \Bigg|\frac{\partial \mathcal{T}(X,\Theta)}{\partial G_{t,i}}\Bigg|\Big\|\nabla_{\tilde{F}}\tilde{G}_{t,i}-\nabla_{F}G_{t,i}\Big\|_F$$
$$\leq \Big(256\sigma d D_f^2\omega log(T)\Big)\Big(2 \sqrt{D_f}\Big)(d+1)\omega + 4\sqrt{\omega} log(T) \Big(  4\sigma \sqrt{D_f}+8\sigma D_f^{\frac{1}{2}} d\Big)(d+1) $$
$$\approx 512 \sigma D_f^{\frac{5}{2}}\omega^2 d^2 + 32 \sigma D_f^{\frac{1}{2}} \omega^2 d^2$$
\subsection{Final Result for Perturbed Gradient Norm Bounds}
Putting this all together, we have that with probability at least $1-\delta$ (for exponentially small $\delta$),
$$||\nabla_{\tilde{\Theta}} \mathcal{T}(X, \tilde{\Theta})-\nabla_{\Theta} \mathcal{T}(X, \Theta)||$$
$$\lessapprox \underbrace{256\sqrt{2}\sigma d D_f^2\sqrt{\omega} log(T)}_{||\nabla_{\tilde{V}^{(2)}} -\nabla_{V^{(2)}} ||} + \underbrace{192\sigma\sqrt{\omega}D_f^2}_{||\nabla_{\tilde{W}^{(2)}} -\nabla_{W^{(2)}} ||}+ \underbrace{128\sqrt{2}\sigma D_f^2\omega^2 log(T)}_{||\nabla_{\tilde{W}^{(1)}}-\nabla_{W^{(1)}} ||}+\underbrace{64\sqrt{2}\sigma dD_f^2\omega^2 log(T)}_{||\nabla_{\tilde{V}^{(1)}}-\nabla_{V^{(1)}} ||}$$
$$+\underbrace{128 \sigma d \omega^2  log(T)D_f^{\frac{3}{2}}}_{||\nabla_{\tilde{M}} -\nabla_{M} ||}+ \underbrace{2048\sigma d D_f^{\frac{7}{2}} \omega^{2}log(T)  + 8\sigma D_f^{\frac{1}{2}}\sqrt{\omega} log(T)d}_{||\nabla_{\tilde{\Gamma}} -\nabla_\Gamma ||}+\underbrace{512 \sigma D_f^{\frac{5}{2}}\omega^2 d^2 + 32 \sigma D_f^{\frac{1}{2}} \omega^2 d^2}_{||\nabla_{\tilde{F}} -\nabla_{F} ||}$$
$$\in O\Big(\sigma D_f^{\frac{7}{2}}\omega^{2}log^2(T)\Big)$$
\end{proof}

\section{Perturbed Hessian Norm Bounds}
\begin{theorem}
\label{perturbedhessianbounds}
Given our transformer construction, a set of $\sigma-$ normal perturbations $\epsilon,$ and a corresponding set of remainder perturbations $\zeta$ (as defined in \ref{average sharpness approximation appendix}), the operator norm of the perturbation to the Hessian obeys:
$$||\nabla^2_{\Theta} \mathcal{T}(X,\Theta+\zeta)-\nabla^2_{\Theta} \mathcal{T}(X,\Theta)||\lessapprox H_p(\sigma, \omega, D_f,T,d)$$

where
$$H_p(\sigma, \omega, D_f,T,d)\in o(\sigma D_f^{\frac{7}{2}}\omega^{\frac{5}{2}}log(T)^{\frac{7}{2}})$$
Combined with our bound on the operator norm of the exact hessian in \ref{exact hessian bounds appendix}, it follows that 
$$||\nabla^2_{\Theta} \mathcal{T}(X,\Theta+\zeta)||\leq ||\nabla^2_{\Theta} \mathcal{T}(X,\Theta)||+||\nabla^2_{\Theta} \mathcal{T}(X,\Theta+\zeta)-\nabla^2_{\Theta} \mathcal{T}(X,\Theta)||$$
$$\lessapprox H_u(\omega, D_f,d)+H_p(\sigma,\omega, D_f,T,d)$$
\end{theorem}
\begin{proof}
We calculate the perturbations in the Hessian norm by calculating the perturbation in each of the 6 non-zero terms contributing to the Hessian norm, as presented in the previous subsection. We illustrate the overall approach with an example. Suppose we want to upper bound the perturbation in the hessian norm contributed by the perturbed second derivative $\nabla^2_{\tilde{M}\tilde{V}^{(2)}}\mathcal{T}(X,\Theta)$
The perturbation in the hessian is given by 
$$\nabla_{\tilde{M}} \Big(\big[\nabla_{\tilde{V}^{(
2)}}\mathcal{T}(X,\Theta)\big]_{i} \Big)-\nabla_{M} \Big(\big[\nabla_{V^{(2)}}\mathcal{T}(X,\Theta)\big]_{i} \Big)$$
$$=\sum_{t=1}^{T+1}\sum_{j=1}^{d+1}\frac{\Big[\nabla_{\tilde{V}^{(2)}} \mathcal{T}(X,\tilde{\Theta}) \Big]_i}{ \partial\tilde{G}_{t,j}} \nabla_{\tilde{M}}\tilde{G}_{t,j}-\frac{\Big[\nabla_{V^{(2)}} \mathcal{T}(X,\Theta) \Big]_i}{ \partial G_{t,j}} \nabla_{M}G_{t,j}$$
Recall that the $i^{th}$ row of $\nabla^2_{\tilde{M}\tilde{V}^{(2)}}\mathcal{T}(X,\tilde{\Theta})-\nabla^2_{MV^{(2)}}\mathcal{T}(X,\Theta)$ is equal to $$Vec\Bigg(\nabla_{\tilde{M}} \Big(\big[\nabla_{\tilde{V}^{(
2)}}\mathcal{T}(X,\Theta)\big]_{i} \Big)-\nabla_{M} \Big(\big[\nabla_{V^{(2)}}\mathcal{T}(X,\Theta)\big]_{i} \Big)\Bigg)$$
for a fixed $i.$ So the contribution to the overall operator norm of the perturbation in $\nabla^2_{\tilde{M}\tilde{V}^{(2)}}\mathcal{T}(X,\tilde{\Theta})$ is given by 
$$2  \max_{||v||=1} concat(s_1,...s_{d+1})^T\Big(\nabla^2_{{W^{(2)}}V^{(2)}}\mathcal{T}(X,\Theta) -\nabla^2_{{\tilde{W}^{(2)}}\tilde{V}^{(2)}}\mathcal{T}(X,\tilde{\Theta}) \Big)concat(p_1,...,p_{d+1})$$
$$\leq 2\sum_{i=1}^{d+1} |s_i| \Big\| Vec\Big(\nabla_{\tilde{M}} \Big(\big[\nabla_{\tilde{V}^{(
2)}}\mathcal{T}(X,\Theta)\big]_{i} \Big)-\nabla_{M} \Big(\big[\nabla_{V^{(2)}}\mathcal{T}(X,\Theta)\big]_{i} \Big)\Big) \Big\| ||\mathbf{p}||$$
$$\leq 2\sum_{i=1}^{d+1} \sum_{k=1}^{d+1} |s_i|  \Big\|e_k^T\nabla_{\tilde{M}} \Big(\big[\nabla_{\tilde{V}^{(
2)}}\mathcal{T}(X,\Theta)\big]_{i} \Big)-e_k^T\nabla_{M} \Big(\big[\nabla_{V^{(2)}}\mathcal{T}(X,\Theta)\big]_{i} \Big) \Big\| ||\mathbf{p_k}||$$
Note that the $s_i$ and $p_j$ here are different from the $s_i$ and $p_i$ involved in calculating the unperturbed hessian, and even different than the singular vectors involved in the calculation of the overall hessian perturbation (as opposed to the operator norm of a sub-matrix within the overall Hessian). However, since the only thing we care about is that these vectors and all projections onto different subspaces of coordinates) have a 2-norm that is upper bounded by 1, we use the same notation.  We refer to the contribution of the perturbation of a given sub-matrix to the operator of the overall Hessian perturbation as $\Delta_{Term_i
},$ where $Term_i$ is the term corresponding to the given sub-matrix in \ref{exact hessian bounds appendix}. 

Ultimately, we will use the triangle inequality to conclude that 
$$\Big\|\nabla^2_{\tilde{\Theta}}\mathcal{T}(X,\tilde{\Theta})\Big\|\leq \Big\|\nabla^2_{\Theta}\mathcal{T}(X,\Theta)\Big\|+\Big\|\nabla^2_{\tilde{\Theta}}\mathcal{T}(X,\tilde{\Theta})-\nabla^2_{\Theta}\mathcal{T}(X,\Theta)\Big\|$$
$$\leq \Big\|\nabla^2_{\Theta}\mathcal{T}(X,\Theta)\Big\|+\sum_{i=1}^{6}\Delta_{Term_i}$$
\subsection{$\Delta_{Term_1}$} 
Recall from \ref{exact hessian bounds appendix} that  $$Term_1 =2  \max_{||v||=1} concat(s_1,...s_{d+1})^T\Big(\nabla^2_{{W^{(2)}}V^{(2)}}\mathcal{T}(X,\Theta) \Big)concat(p_1,...,p_{d+1}).$$
Recall our formula for $\Big[\nabla_{V^{(2)}} \mathcal{T}(X,\Theta) \Big]_i =(G_{:,i})^T{\phi^{(2)}_{T+1}}$ from \ref{exact hessian bounds appendix}. Whereas in the exact hessian case, $G_{:,i}$ is independent of all parameter matrices (since the first $d$ columns of $G,$ like $X,$ encode purely positional information), in the perturbed case $G_{:,i}$ may in fact depend on the perturbations of some of the parameter matrices. To handle this more general case, we will use the product rule:
$$\nabla_{\tilde{W}^{(2)}}\Big(\Big[\nabla_{\tilde{V}^{(2)}} \mathcal{T}(X,\tilde{\Theta}) \Big]_i\Big) =\nabla_{\tilde{W}^{(2)}}\Big((\tilde{G}_{:,i})^T{\tilde{\phi}^{(2)}_{T+1}}\Big)=\sum_{t=1}^{T+1} \nabla_{\tilde{W}^{(2)}}\big(\big[\tilde{\phi}^{(2)}_{T+1}\big]_t\big)\tilde{G}_{t,i}+\nabla_{\tilde{W}^{(2)}}\big(\tilde{G}_{t,i}\big)\big[\tilde{\phi}^{(2)}_{T+1}\big]_t$$
In this case, the perturbed MLP output $\tilde{G}_{:,i}$ does not depend on the (downstream) parameters $\tilde{W}^{(2)},$ and $\nabla_{\tilde{W}^{(2)}}\big(\tilde{G}_{:,i}\big){\tilde{\phi}^{(2)}_{T+1}}=0.$ It remains to calculate $\nabla_{\tilde{W}^{(2)}}\big({\tilde{\phi}^{(2)}_{T+1}}\big)\tilde{G}_{:,i}.$ Recall from \ref{exact hessian bounds appendix} (specifically, in the derivation of $\nabla^2_{W^{(2)}V^{(2)}} \mathcal{T}(X,\Theta)$)  that $\nabla_{W^{
(2)}}[\phi_{T+1}^{(2)}]_t = g_{T+1}e_t^T\phi'^{(2)}_{T+1}G, $ and thus we can write the perturbed Hessian as:
$$\nabla_{\tilde{W}^{(2)}}\Big(\Big[\nabla_{\tilde{V}^{(2)}} \mathcal{T}(X,\tilde{\Theta}) \Big]_i\Big) =\sum_{t=1}^{T+1} \nabla_{\tilde{W}^{(2)}}\big(\big[\tilde{\phi}^{(2)}_{T+1}\big]_t\big)\tilde{G}_{t,i}=\sum_{t=1}^{T+1} \tilde{G}_{t,i}\tilde{g}_{T+1}e_t^T\tilde{\phi}'^{(2)}_{T+1}\tilde{G}$$
$$\nabla_{\tilde{W}^{(2)}}\Big(\Big[\nabla_{\tilde{V}^{(2)}} \mathcal{T}(X,\tilde{\Theta}) \Big]_i\Big)-\nabla_{W^{(2)}}\Big(\Big[\nabla_{V^{(2)}} \mathcal{T}(X,\Theta) \Big]_i\Big) = \sum_{t=1}^{T+1} \tilde{G}_{t,i}\tilde{g}_{T+1}e_t^T\tilde{\phi}'^{(2)}_{T+1}\tilde{G} -G_{t,i}g_{T+1}e_t^T\phi'^{(2)}_{T+1}G$$
$$\approx \sum_{t=1}^{T+1} (\tilde{G}_{t,i}\tilde{g}_{T+1}-G_{t,i}g_{T+1})e_t^T\phi'^{(2)}_{T+1}G$$
$$+ \sum_{t=1}^{T+1} G_{t,i}g_{T+1}e_t^T\Big(\tilde{\phi}'^{(2)}_{T+1}-\phi'^{(2)}_{T+1}\Big)G$$
$$+\sum_{t=1}^{T+1} G_{t,i}g_{T+1}e_t^T\phi'^{(2)}_{T+1}(\tilde{G}-G)$$

We can decompose the contribution to the perturbation of the overall hessian using the triangle inequality:
$$\Delta_{Term_1}\leq \underbrace{2\sum_{i=1}^{d+1} \sum_{t=1}^{\omega}\sum_{k=1}^{d+1}  |s_{i}|\Big|G_{t,i}G_{T+1,k}\Big\|\Big|e_t^T{\phi'^{(2)}_{T+1}}(\tilde{G}-G)\Big\|||\mathbf{p_k}|| }_{Term_{I}}$$
$$+ \underbrace{2\sum_{i=1}^{d+1} \sum_{t=1}^{\omega}\sum_{k=1}^{d+1}  |s_{i}|\Big|G_{t,i}G_{T+1,k}\Big\|\Big|e_t^T(\tilde{\phi}'^{(2)}_{T+1}-\phi'^{(2)}_{T+1})G\Big\|||\mathbf{p_k}||}_{Term{II}}$$
$$+ \underbrace{2 \sum_{i=1}^{d+1} \sum_{t=1}^{\omega}\sum_{k=1}^{d+1}  |s_{i}|\Big|\tilde{G}_{t,d+1}\tilde{G}_{T+1,d+1}-G_{t,d+1}G_{T+1,d+1}\Big\|\Big|e_t^T\phi'^{(2)}_{T+1}G\Big\|||\mathbf{p_{k}}||
}_{Term_{III}}$$
First, we bound $Term_I.$  By \ref{deoralemma}, we have $\Big\|e_t^T{\phi'^{(2)}_{T+1}}(\tilde{G}-G)\Big\| \leq 2 ||\tilde{G}-G||_{2,\infty} \lessapprox 64\sigma D_f^2.$
Thus
$$Term_I \lessapprox  64\sigma D_f^2 ||G||_{1,\infty}^2 \sum_{i=1}^{d+1} \sum_{t=1}^{\omega}\sum_{k=1}^{d+1}  |s_{i}|||\mathbf{p_k}||$$
$$\lessapprox 64\sigma d \omega D_f^2 $$
For $Term_{II},$ we note that $\Big\|e_t^T(\tilde{\phi}'^{(2)}_{T+1}-\phi'^{(2)}_{T+1})G\Big\| \leq  \Big\|G^T(\tilde{\phi}'^{(2)}_{T+1}-\phi'^{(2)}_{T+1})e_t\Big\|\leq ||G^T||_{1,2}\Big\|(\tilde{\phi}'^{(2)}_{T+1}-\phi'^{(2)}_{T+1})e_t\Big\|_1= ||G||_{2,\infty}\Big\|(\tilde{\phi}'^{(2)}_{T+1}-\phi'^{(2)}_{T+1})e_t\Big\|_1$
$$\lessapprox 768\sqrt{2}\sigma d D_f^2\sqrt{\omega} log(T) $$
It follows that
$$Term_{II}\lessapprox 1536\sqrt{2} \sigma D_f^2\omega^{\frac{3}{2}} d^2 log(T) $$
Finally, for $Term_{III}, $ note that 
$$\Big|\tilde{G}_{t,d+1}\tilde{G}_{T+1,d+1}-G_{t,d+1}G_{T+1,d+1}\Big| $$
$$\lessapprox \Big|G_{t,d+1}\big(\tilde{G}_{T+1,d+1}-G_{T+1,d+1}\big)\Big|+\Big|\big(\tilde{G}_{t,d+1}-G_{t,d+1}\big)G_{T+1,d+1}\Big|$$
$$\leq 2||\tilde{G}-G||_{2,\infty} \leq 64\sigma D_f^2$$
Therefore 
$$Term_{III} = 2\sum_{i=1}^{d+1} \sum_{t=1}^{\omega}\sum_{k=1}^{d+1}  |s_{i}|\Big|\tilde{G}_{t,d+1}\tilde{G}_{T+1,d+1}-G_{t,d+1}G_{T+1,d+1}\Big\|\Big|e_t^T\phi'^{(2)}_{T+1}G\Big\|||\mathbf{p_{k}}||
$$
$$\leq 128\sqrt{2} \sigma d \omega D_f^2 $$
Putting this all together, we have
$$\Delta_{Term_1}\leq  64\sigma d \omega D_f^2 +1536\sqrt{2} \sigma D_f^2\omega^{\frac{3}{2}} d^2 log(T) + 128\sqrt{2} \sigma d \omega D_f^2$$
$$\approx 1536\sqrt{2} \sigma D_f^2\omega^{\frac{3}{2}} d^2 log(T)$$

\subsection{$\Delta_{Term_2}$}
$$Term_2 = 2  \max_{||v||=1} concat(s_1,...,s_{d+1})^T\Big(\nabla^2_{MV^{(2)}}\mathcal{T}(X,\Theta) \Big)concat(\eta_1,...\eta_{d+1})$$
Recall that 
$$\Big[\nabla_{\tilde{V}^{(2)}} \mathcal{T}(X,\tilde{\Theta}) \Big]_i=e_{i}^T\tilde{G}^T{\tilde{\phi}^{(2)}_{T+1}}$$

And note that this expression is very similar to $ \mathcal{T}(X,\tilde{\Theta})=(V^{(2)})^T\tilde{G}^T\tilde{\phi}_{T+1}^{(2)},$ except that the $V^{(2)}$ has been replaced by $e_i^T.$ We can therefore follow the derivation of $\nabla_G \mathcal{T}(X,\Theta)$ in \ref{unperturbed gradient bounds appendix} to conclude that 
$$\frac{\partial\Big( e_{i}^T\tilde{G}^T{\tilde{\phi}^{(2)}_{T+1}}\Big)}{ \partial\tilde{G}_{t,j}}  = e_t^T\tilde{\phi}^{(2)}_{T+1}e_i^Te_j+e_t^Te_{T+1}e_i^T\tilde{G}^T{\tilde{\phi}'^{(2)}_{T+1}}\tilde{G}(\tilde{W}^{(2)})^Te_j$$
$$=e_t^T\tilde{\phi}^{(2)}_{T+1}I[i=j]+e_i^T\tilde{G}^T{\tilde{\phi}'^{(2)}_{T+1}}\tilde{G}(\tilde{W}^{(2)})^Te_jI[t=T+1]$$
We have from \ref{unperturbed gradient bounds appendix} that $$\nabla_MG_{t,i} =b_te_{i}^TF^Tdiag(I\big[M^Tb_t+\Gamma>0\big])$$
Therefore
$$\nabla_{\tilde{M}} \Big(\Big[\nabla_{\tilde{V}^{(2)}} \mathcal{T}(X,\tilde{\Theta}) \Big]_i\Big)=\nabla_{\tilde{M}} \Big(e_{i}^T\tilde{G}^T{\tilde{\phi}^{(2)}_{T+1}}\Big)$$
$$=\sum_{t=1}^{T+1} \sum_{j=1}^{d+1} \frac{\partial\Big( e_{i}^T\tilde{G}^T{\tilde{\phi}^{(2)}_{T+1}}\Big)}{ \partial\tilde{G}_{t,j}} \nabla_{\tilde{M}}\tilde{G}_{t,j}$$
We can decompose the perturbation as:
$$\nabla_{\tilde{M}} \Big(\Big[\nabla_{\tilde{V}^{(2)}} \mathcal{T}(X,\tilde{\Theta}) \Big]_i\Big)-\nabla_{M} \Big(\Big[\nabla_{V^{(2)}} \mathcal{T}(X,\Theta) \Big]_i\Big)$$
$$=\sum_{t=1}^{T+1} \sum_{j=1}^{d+1} \frac{\partial\Big( e_{i}^T\tilde{G}^T{\tilde{\phi}^{(2)}_{T+1}}\Big)}{ \partial\tilde{G}_{t,j}} \nabla_{\tilde{M}}\tilde{G}_{t,j}- \frac{\partial\Big( e_{i}^TG^T{\phi^{(2)}_{T+1}}\Big)}{ \partial G_{t,j}} \nabla_{M}G_{t,j}$$
$$\approx \sum_{t=1}^{T+1} \sum_{j=1}^{d+1} \Big(\frac{\partial\Big( e_{i}^T\tilde{G}^T{\tilde{\phi}^{(2)}_{T+1}}\Big)}{ \partial\tilde{G}_{t,j}}- \frac{\partial\Big( e_{i}^TG^T{\phi^{(2)}_{T+1}}\Big)}{ \partial G_{t,j}}\Big) \nabla_{M}G_{t,j}+\sum_{t=1}^{T+1} \sum_{j=1}^{d+1}\frac{\partial\Big( e_{i}^TG^T{\phi^{(2)}_{T+1}}\Big)}{ \partial G_{t,j}}\Big(\nabla_{\tilde{M}}\tilde{G}_{t,j}- \nabla_{M}G_{t,j}\Big)$$
Thus 
$$\Delta_{Term_2}=\max_{||v||=1}\sum_{i=1}^{d+1} s_{i} Vec\Big(\sum_{t=1}^{T+1} \sum_{j=1}^{d+1} \frac{\partial\Big( e_{i}^T\tilde{G}^T{\tilde{\phi}^{(2)}_{T+1}}\Big)}{ \partial\tilde{G}_{t,j}} \nabla_{\tilde{M}}\tilde{G}_{t,j}- \frac{\partial\Big( e_{i}^TG^T{\phi^{(2)}_{T+1}}\Big)}{ \partial G_{t,j}} \nabla_{M}G_{t,j}\Big)\mathbf{\eta}$$

$$\leq \sum_{i=1}^{d+1} |s_{i}| \sum_{k=1}^{d+1} \sum_{t=1}^{T+1} \sum_{j=1}^{d+1} \Bigg| \Bigg| \frac{\partial\Big( e_{i}^T\tilde{G}^T{\tilde{\phi}^{(2)}_{T+1}}\Big)}{ \partial\tilde{G}_{t,j}} \nabla_{\tilde{M}}\tilde{G}_{t,j}- \frac{\partial\Big( e_{i}^TG^T{\phi^{(2)}_{T+1}}\Big)}{ \partial G_{t,j}} \nabla_{M}G_{t,j} \Bigg| \Bigg|||\mathbf{\eta_k}||$$
$$\leq \underbrace{\sum_{i=1}^{d+1} |s_{i}| \sum_{k=1}^{d+1} \sum_{t=1}^{T+1} \sum_{j=1}^{d+1} \Bigg| \frac{\partial\Big( e_{i}^T\tilde{G}^T{\tilde{\phi}^{(2)}_{T+1}}\Big)}{ \partial\tilde{G}_{t,j}}- \frac{\partial\Big( e_{i}^TG^T{\phi^{(2)}_{T+1}}\Big)}{ \partial G_{t,j}}\Bigg| \Big\| \nabla_{M}G_{t,i}\Big\|||\mathbf{\eta_k}||}_{Term_A}$$
$$+ \underbrace{\sum_{i=1}^{d+1} |s_{i}| \sum_{k=1}^{d+1} \sum_{t=1}^{T+1} \sum_{j=1}^{d+1} \Bigg|\frac{\partial\Big( e_{i}^TG^T{\phi^{(2)}_{T+1}}\Big)}{ \partial G_{t,j}} \Bigg|\Big\|\nabla_{\tilde{M}}\tilde{G}_{t,i}- \nabla_{M}G_{t,i}\Big\| ||\mathbf{\eta_k}||}_{Term_B}$$

In the above series of inequalities, we used Cauchy-Schwarz and the triangle inequality multiple times. Now, we derive the ingredients needed to evaluate the RHS above. It was shown in \ref{unperturbed gradient bounds appendix} that 
$$\nabla_{M}G_{t,i}=b_te_{i}^TF^Tdiag(I\big[M^Tb_t+\Gamma>0\big]),$$
and that the norm of this intra-layer gradient term can be bounded as:
$$\Big\|\nabla_{M}G_{t,i}\Big\| \lessapprox8D_f^{\frac{3}{2}}I[i=d+1\wedge t\leq \omega ]$$
In \ref{perturbedmlpgradientlemma}, it was shown that
$$\Bigg|\frac{\partial\Big(\mathcal{T}(X,\tilde{\Theta})\Big)}{ \partial\tilde{G}_{t,d+1}}- \frac{\partial\Big(\mathcal{T} (X,\Theta)\Big)}{ \partial G_{t,d+1}}\Bigg|=|\langle\nabla_{\tilde{g}_t} \mathcal{T}(X,\tilde{\Theta})-\nabla_{g_t} \mathcal{T}(X,\Theta),e_{d+1}\rangle| \lessapprox  256\sigma d D_f^2\omega log(T)$$
We can follow similar arguments as \ref{perturbedmlpgradientlemma}  but with $V^{(2)}$ replaced with $e_i,$ and conclude that
$$\Bigg|\frac{\partial\Big( e_{i}^T\tilde{G}^T{\tilde{\phi}^{(2)}_{T+1}}\Big)}{ \partial\tilde{G}_{t,d+1}}- \frac{\partial\Big( e_{i}^TG^T{\phi^{(2)}_{T+1}}\Big)}{ \partial G_{t,d+1}}\Bigg|\lessapprox   256\sigma d D_f^2\sqrt{\omega} log(T) $$
Note that this differs from the perturbation in the gradient in \ref{perturbed gradient bounds appendix} only by a factor of $\sqrt{\omega}, $ as expected, since $V^{(2)}$ was carrying this factor. Thus
$$Term_A=\sum_{k=1}^{d+1} \sum_{i=1}^{d+1}\sum_{t=1}^{\omega}  |s_{i}|\Big(256\sigma d D_f^2\sqrt{\omega} log(T)\Big)\Big(8D_f^{\frac{3}{2}}\Big)||\mathbf{\eta_k}||= 2048 \sigma d(d+1) D_f^{\frac{7}{2}} \omega^{\frac{3}{2}}log(T)$$
We now move on to $Term_B.$ In \ref{perturbed gradient bounds appendix} we showed that 
$$\Big\|\nabla_{\tilde{M}}\tilde{G}_{t,i}-\nabla_{M}G_{t,i}\Big\|_F 
\lessapprox 2\sigma D_f^{\frac{1}{2}}  + 32\sigma d D_f^{\frac{3}{2}}I[i=d+1]$$
Finally,  we can make a similar argument to \ref{unperturbedmlpgradientlemma}, but once again replacing $V^{(2)}$ with $e_i,$ to conclude that 
$$||\nabla_{g_t} e_{i}^T\tilde{G}^T{\tilde{\phi}^{(2)}_{T+1}}|| \leq \frac{|c_t|}{\sqrt{\omega}} + 4log(T)I[t=T+1] +4\omega^{\frac{1}{2}}log(T)I[t\leq\omega] $$
Summing over $t\in [T+1], $ we have 
$$\sum_{t=1}^{T+1}||\nabla_{g_t} e_{i}^T\tilde{G}^T{\tilde{\phi}^{(2)}_{T+1}}|| \leq 1 + 4log(T) +4\omega^{\frac{3}{2}}log(T) \approx  4\omega^{\frac{3}{2}}log(T)$$

$$Term_B  \leq 2\sum_{i=1}^{d+1} |s_{i}| \sum_{k=1}^{d+1}\sum_{j=1}^{d+1} \Big(4\omega^{\frac{3}{2}}log(T)\Big)\Big(2\sigma D_f^{\frac{1}{2}}  + 32\sigma d D_f^{\frac{3}{2}}I[i=d+1]\Big)||\mathbf{\eta_k}||$$
$$=8\omega^{\frac{3}{2}}log(T)\sum_{i=1}^{d+1} |s_{i}| \sum_{k=1}^{d+1} \sum_{j=1}^{d+1}\Big(2\sigma D_f^{\frac{1}{2}}  + 32\sigma d D_f^{\frac{3}{2}}I[i=d+1]\Big)||\mathbf{\eta_k}||$$
$$=8\omega^{\frac{3}{2}}log(T)\sum_{i=1}^{d+1} |s_{i}| \sum_{k=1}^{d+1} \Big(2(d+1)\sigma D_f^{\frac{1}{2}}  + 32\sigma d D_f^{\frac{3}{2}}\Big)||\mathbf{\eta_k}||$$
$$\leq 8\omega^{\frac{3}{2}}log(T)(d+1)\Big(2(d+1)\sigma D_f^{\frac{1}{2}}  + 32\sigma d D_f^{\frac{3}{2}}\Big)$$
$$\approx 256\sigma \omega^{\frac{3}{2}} d^2 D_f^{\frac{3}{2}} log(T)$$
Putting $Term_A $ and $Term_B$ together, we have 

$$\Delta_{Term_2} \lessapprox 2048 \sigma d^2 D_f^{\frac{7}{2}} \omega^{\frac{3}{2}}log(T) + 256\sigma \omega^{\frac{3}{2}}d^2 D_f^{\frac{3}{2}} log(T)$$
\subsection{$\Delta_{Term_3}$}
$$ Term_3 =2\max_{||v||=1} concat(s_1,...,s_{d+1})^T\Big(\nabla^2_{\Gamma V^{(2)}}\mathcal{T}(X,\Theta)\Big)\mathbf{\gamma_1}$$
We follow the same approach as $\tilde{Term}_2$.  We have 
$$\tilde{Term}_3 - Term_3 \lessapprox \underbrace{2\sum_{i=1}^{d+1} |s_{i}| \sum_{k=1}^{d+1} \sum_{t=1}^{T+1} \sum_{j=1}^{d+1} \Bigg| \frac{\partial\Big( e_{i}^T\tilde{G}^T{\tilde{\phi}^{(2)}_{T+1}}\Big)}{ \partial\tilde{G}_{t,j}}- \frac{\partial\Big( e_{i}^TG^T{\phi^{(2)}_{T+1}}\Big)}{ \partial G_{t,j}}\Bigg| \Big\| \nabla_{\Gamma}G_{t,i}\Big\|||\mathbf{\eta_k}||}_{Term_A}$$
$$+ \underbrace{2\sum_{i=1}^{d+1} |s_{i}| \sum_{k=1}^{d+1} \sum_{t=1}^{T+1} \sum_{j=1}^{d+1} \Bigg|\frac{\partial\Big( e_{i}^TG^T{\phi^{(2)}_{T+1}}\Big)}{ \partial G_{t,j}} \Bigg|\Big\|\nabla_{\tilde{\Gamma}}\tilde{G}_{t,i}- \nabla_{\Gamma}G_{t,i}\Big\| ||\mathbf{\eta_k}||}_{Term_B}$$
In \ref{unperturbed gradient bounds appendix}, we showed that 
$$||\nabla_\Gamma G_{t,i}||\lessapprox 8 D_f^{\frac{3}{2}}I[i=d+1 \wedge t\leq \omega]$$
In addition, \ref{perturbed gradient bounds appendix}, tells us that 
$$||\nabla_{\tilde{\Gamma}}\tilde{G}_{t,i}-\nabla_{\Gamma}G_{t,i}||\leq  2\sigma \sqrt{D_f}$$
Combining this with our previous bounds on $\Bigg| \frac{\partial\Big( e_{i}^T\tilde{G}^T{\tilde{\phi}^{(2)}_{T+1}}\Big)}{ \partial\tilde{G}_{t,j}}- \frac{\partial\Big( e_{i}^TG^T{\phi^{(2)}_{T+1}}\Big)}{ \partial G_{t,j}}\Bigg| \Big\| \nabla_{\Gamma}G_{t,i}\Big\|$ earlier in this section, we have
$$Term_A \leq 2 |s_{d+1}| \sum_{k=1}^{d+1} \sum_{t=1}^{\omega} \sum_{j=1}^{d+1} \Big( 256\sigma d D_f^2\sqrt{\omega} log(T) \Big) \Big(8D_f^{\frac{3}{2}}\Big)||\mathbf{\eta_k}||$$
$$\leq 2048\sigma d D_f^{\frac{7}{2}}\omega^{\frac{3}{2}}log(T)$$
$$Term_B \leq 4\sum_{i=1}^{d+1} |s_{i}| \sum_{k=1}^{d+1}  \sum_{j=1}^{d+1} \Big(1+4log(T)+4\omega^{\frac{3}{2}}log(T)\Big)\Big(2 \sigma D_f^{\frac{1}{2}}\Big)  ||\mathbf{\eta_k}||$$
$$\approx 32\sigma \omega^{\frac{3}{2}} D_f^{\frac{1}{2}}d log(T)$$
Putting $Term_A$ and $Term_B$ together, we have 
$$\Delta_{Term_3} \lessapprox  2048\sigma d D_f^{\frac{7}{2}}\omega^{\frac{3}{2}}log(T)$$
 \subsection{$\Delta_{Term_4}$}
 $$  Term_4= 2\max_{||v||=1}concat(s_1,...,s_{d+1})^T\Big(\nabla^2_{FV^{(2)}}\mathcal{T}(X,\Theta) \Big) concat(\mathbf{\nu}_1,...\nu_{4(D_f+1)})$$
In \ref{unperturbed gradient bounds appendix}, we showed that 
$$||\nabla_F G_{t,i}||\lessapprox 2\sqrt{D_f}I[t\leq \omega]$$
In addition, \ref{perturbed gradient bounds appendix}, tells us that 
$$||\nabla_{\tilde{F}}\tilde{G}_{t,i}-\nabla_{F}G_{t,i}||\leq  8\sigma D_f^{\frac{1}{2}} d$$
$$\Delta_{Term} \lessapprox \underbrace{2\sum_{i=1}^{d+1} |s_{i}| \sum_{k=1}^{d+1} \sum_{t=1}^{T+1} \sum_{j=1}^{d+1} \Bigg| \frac{\partial\Big( e_{i}^T\tilde{G}^T{\tilde{\phi}^{(2)}_{T+1}}\Big)}{ \partial\tilde{G}_{t,j}}- \frac{\partial\Big( e_{i}^TG^T{\phi^{(2)}_{T+1}}\Big)}{ \partial G_{t,j}}\Bigg| \Big\|Vec\big( \nabla_{F}G_{t,i}\big)\Big\|||\mathbf{\eta_k}||}_{Term_A}$$
$$+ \underbrace{2\sum_{i=1}^{d+1} |s_{i}| \sum_{k=1}^{d+1} \sum_{t=1}^{T+1} \sum_{j=1}^{d+1} \Bigg|\frac{\partial\Big( e_{i}^TG^T{\phi^{(2)}_{T+1}}\Big)}{ \partial G_{t,j}} \Bigg|\Big\|Vec\big(\nabla_{\tilde{F}}\tilde{G}_{t,i}- \nabla_{F}G_{t,i}\big )\Big\| ||\mathbf{\eta_k}||}_{Term_B}$$
$$Term_A \leq 2\sum_{i=1}^{d+1} |s_{i}| \sum_{k=1}^{d+1} \sum_{t=1}^{\omega} \sum_{j=1}^{d+1} \Big(  256\sigma d D_f^2\sqrt{\omega} log(T)\Big)\big(2\sqrt{D_f}\big)||\mathbf{\eta_k}||$$
$$\leq 1024 \sigma d^3 D_f^{\frac{5}{2}} \omega^{\frac{3}{2}} log(T) $$
  $$Term_B \leq 2\sum_{i=1}^{d+1} |s_{i}| \sum_{k=1}^{d+1} \sum_{j=1}^{d+1} \Big(1+4log(T)+4\omega^{\frac{3}{2}}log(T)\Big)\Big(8\sigma D_f^{\frac{1}{2}}d\Big) ||\mathbf{\eta_k}||$$
  $$\lessapprox 64\sigma D_f^{\frac{1}{2}}log(T)d^2 $$
  Thus
  $$\Delta_{Term_4} \lessapprox 1024 \sigma d^3 D_f^{\frac{5}{2}} \omega^{\frac{3}{2}} log(T) +64\sigma \omega^{\frac{3}{2}} D_f^{\frac{1}{2}}log(T)d^2$$
  $$\approx  1024 \sigma D_f^{\frac{5}{2}} \omega^{\frac{3}{2}} d^3  log(T)$$
  \subsection{$\Delta_{Term_5}$}
  $$Term_5=2\max_{||v||=1} \gamma_1^T\Big(\nabla^2_{F \Gamma}\mathcal{T}(X,\Theta) \Big) concat(\nu_1,...,\nu_{4(D_f+1)})$$
$$=2\max_{||v||=1} concat(\nu_1,...,\nu_{4(D_f+1)})^T\Big(\nabla^2_{ \Gamma F}\mathcal{T}(X,\Theta) \Big)\gamma_1 $$
Recall that $$\Big[\nabla_{F} \mathcal{T}(X,\Theta) \Big]_{i,j}= \sum_{k=1}^{d+1}\sum_{t=1}^{\omega} c_te_i^T\big(M^Tb_t + \Gamma\big)_+e_{k}^T e_j $$

To obtain $\nabla^2_{\tilde{F} \tilde{\Gamma}}\mathcal{T}(X,\tilde{\Theta})$ in the perturbed case, we once again modify the unperturbed hessian to account for the fact that the final derivative in the perturbed case depends on more than just the last dimension. We start with a generalization of the perturbed gradient, noting that in the unperturbed case, we had $\frac{\partial  \mathcal{T}(X,\tilde{\Theta})}{\partial \tilde{G}_{t,d+1}} =c_t .$ 
$$\Big[\nabla_{\tilde{F}} \mathcal{T}(X,\tilde{\Theta}) \Big]_{i,j}= \sum_{k=1}^{d+1}\sum_{t=1}^{\omega} \frac{\partial  \mathcal{T}(X,\tilde{\Theta})}{\partial \tilde{G}_{t,k}}e_i^T\nabla_{\tilde{F}}\tilde{G}_{t,k}e_j$$
$$\Big[\nabla_{\tilde{F}} \mathcal{T}(X,\tilde{\Theta}) \Big]_{i,j}= \sum_{k=1}^{d+1}\sum_{t=1}^{\omega}  \frac{\partial  \mathcal{T}(X,\tilde{\Theta})}{\partial \tilde{G}_{t,k}} e_i^T\big(\tilde{M}^T\tilde{b}_t + \tilde{\Gamma}\big)_+e_{k}^T e_j $$
$$=\sum_{t=1}^{\omega}  \frac{\partial  \mathcal{T}(X,\tilde{\Theta})}{\partial \tilde{G}_{t,j}}\big([(\tilde{M}_{:,i})^T\tilde{b}_t + \tilde{\Gamma}_{i}]\big)_+$$
Taking the derivative with respect to $\tilde{\Gamma}$ yields
$$\nabla_{\tilde{\Gamma}} \Big(\big[\nabla_{\tilde{F} }\mathcal{T}(X,\tilde{\Theta})\big]_{i,j} \Big)= \sum_{t=1}^{\omega} \frac{\partial  \mathcal{T}(X,\tilde{\Theta})}{\partial \tilde{G}_{t,j}}e_iI\big[\tilde{M}_i^T\tilde{b}_t + \tilde{\Gamma}_i>0\big] $$
We have used the fact that $ \frac{\partial  \mathcal{T}(X,\tilde{\Theta})}{\partial \tilde{G}_{t,k}}$ is determined by the portion of the transformer downstream of the MLP; namely, the second attention layer. Therefore it is independent of the MLP and constant relative to $\Gamma,$ and can be taken outside of the derivative. We then upper bound the difference as 

$$\Delta_{Term_5}\leq\sum_{i=1}^{4(D_f+1)}\sum_{j=1}^{d+1}\sum_{t=1}^{\omega}|\mathbf{\nu}_{ij}|\Bigg|Vec\Big( \frac{\partial  \mathcal{T}(X,\tilde{\Theta})}{\partial \tilde{G}_{t,j}}e_iI\big[\tilde{M}_i^T\tilde{b}_t + \tilde{\Gamma}_i>0\big]- \frac{\partial  \mathcal{T}(X,\Theta)}{\partial G_{t,j}}e_iI\big[M_i^Tb_t + \Gamma_i>0\big] \Big)\Bigg|||\mathbf{\gamma}_{1}||$$
$$=\sum_{i=1}^{4(D_f+1)}\sum_{j=1}^{d+1}\sum_{t=1}^{\omega}\sum_{k=1}^{d+1}|\mathbf{\nu}_{ij}|I[i=k]\Bigg|\frac{\partial  \mathcal{T}(X,\tilde{\Theta})}{\partial \tilde{G}_{t,j}}I\big[\tilde{M}_i^T\tilde{b}_t + \tilde{\Gamma}_i>0\big]- \frac{\partial  \mathcal{T}(X,\Theta)}{\partial G_{t,j}}I\big[M_i^Tb_t + \Gamma_i>0\big] \Bigg||\mathbf{\gamma}_{1k}|$$
$$=\sum_{i=1}^{4(D_f+1)}\sum_{j=1}^{d+1}\sum_{t=1}^{\omega}|\mathbf{\nu}_{ij}|\Bigg|\frac{\partial  \mathcal{T}(X,\tilde{\Theta})}{\partial \tilde{G}_{t,j}}I\big[\tilde{M}_i^T\tilde{b}_t + \tilde{\Gamma}_i>0\big]- \frac{\partial  \mathcal{T}(X,\Theta)}{\partial G_{t,j}}I\big[M_i^Tb_t + \Gamma_i>0\big] \Bigg||\mathbf{\gamma}_{1i}|$$
$$\lessapprox \underbrace{\sum_{i=1}^{4(D_f+1)}\sum_{j=1}^{d+1}\sum_{t=1}^{\omega}|\mathbf{\nu}_{ij}|\Bigg|\frac{\partial  \mathcal{T}(X,\Theta)}{\partial G_{t,j}}\Big(I\big[\tilde{M}_i^T\tilde{b}_t + \tilde{\Gamma}_i>0\big]- I\big[M_i^Tb_t + \Gamma_i>0\big] \Big)\Bigg||\mathbf{\gamma}_{1i}|}_{Term_A}$$
$$+\underbrace{\sum_{i=1}^{4(D_f+1)}\sum_{j=1}^{d+1}\sum_{t=1}^{\omega}|\mathbf{\nu}_{ij}|\Bigg|\Big(\frac{\partial  \mathcal{T}(X,\tilde{\Theta})}{\partial \tilde{G}_{t,j}}- \frac{\partial  \mathcal{T}(X,\Theta)}{\partial G_{t,j}}\Big)I\big[M_i^Tb_t + \Gamma_i>0\big] \Bigg||\mathbf{\gamma}_{1i}|}_{Term_B}$$
Where the last inequality follows from the triangle inequality, and the fact that we are approximating the perturbation using only terms that are first-order in $\sigma. $ Once again, we use the fact that with high probability, none of the neurons in the MLP will flip with high probability to conclude that $Term_A=0.$ Clearly, $\Big|I\big[M^Tb_t+\Gamma\geq0\big]_i\Big|\leq 1, $ and thus 
$$Term_B \leq \sum_{i=1}^{4(D_f+1)}\sum_{j=1}^{d+1}\sum_{t=1}^{\omega}|\mathbf{\nu}_{ij}|\Bigg|\frac{\partial  \mathcal{T}(X,\tilde{\Theta})}{\partial \tilde{G}_{t,j}}- \frac{\partial  \mathcal{T}(X,\Theta)}{\partial G_{t,j}}\Bigg||\mathbf{\gamma}_{1i}|$$
In \ref{perturbedmlpgradientlemma} we show that
$$\Bigg|\frac{\partial  \mathcal{T}(X,\tilde{\Theta})}{\partial \tilde{G}_{t,i}}- \frac{\partial  \mathcal{T}(X,\Theta)}{\partial G_{t,i}}\Bigg|\lessapprox 256\sigma d D_f^2\omega log(T)I[i=d+1] +1536\sigma d D_f^2 \omega^{\frac{3}{2}} log^2(T)I[i<d+1].$$
Therefore,
$$\Delta_{Term_5}= Term_B \leq $$
$$\sum_{i=1}^{4(D_f+1)}\sum_{t=1}^{\omega}|\mathbf{\nu}_{i,d+1}|\Bigg|\frac{\partial  \mathcal{T}(X,\tilde{\Theta})}{\partial \tilde{G}_{t,d+1}}- \frac{\partial  \mathcal{T}(X,\Theta)}{\partial G_{t,d+1}}\Bigg||\mathbf{\gamma}_{1i}|+\sum_{i=1}^{4(D_f+1)}\sum_{j=1}^{d}\sum_{t=1}^{\omega}|\mathbf{\nu}_{ij}|\Bigg|\frac{\partial  \mathcal{T}(X,\tilde{\Theta})}{\partial \tilde{G}_{t,j}}- \frac{\partial  \mathcal{T}(X,\Theta)}{\partial G_{t,j}}\Bigg||\mathbf{\gamma}_{1i}|$$
$$\leq 512\sigma d D_f^{\frac{7}{2}}\omega^2 log(T)+3072\sigma D_f^{\frac{7}{2}} \omega^{\frac{5}{2}} d^{\frac{3}{2}}log^2(T) $$
$$\approx 3072\sigma D_f^{\frac{7}{2}} \omega^{\frac{5}{2}} d^{\frac{3}{2}}log^2(T) $$

 \subsection{$\Delta_{Term_6}$}
Recall that 
$$Term_6 = 2  \max_{||v||=1} concat(\eta_1,\dots,\eta_{d+1})^T\Big(\nabla^2_{F M}\mathcal{T}(X,\Theta) \Big)concat(\nu_1,...,\nu_{4(D_f+1)})$$
We once again re-derive the hessian under the assumption that the output depends on all dimensions of $g_t.$ By the chain rule, we have the following for the perturbed gradient with respect to $M:$
$$\nabla_{\tilde{M}} \mathcal{T}(X,\tilde{\Theta}) = \sum_{k=1}^{d+1}\sum_{t=1}^{\omega} \frac{\partial  \mathcal{T}(X,\tilde{\Theta})}{\partial \tilde{G}_{t,k}}\tilde{b}_te_{k}^T\tilde{F}^Tdiag(I\big[\tilde{M}^T\tilde{b}_t+\tilde{\Gamma}>0\big])$$
$$\implies \Big[\nabla_{\tilde{M}} \mathcal{T}(X,\tilde{\Theta})\Big]_{ij} = \sum_{k=1}^{d+1}\sum_{t=1}^{\omega} \frac{\partial  \mathcal{T}(X,\tilde{\Theta})}{\partial \tilde{G}_{t,k}}e_i^T\tilde{b}_te_{k}^T\tilde{F}^Tdiag(I\big[\tilde{M}^T\tilde{b}_t+\tilde{\Gamma}>0\big])e_j$$
where $e_i,e_k\in \mathbb{R}^{d+1},$ $e_j\in\mathbb{R}^{4(D_f+1)}.$ The partial derivative $ \frac{\partial  \mathcal{T}(X,\tilde{\Theta})}{\partial \tilde{G}_{t,k}}$ is independent of $F,$ and therefore the gradient with resect to $F$ is given by:
$$\nabla_{\tilde{F}}\Big(\Big[\nabla_{\tilde{M}} \mathcal{T}(X,\tilde{\Theta})\Big]_{ij}\Big)=\sum_{k=1}^{d+1}\sum_{t=1}^{\omega} \frac{\partial  \mathcal{T}(X,\tilde{\Theta})}{\partial \tilde{G}_{t,k}}\nabla_{\tilde{F}}\Big(e_i^T\tilde{b}_te_{k}^T\tilde{F}^Tdiag(I\big[\tilde{M}^T\tilde{b}_t+\tilde{\Gamma}>0\big])e_j\Big)$$
We again leverage the rule of matrix calculus that says $\frac{d}{dM}\Big(x^TMy\Big) = xy^T,$ we have 
$$\nabla_{\tilde{F}}\Big(\Big[\nabla_{\tilde{M}} \mathcal{T}(X,\tilde{\Theta})\Big]_{ij}\Big) = \sum_{k=1}^{d+1}\sum_{t=1}^{\omega} \frac{\partial  \mathcal{T}(X,\tilde{\Theta})}{\partial \tilde{G}_{t,k}}e_i^T\tilde{b}_te_ke_j^Tdiag(I\big[\tilde{M}^T\tilde{b}_t+\tilde{\Gamma}>0\big])$$
We can decompose this perturbation as follows:
$$\nabla_{\tilde{F}}\Big(\Big[\nabla_{\tilde{M}} \mathcal{T}(X,\tilde{\Theta})\Big]_{ij}\Big) -\nabla_F\Big(\Big[\nabla_{M} \mathcal{T}(X,\Theta)\Big]_{ij}\Big)$$
$$\lessapprox \sum_{k=1}^{d+1}\sum_{t=1}^{\omega} \Big(\frac{\partial  \mathcal{T}(X,\tilde{\Theta})}{\partial \tilde{G}_{t,k}}- \frac{\partial  \mathcal{T}(X,\Theta)}{\partial G_{t,k}}\Big)e_i^Tb_te_ke_j^Tdiag(I\big[M^Tb_t+\Gamma>0\big])$$
$$+ \sum_{k=1}^{d+1}\sum_{t=1}^{\omega}  \frac{\partial  \mathcal{T}(X,\tilde{\Theta})}{\partial \tilde{G}_{t,k}}e_i^T(\tilde{b}_t-b_t)e_ke_j^Tdiag(I\big[M^Tb_t+\Gamma>0\big])$$
$$+ \sum_{k=1}^{d+1}\sum_{t=1}^{\omega}  \frac{\partial  \mathcal{T}(X,\tilde{\Theta})}{\partial \tilde{G}_{t,k}}\Big(e_i^Tb_te_ke_j^Tdiag(I\big[\tilde{M}^T\tilde{b}_t+\tilde{\Gamma}>0\big])-diag(I\big[M^Tb_t+\Gamma>0\big])
\Big)$$
Note that once again we argue that for $\sigma$ small in the sense described in \ref{mlpperturbationlemma},
the final sum will vanish. In the first sum, note that since the first $d$ dimensions of the unperturbed attention outputs $b_t$ are all $0,$ the summand vanishes when $i\neq d+1. $ To bound the norm of the perturbation, we apply the triangle inequality, where $\mathbf{\nu_k}\in \mathbb{R}^{4(D_f+1)}$ denotes the sub-vector of $\mathbf{\nu}$ corresponding to the $k$-th column of $F$:
$$\Delta_{Term_6} \lessapprox \sum_{j=1}^{d+1}\sum_{k=1}^{d+1}\sum_{t=1}^{\omega} |\eta_{d+1,j}|\Bigg|\frac{\partial  \mathcal{T}(X,\tilde{\Theta})}{\partial \tilde{G}_{t,k}}- \frac{\partial  \mathcal{T}(X,\Theta)}{\partial G_{t,k}}\Bigg|||\mathbf{\nu_k}||$$
$$+\sum_{i=1}^{d+1}\sum_{j=1}^{d+1}\sum_{k=1}^{d+1}\sum_{t=1}^{\omega}|\mathbf{\eta_{ij}}|  \Bigg|\frac{\partial  \mathcal{T}(X,\tilde{\Theta})}{\partial \tilde{G}_{t,k}}\Bigg|\Big\|\tilde{b}_t-b_t\Big\|||\mathbf{\nu_k}||$$
We once again use \ref{perturbedmlpgradientlemma}:
$$\Bigg|\frac{\partial  \mathcal{T}(X,\tilde{\Theta})}{\partial \tilde{G}_{t,i}}- \frac{\partial  \mathcal{T}(X,\Theta)}{\partial G_{t,i}}\Bigg|\lessapprox 256\sigma d D_f^2\omega log(T)I[i=d+1] +1536\sigma d D_f^2 \omega^{\frac{3}{2}} log^2(T)I[i<d+1].$$
we also leverage our bound from \ref{unperturbedmlpgradientlemma} that says that 
 $$||\nabla_{g_t}\mathcal{T}(X,\Theta)|| \leq  |c_t| + 4\sqrt{\omega}log(T)I[t=T+1]+4\omega log(T)I[t\in [\omega]]$$
Plugging these bounds in, we get
$$\Delta_{Term_6} \lessapprox 1536\sigma d D_f^2 \omega^{\frac{3}{2}} log^2(T)\sum_{j=1}^{d+1}\sum_{k=1}^{d}\sum_{t=1}^{\omega} |\eta_{d+1,j}|||\mathbf{\nu_k}||+256\sigma d D_f^2\omega log(T)\sum_{j=1}^{d+1}\sum_{t=1}^{\omega} |\eta_{d+1,j}|||\mathbf{\nu_{d+1}}||$$
$$+ 4\omega^2 log(T)\max_{t\leq\omega}\Big\|\tilde{b}_t-b_t\Big\|\sum_{i=1}^{d+1}\sum_{j=1}^{d+1}\sum_{k=1}^{d+1} |\mathbf{\eta_{ij}}| ||\mathbf{\nu_k}||$$
$$\leq 1536\sigma D_f^2 \omega^{\frac{5}{2}} d^2 log^2(T) +256\sigma D_f^2\omega^2 d^{\frac{3}{2}} log(T)+16\sigma \omega^2 d^{\frac{5}{2}} log(T)$$
$$\approx 1536\sigma D_f^2 \omega^{\frac{5}{2}} d^2 log^2(T)+16\sigma \omega^2 d^{\frac{5}{2}} log(T) $$

\subsection{Final Result for Perturbed Hessian Norm Bounds}
Putting these together, we have 
$$\underbrace{1536\sqrt{2} \sigma D_f^2\omega^{\frac{3}{2}} d^2 log(T)}_{\Delta_{Term_1}}+\underbrace{2048 \sigma d^2 D_f^{\frac{7}{2}} \omega^{\frac{3}{2}}log(T)}_{\Delta_{Term_2}}+\underbrace{2048\sigma d D_f^{\frac{7}{2}}\omega^{\frac{3}{2}}log(T)}_{\Delta_{Term_3}}$$
$$+\underbrace{1024 \sigma D_f^{\frac{5}{2}} \omega^{\frac{3}{2}} d^3  log(T)}_{\Delta_{Term_4}}+\underbrace{ 3072\sigma D_f^{\frac{7}{2}} \omega^{\frac{5}{2}} d^{\frac{3}{2}}log^2(T)}_{\Delta_{Term_5}}+\underbrace{1536\sigma D_f^2 \omega^{\frac{5}{2}} d^2 log^2(T)+16\sigma \omega^2 d^{\frac{5}{2}} log(T)}_{\Delta_{Term_6}}$$

$$\in O(\sigma D_f^{\frac{7}{2}}\omega^{\frac{5}{2}}log(T)^{\frac{7}{2}})$$
\end{proof}

\end{document}